\pdfoutput=1
\documentclass{article}
\usepackage{iclr2027_conference,times}
\iclrfinalcopy
\usepackage[T1]{fontenc}
\usepackage{amsmath,amssymb,amsthm,graphicx,booktabs,tabularx,longtable}
\usepackage{hyperref,url,enumitem,wrapfig}
\hypersetup{colorlinks=true,linkcolor=blue,citecolor=blue,urlcolor=blue,pdftitle={How Far Do Persona Effects Generalize in Language Models?},pdfauthor={Yufan Zhou, Yuxuan Liu, Enze Ma, Lyumanshan Ye, Zhongqi Yue, Robin De Croon, Yucheng Jin, Katrien Verbert, Zhao Wang}}
\title{How Far Do Persona Effects Generalize\\in Language Models?}
\author{\parbox[t]{\dimexpr\textwidth-2\tabcolsep\relax}{\centering
Yufan Zhou$^{1}$\quad Yuxuan Liu$^{2}$\quad Enze Ma$^{3}$\quad Lyumanshan Ye$^{4}$\quad Zhongqi Yue$^{5}$\\
Robin De Croon$^{1}$\quad Yucheng Jin$^{6\dagger}$\quad Katrien Verbert$^{1\dagger}$\quad Zhao Wang$^{7\dagger}$\\[4pt]
\normalfont\small
$^{1}$KU~Leuven\quad $^{2}$East~China~University~of~Science~and~Technology\quad $^{3}$University~of~Illinois~Chicago\\
$^{4}$Shanghai~Jiao~Tong~University\quad $^{5}$Microsoft~Research\quad $^{6}$Duke~Kunshan~University\quad $^{7}$Zhejiang~University\\[2pt]
$^{\dagger}$Corresponding authors}}

\newcommand{\E}{\mathbb{E}}
\newcommand{\CE}{\operatorname{CE}}
\newcommand{\sigmoid}{\sigma}

\newcommand{\argmin}{\operatorname*{arg\,min}}
\newtheorem{proposition}{Proposition}

\makeatletter
\renewcommand\section{\@startsection{section}{1}{\z@}{-1.4ex plus -0.3ex minus -.2ex}{1.0ex plus 0.2ex minus 0.2ex}{\large\sc\raggedright}}
\renewcommand\subsection{\@startsection{subsection}{2}{\z@}{-1.3ex plus -0.3ex minus -.2ex}{0.6ex plus .1ex}{\normalsize\sc\raggedright}}
\makeatother
\begin{document}
\maketitle
\lhead{Preprint}
\suppressfloats[t]
\begin{abstract}
Persona prompts ask language models to answer as particular kinds of people. We test whether relationships learned from these effects predict responses to new questions and remain useful across models and prompts. Across 57 attributes, three behavioral domains, and seven pairs of open 7 to 9B checkpoints, persona effects can be predictable without being portable. Separate attribute and task gains improve prediction beyond shared scaling significantly in OLMo-3 and Qwen2.5, with the most robust evidence in OLMo-3. In that model, target refitting significantly outperforms gains borrowed from each of the other six pairs. Across model transfers, borrowed gains with one amplitude underperform shared scaling in most directions; allowing two target parameters removes the significant losses but yields no significant benefit over target shared scaling. After rewording, refitting significantly outperforms reuse with one amplitude in all six tested pairs, while changes of examples or country context often preserve reuse value. In the tested prompt transfers, regularized updates outperform both reuse strategies in median at 64 target questions per attribute. A separate survey comparison finds that selecting the more responsive checkpoint can worsen human fit; responsiveness is confounded with training status, and temperature calibration largely removes this cost but not errors in group ordering. Within the tested gain representation, apparent transfer can come from target calibration; source relationships must add predictive value beyond calibration and regularization. Code and data are available at \url{https://github.com/thzva/persona-gain}.
\end{abstract}
\section{Introduction}
Persona prompting lets researchers generate answers that appear to come from particular kinds of people: a retiree in India, an entrepreneur comfortable with risk, a devout rural voter. These answers are increasingly used as evidence about human behavior: to pilot surveys, replicate experiments, and populate social simulations \citep{argyle2023outofone,aher2023simulate,park2024thousand}. Getting a model to change its answers when the persona changes is easy. The hard part is knowing what those changes mean: whether they form \emph{generalizable behavioral regularities}, which \emph{conditions} those regularities survive, and how far they \emph{correspond to real populations}.

When a model or prompt changes, a researcher must decide whether an earlier persona measurement still applies. Existing evaluations examine survey agreement \citep{santurkar2023opinions}, responsiveness to attributes \citep{beck2024sociodemographic}, population diversity \citep{chameleon2026limit}, and sensitivity to wording and format \citep{lutz2025prompt}. They characterize behavior within or across settings, but not what a source relationship contributes to prediction in a new target. A model may respond more strongly while preserving the same attribute relationships, and two fitted patterns may look similar without supporting useful transfer. Moreover, when adaptation uses target data, improved prediction may come from target calibration rather than the borrowed pattern. If both the model and prompt change, an audit may attribute a prompt effect to the model update; a simulation that reuses an earlier relationship may impose a pattern unsupported by the target. The practical question is therefore which parts of an earlier measurement remain informative, and which must be estimated again on the target.

We measure a persona effect as the change in answer log odds when one attribute value changes. A gain describes how this effect scales between checkpoints or prompts. We first test whether gains that vary across attributes and tasks predict responses to new questions better than one shared multiplier. We then compare source reuse, target calibration, and target refitting using the same target fitting and evaluation questions (\S\ref{sec:method}). Gains between Base and trained checkpoints describe training changes; within a deployed model they describe changes between prompts, and we also test direct reuse without a Base (Appendix~\ref{app:direct-transfer}). The behavioral studies cover 57 attributes, three domains, and seven matched checkpoint pairs of open 7 to 9B models. A separate survey study tests the relationship between responsiveness and human fidelity.

\begin{figure}[!t]
\centering\includegraphics[width=\linewidth]{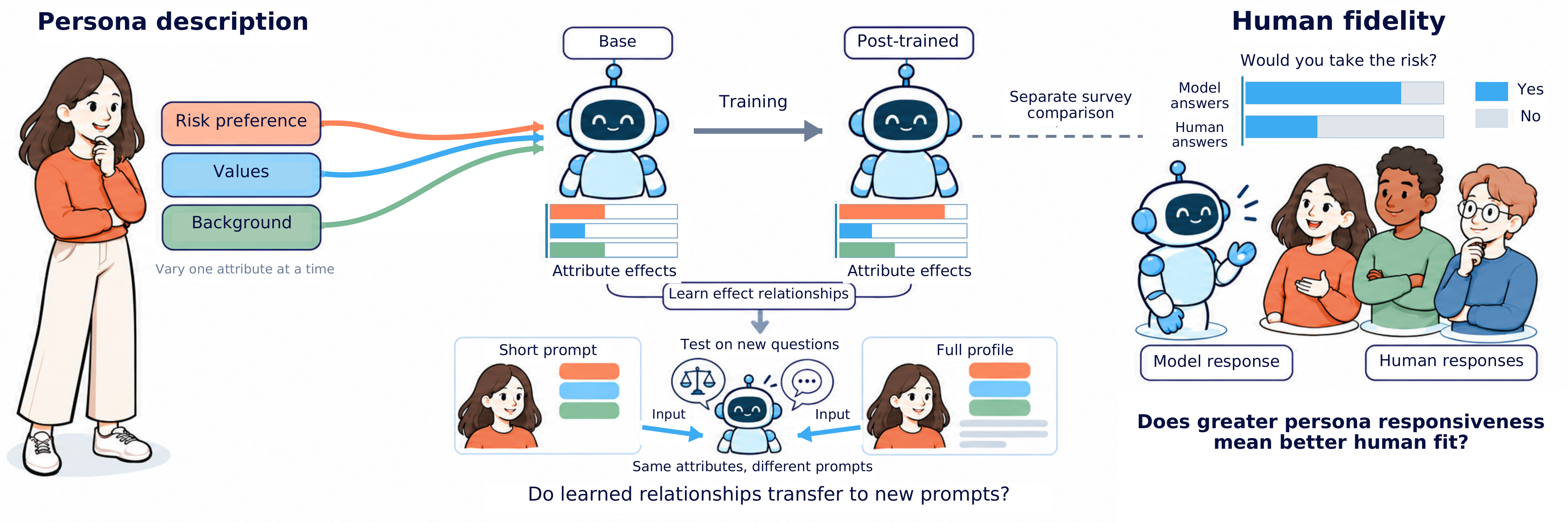}
\caption{\textbf{Predictability, reuse, and human fidelity require separate tests.} We vary one attribute at a time, compare effects before and after training, and learn relationships between them. Below, these relationships are tested on new questions with short prompts or full profiles. Right, human surveys provide a separate evaluation with different attributes, rather than validating the fitted gains. Model and human answer distributions can disagree (blue: Yes; gray: No). The question and bars are illustrative, not measured results.}
\label{fig:workflow}
\end{figure}

Our contribution is a controlled predictive comparison that separates the information supplied by a source relationship from the benefits of target calibration and regularization. It tests whether reuse adds predictive value and how much a target update recovers under the same data budget. Three findings follow:
\begin{itemize}[leftmargin=1.2em,itemsep=2pt,topsep=3pt,parsep=0pt]
\item \textbf{Training reshapes persona effects beyond uniform scaling.} Separate attribute and task gains improve prediction significantly in OLMo-3 and Qwen2.5 after fitting uncertainty and correction across pairs; only OLMo-3 remains significant in every robustness condition (\S\ref{sec:results}).
\item \textbf{Apparent transfer can come from target calibration.} In OLMo-3, the target with the most structure, refitting significantly beats all six borrowed patterns under either reuse predictor; for weakly structured targets the test can only bound source value. After rewording, refitting beats reuse with one amplitude in all six pairs; other prompt changes often leave source patterns useful (\S\ref{sec:transfer}).
\item \textbf{Responsiveness alone does not establish human fidelity.} In the survey comparisons, the more responsive checkpoint is always the trained one. Selecting it can worsen human fit, but calibration largely removes this cost; errors in group ordering remain (\S\ref{sec:human}).
\end{itemize}
Persona effects can be \textbf{predictable without being portable}, and reuse must be established on the target rather than inferred from similarity between fitted patterns.

\section{A predictive framework for persona effects}
\label{sec:method}
We begin with two questions: how does an attribute change a model's answer, and how does training change that response? We measure the first by comparing prompts and describe the second with a fitted multiplier, which we call a gain. We then ask whether these gains predict responses to new questions and remain useful when the context changes.

\subsection{Measuring how an attribute changes an answer}
A single answer cannot tell us how much a persona attribute mattered. We therefore compare two prompts that ask the same question and differ in exactly one attribute value, for example a higher versus lower stated risk tolerance, and measure the resulting shift in the model's preference between answers A and B. Let $z_m(q,P)=\log p_m(A\mid q,P)-\log p_m(B\mid q,P)$ be model $m$'s answer log odds for question $q$ and persona prompt $P$. For attribute $a$, values $v_i,v_j$, and presentation $c$,
\begin{equation}
 d_m(q,a,i,j;c)=z_m(q,P_c(a,v_i))-z_m(q,P_c(a,v_j)),
 \label{eq:contrasts}
\end{equation}
where $P_c(a,v)$ renders value $v$ in presentation $c$.

\paragraph{Coverage.} We measure 57 attributes in 12 categories through 287 statements, spanning demographics, values, psychology, relationships, work, health, and daily habits. Each prompt follows a fixed completion template: a short persona statement, a situation, and a forced choice between A and B after four balanced examples (Appendix~\ref{app:design}). Gains are fitted on 181 opinion, decision, and social judgment questions and tested on 45 newly written ones; every condition of a question stays in the same split, so success means generalizing to unseen questions. Weights balance tasks and attribute categories. T\"ulu refers to T\"ulu-3.

\subsection{Learning how training changes attribute effects}
Let $x$ be an attribute effect in the Base checkpoint and $y$ the same effect after training. The effect can differ from question to question; those differences remain in $x$. What we ask is whether a common rule describes how training changes them. The simplest account of training is \emph{shared scaling}, $\widehat y=gx$: one multiplier that amplifies or attenuates every persona effect alike. It includes no change as the special case $g=1$, allowing it to account for uniform changes in magnitude. We ask whether training does more than this, by letting the multiplier, or \emph{gain}, depend on attribute $a$ and task category $b$:
\begin{equation}
 \widehat g=\argmin_{g\in\mathcal G}\sum_{q\in Q_{\rm tr}}\sum_{a}\sum_{i<j}w_{q,a,ij}\bigl(y_{q,a,ij}-g_{a,b(q)}x_{q,a,ij}\bigr)^2,\qquad
 \widehat y_{q,a,ij}=\widehat g_{a,b(q)}\,x_{q,a,ij},
 \label{eq:gains}
\end{equation}
where $b(q)$ is the task category of question $q$. The class $\mathcal G$ ranges from one shared gain, through additive gains $g_{ab}=\mu+u_a+v_b$, to a separate gain for each combination of attribute and task. A gain describes how two measured effects relate; it is a property of a model pair, not of the attribute. Adding gains cannot worsen the fit to the training questions, so we judge predictions on new questions:
\begin{equation}
 Q^2(\widehat y)=1-\frac{\|y-\widehat y\|_w^2}{\|y\|_w^2},\qquad
 \Delta Q^2=Q^2(\widehat y_{\rm rich})-Q^2(\widehat y_{\rm shared}),
 \label{eq:q2}
\end{equation}
with the same balancing weights. Positive $\Delta Q^2$ means that allowing gains to vary by attribute and task predicts new questions better than one shared multiplier. We report this difference in \emph{points} ($100\times\Delta Q^2$). Because $Q^2$ uses the zero prediction as its reference, its denominator is the uncentered energy of the target contrasts; it is neither an $R^2$ centered on the mean nor an accuracy.

A point removes one percent of the target contrast energy from squared error, so its size in log odds varies by model; Table~\ref{tab:prediction} reports both.

This comparison asks whether training changes more than the overall strength of the response. With gains refitted, $\Delta Q^2$ is invariant to positive global rescaling of either checkpoint (Proposition~\ref{prop:invariance}, Appendix~\ref{app:math}). A global temperature change alone therefore cannot produce an advantage: the advantage must reflect differences across attributes or tasks. To distinguish the attribute contribution from task differences, we also compare against one gain per task in Appendix Table~\ref{tab:olmo-ablation}.

\subsection{Testing whether a learned relationship carries over}
\begin{figure}[t]
\centering
\begin{minipage}[c]{0.25\linewidth}
\includegraphics[width=\linewidth]{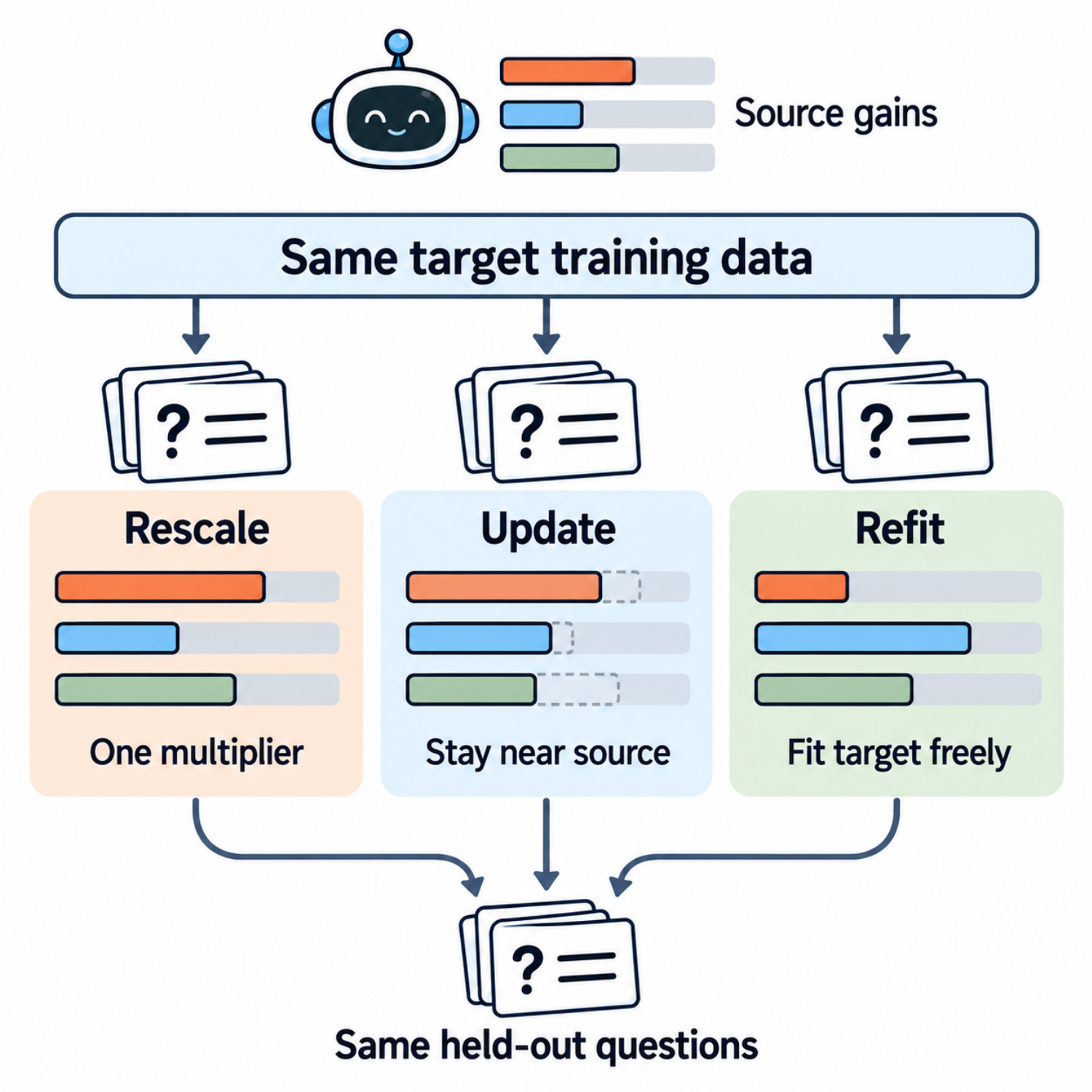}
\end{minipage}\hfill
\begin{minipage}[c]{0.71\linewidth}
\caption{\textbf{Keep, update, or relearn the source gains.} Bars illustrate fitted attribute gains, not model weights. Rescale applies one multiplier; update penalizes departures from the source pattern (dashed guides); refit learns target gains without that penalty. Question cards denote data, not added parameters. All methods use the same target Base effects, fitting questions, and evaluation questions, and are compared with shared scaling. The language models remain fixed.}
\label{fig:predictability-reuse}
\end{minipage}
\end{figure}
To test reuse, we carry gains learned in a source setting into a target with another checkpoint pair, prompt, or persona composition. Every predictor of training gains receives the target's Base effects and predicts its effects after training. This tests whether training reshapes effects similarly across models, given each target's own Base responses. It is a diagnostic for models with accessible Base checkpoints. We also test reuse without a Base, across prompts in \S\ref{sec:wording} and directly between models in Appendix~\ref{app:direct-transfer}.

The central predictor is \textbf{adapted reuse}: it keeps the source gains $g^{\rm src}_{ab}$ and fits a single target amplitude $\beta$, $\widehat y^{\rm tgt}=\beta g^{\rm src}_{ab}x^{\rm tgt}$, which allows a uniform change in overall responsiveness. 

A flexible variant with two target parameters $\widehat y^{\rm tgt}=(u+v g^{\rm src}_{ab})x^{\rm tgt}$ adjusts the average level and the attribute differences separately. It can recover shared scaling at $v=0$, so improvement on new questions measures what the source pattern adds beyond target calibration; like the other gain predictors it remains multiplicative in $x$ and cannot predict an effect where the Base effect is zero.

\textbf{Refitting} learns new gains from target questions, and a \textbf{regularized update} lies in between (Figure~\ref{fig:predictability-reuse}; Appendix~\ref{app:shrinkage}).

Three comparisons tell us what can be reused and what benefits from updating, all on the same new target questions. \emph{Reuse minus shared scaling}: does the source relationship still carry predictive value? \emph{Refitting minus reuse}: how much is lost by keeping the source pattern? \emph{Refitting minus shared scaling}: does the target have learnable structure of its own?

\paragraph{Inference.} Intervals resample whole questions (and profile identities where relevant); each study corrects for its own family of comparisons (Appendix~\ref{app:statistics} specifies which analyses refit predictors within each draw). We call a difference \emph{significant} when its corrected interval excludes zero; each interval also states how large an undetected effect could be, so ``no detected advantage'' is read together with its bound (for flexible reuse across models, the median corrected interval half width is 0.46 points in the original analysis). We use 1 point of $Q^2$ as a prespecified magnitude reference, not a validated threshold of practical utility; Appendix~\ref{app:q2-logodds} translates points into log odds.

\section{Training reshapes persona effects predictably}
\label{sec:results}
\subsection{Training does not simply turn persona effects up or down}
Training can make a model more responsive to personas overall while making it less responsive to a particular attribute. Average magnitude ranges from 0.76 to 4.96 times the Base level across pairs, but the average hides the attributes: in OLMo-3, SFT amplifies 157 of 171 combinations of attribute and task, yet gender contrasts \emph{shrink} in all three tasks. Figure~\ref{fig:breadth} shows the same unevenness across categories and models, even though effects generally keep their direction.

\begin{figure}[!htbp]
\centering
\includegraphics[width=0.72\linewidth]{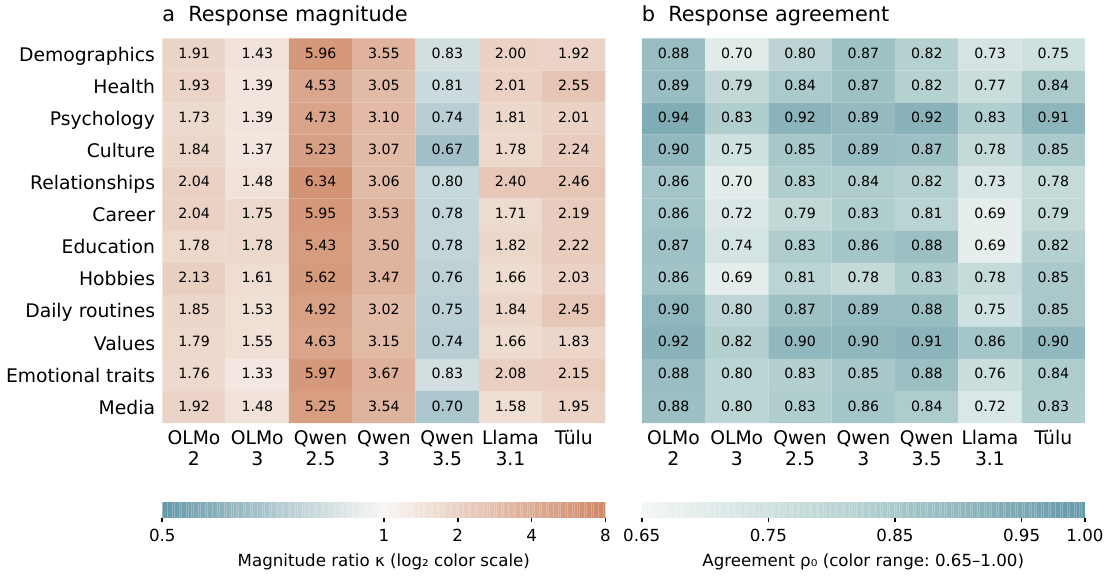}
\caption{\textbf{Training changes some persona effects more than others.} Each cell summarizes one category in one model pair by the median over attributes ($\kappa$ in log space); Qwen2.5 in bfloat16. Left: the magnitude ratio $\kappa$, where one means unchanged strength; colors use a log scale. Right: uncentered agreement $\rho_0$ between Base and trained effects. The summaries cover all 57 attributes.}
\label{fig:breadth}
\end{figure}

\subsection{Prediction improves most reliably in OLMo-3}
Part of this variation predicts new questions. Gains fitted on 181 questions predict 45 new ones better than one shared multiplier by point estimate in six of seven pairs (Table~\ref{tab:prediction}). The improvement in log odds is modest, but it measures predictive value beyond shared scaling. The evidence is strongest in OLMo-3, where the improvement stays significant under the strictest test and under chat interfaces, reversed answer order, and an independently written question set; Qwen2.5 also passes the strictest test, and support elsewhere varies (Appendices~\ref{app:crossmodel}, \ref{app:robustness}).

An additive model with attribute and task terms recovers most of the improvement in the five pairs supported when fitted gains are held fixed (Appendix Table~\ref{tab:olmo-ablation}).

\begin{table}[!htbp]
\centering\small
\caption{\textbf{Prediction gains and scaling across training stages.} The prediction column gives the improvement of separate attribute and task gains (171 per pair, each fitted from the roughly 60 fitting questions of its task) over shared scaling on 45 new questions, in $Q^2$ points, with the weighted root mean square prediction error of both predictors in log odds. Bold values remain significant when fitting uncertainty and correction across seven pairs are included. The stage columns report $Q^2_{\rm shared}$ and the fraction of prediction error remaining after scaling, $R=L_{\rm shared}/L_{\rm identity}$. Lower $R$ means scaling removes more error. The two column groups use distinct evaluations, specified below.}
\label{tab:prediction}
\label{tab:stage-main}
\begin{tabular*}{\linewidth}{@{\extracolsep{\fill}}llrrr@{\hspace{12pt}}rr@{}}
\toprule
& & \multicolumn{3}{c}{Prediction on new questions} & \multicolumn{2}{c}{Stage comparison} \\
\cmidrule(lr){3-5}\cmidrule(l){6-7}
& & & \multicolumn{2}{c}{RMSE (log odds)} & & \\
Model & Transition & $\Delta Q^2$ (points) & shared & gains & $Q^2_{\rm shared}$ & $R$ \\
\midrule
Llama-3.1 & Base $\to$ Instruct & 0.07 & 1.026 & 1.025 & 0.637 & 0.835 \\
\addlinespace[2pt]
OLMo-2 & Base $\to$ SFT & $-0.02$ & 1.079 & 1.080 & 0.787 & 0.633 \\
 & SFT $\to$ DPO & & & & 0.969 & 0.308 \\
\addlinespace[2pt]
OLMo-3 & Base $\to$ SFT & $\mathbf{3.10}$ & 1.054 & 1.010 & 0.639 & 0.966 \\
 & SFT $\to$ DPO & & & & 0.936 & 0.274 \\
\addlinespace[2pt]
Qwen2.5 & Base $\to$ Instruct & $\mathbf{1.28}$ & 3.204 & 3.113 & 0.770 & 0.331 \\
\addlinespace[2pt]
Qwen3 & Base $\to$ Instruct & 0.87 & 3.271 & 3.218 & 0.728 & 0.456 \\
\addlinespace[2pt]
Qwen3.5 & Base $\to$ Instruct & 1.13 & 0.575 & 0.560 & 0.784 & 0.487 \\
\addlinespace[2pt]
T\"ulu & Base $\to$ SFT & 2.11 & 1.094 & 1.053 & 0.755 & 0.636 \\
 & SFT $\to$ DPO & & & & 0.962 & 0.219 \\
\bottomrule
\end{tabular*}
\par\smallskip
{\footnotesize\leftskip=0pt\rightskip=0pt\parfillskip=0pt plus 1fil\noindent Stage columns use cross validation over questions for OLMo-3/T\"ulu and the 181/45 split otherwise; Base$\to$Instruct rows bundle stages; Qwen2.5 in float32. Intervals and results for the final stage: Appendices~\ref{app:crossmodel}, \ref{app:scope-support}.\par}
\end{table}

\subsection{Shared scaling fits the tested DPO steps better}
Across three trajectories, one multiplier describes the step from SFT to DPO far better than the step from Base to SFT (Table~\ref{tab:prediction}): shared scaling leaves 63 to 97\% of the error of predicting no change during SFT but only 22 to 31\% during DPO. Measuring the ratio against each stage's own change improves comparability, but cannot by itself rule out that smaller changes are easier to approximate by scaling; OLMo-3 gives the clearest contrast, because its two stages change effects by similar amounts, yet scaling leaves far less error after DPO (Appendix~\ref{app:scope-support}). Uneven change is not automatically predictable, though: OLMo-2 departs from uniform scaling without a significant gain on new questions.

\section{Predictable is not portable}
\label{sec:transfer}
\subsection{Across models: borrowed gains add little reusable value}
We evaluate 42 directed transfers among seven related checkpoint pairs, not independent replications, retaining the target's Base effects and borrowing the source gains (\S\ref{sec:method}). OLMo-3 provides the clearest case: its own refit gains 3.1 points over shared scaling, while incoming patterns lose 1.8 points on average with one amplitude and gain only 0.19 with two parameters. Refitting significantly beats all six sources under either predictor (Appendix~\ref{app:target-reference}). Qwen2.5 also has structure (1.3 points), but its refitting advantage over flexible reuse is not significant after correction for individual directions. For the other targets, whose structure is small or unresolved, the results mainly show no detected additional reuse benefit. Reuse with one amplitude predicts worse than shared scaling in 40 of 42 directions, significantly in 29 (Figure~\ref{fig:stages}a). Losses also persist under chat interfaces, reversed answer order, and independent questions (Appendix~\ref{app:robustness}). Sharing a Base does not establish useful gain reuse: neither direction between T\"ulu and Llama-3.1 has a significant advantage over shared scaling, with one amplitude or two parameters. Both do outperform shuffled attribute patterns (Appendix~\ref{app:permutation-null}).

Avoiding a transfer loss does not establish a benefit. With two parameters and Qwen2.5 evaluated in float32, no direction differs significantly from shared scaling. The median gain is 0.03 points, and corrected upper bounds are below one point in 39 of 42 directions (40 in bfloat16; Figure~\ref{fig:stages}b). The corresponding counts are 37, 39, and 36 under chat interfaces, averaged answer orders, and independent questions (Appendix~\ref{app:robustness}). These bounds must be read against each target's own structure. The corresponding ratios to target refitting are descriptive and unstable when target gains approach zero (Appendix~\ref{app:target-reference}). Attribute permutations show that source patterns retain some attribute information. Under the tested noise injection model, estimation noise accounts for only a small part of the transfer gap (Appendix~\ref{app:splithalf-noise}). Direct transfer without a Base gives a median $Q^2$ of 0.47, compared with 0.62 to 0.79 for scaling each target's own Base effects (Appendix~\ref{app:direct-transfer}).

\begin{figure}[!t]
\centering
\includegraphics[width=1.0\linewidth]{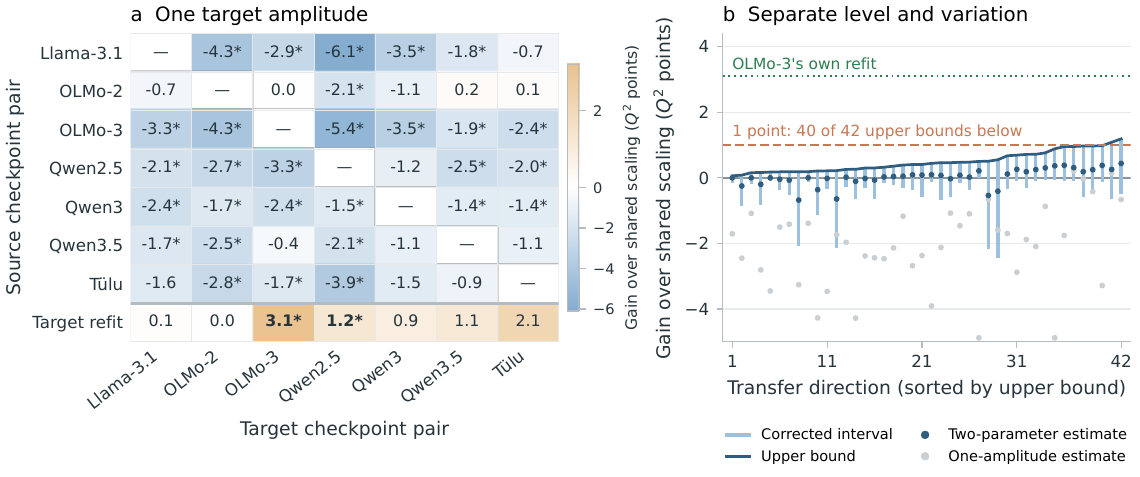}
\caption{\textbf{Borrowed gains hurt with one amplitude and add little with two.} (a) Adapted reuse minus target shared scaling, in $Q^2$ points; rows supply training gains, columns receive them. Asterisks mark intervals excluding zero after correction across 42 directions. The bottom row is each target's own refit, bold when significant across seven targets; it is a reference, not an upper bound. (b) The same 42 directions with two target parameters, sorted by the corrected upper bound of their gain over shared scaling: 40 upper bounds lie below one point with Qwen2.5 in bfloat16 as shown, 39 in float32 (Appendix~\ref{app:controls}). Gray dots show the estimates with one amplitude; lower bounds below $-5$ are clipped. T\"ulu and Llama-3.1 share a base model, so directions are not all independent (Appendix~\ref{app:cross-model-reuse}).}
\label{fig:stages}
\end{figure}

\subsection{Across prompts: rewording carries a cost, other changes need not}
\label{sec:wording}
Changing prompt wording can reduce the predictive value of a source relationship. We fit source gains jointly on two templates and evaluate a target that changes wording, including the answering perspective, examples, or both (Figure~\ref{fig:transfer-refit}). After wording changes, refitting significantly beats reuse with one amplitude in all six pairs. In OLMo-3, T\"ulu, and Llama-3.1, this reuse is significantly worse than shared scaling, and refitting also significantly beats flexible reuse. The latter result holds with one source template and, conditional on the fitted source, with two (Appendices~\ref{app:extensions}, \ref{app:flex-prompt}). Qwen3 retains significant reuse value after wording changes with two source templates. Within a deployed model without a Base reference, refitting also beats unchanged reuse after wording changes in all six models (Appendix~\ref{app:base-free-reuse}).

\begin{figure}[!t]
\centering\includegraphics[width=0.82\linewidth]{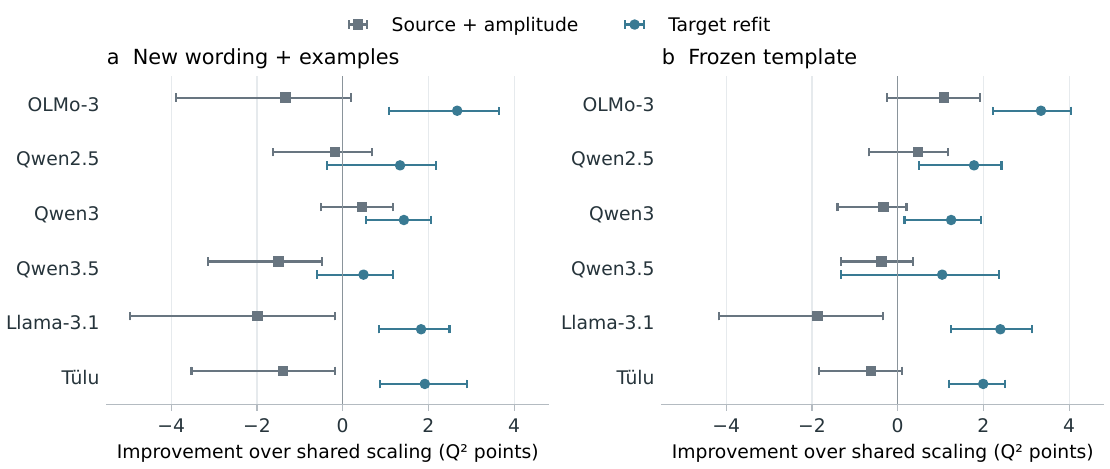}
\caption{\textbf{Refitting improves prediction after prompt changes.} Gray keeps the source pattern and adjusts its amplitude; blue learns new target gains. Values are improvements over target shared scaling on new questions, in $Q^2$ points. Panel (a) changes wording and examples. Panel (b) uses a questionnaire template frozen before the OLMo-3/T\"ulu runs and reused for later pairs. Intervals include fitting uncertainty and correction within each study (Appendices~\ref{app:extensions} and~\ref{app:shrinkage}).}
\label{fig:transfer-refit}
\end{figure}

Coefficient similarity hides this cost: uncentered source and target gain correlations of 0.97 to 0.99 fall to 0.15 to 0.63 once the shared mean is removed, so similarity is no substitute for a transfer test (Appendix~\ref{app:similarity}).

Reuse can nevertheless help. When only the examples change, the pattern fitted on two source templates improves on shared scaling in four of six pairs with either one amplitude or two parameters, conditional on the fitted source. A separate study of 16 psychological, value, and emotion attributes finds useful reuse across five country contexts in five models (Appendix~\ref{app:extensions}). Thus the type of prompt change matters.

\subsection{When reuse fails: update, and at what cost}
\label{sec:remeasure}
How much target data does an update need? At eight questions per attribute, flexible reuse outperforms the update shrunk toward the source in median in 11 of 12 settings. The ordering reverses in 11 settings at 32 questions and all 12 at 64 (Figure~\ref{fig:targetbudget}). Unregularized refitting overtakes reuse with one amplitude at 16 to 32 questions in eleven settings and at 64 in the remaining one. Regularized updates beat reuse with one amplitude at every tested budget and nearly match refitting from 32 questions. We detect no consistent advantage from shrinking toward the source rather than a shared gain: corrected intervals range from $-0.27$ to 0.17 points in OLMo-3 to $-0.91$ to 0.67 in Qwen3.5 (Appendix~\ref{app:shrinkage}).

A template frozen before data collection confirms the refitting advantage prospectively in OLMo-3 and T\"ulu; later runs of Llama-3.1, Qwen2.5, and Qwen3 agree, and Qwen3.5 is unresolved (Appendix Table~\ref{tab:prospective}).

\begin{figure}[!t]
\centering\includegraphics[width=0.86\linewidth]{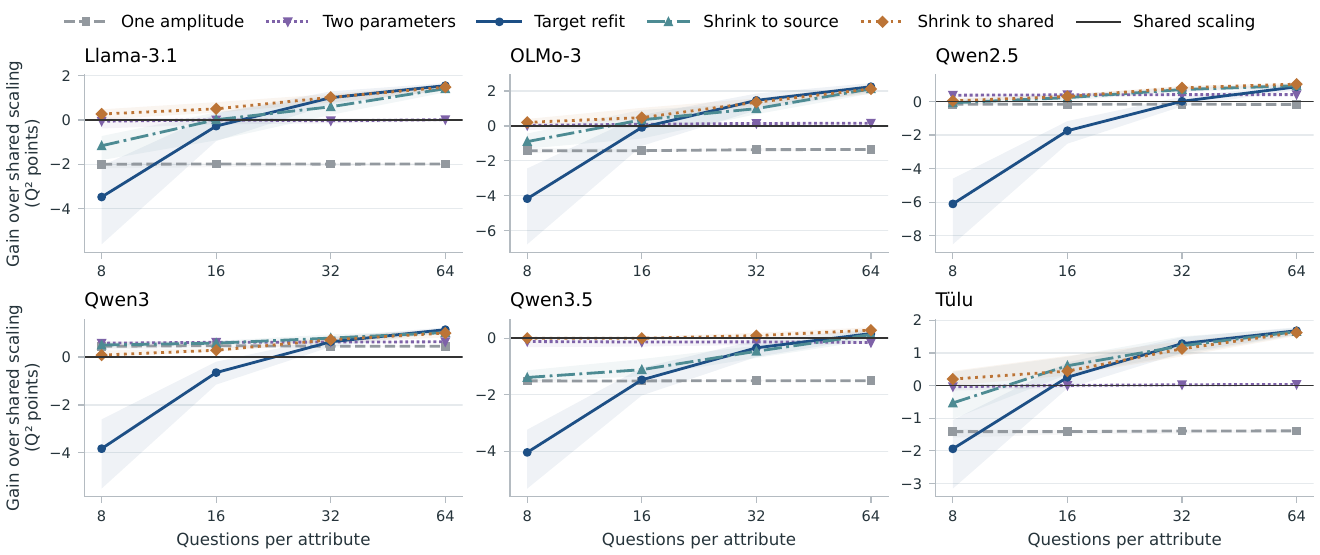}
\caption{\textbf{Regularization helps updates with limited data; the source prior shows no detected advantage.} On the template with joint wording and example changes, all six strategies use the same subsets of target questions. Lines show median improvements over shared scaling, in $Q^2$ points, across 200 subsets; shading is the middle 50\%. Penalties are selected on target fitting questions only. Results for the frozen template: Appendix~\ref{app:shrinkage}.}
\label{fig:targetbudget}
\end{figure}

\section{Composition changes what an attribute does}
\label{sec:composition}
We test whether effects measured for isolated attributes remain stable when other persona information is added.

\subsection{Training changes relative cue influence differently across models}
We pair a country cue, USA versus India, with a risk, planning, or social preference that either bears on the question or not, and compare the size of the country effect between the two (Appendix~\ref{app:composition}). Training shifts this relative influence in opposite directions: down in T\"ulu, Qwen2.5, and Qwen3, up in Qwen3.5, unresolved in OLMo-3 and Llama-3.1 (Figure~\ref{fig:composition}a). The effect is context dependent and model specific; the data do not support a single mechanism.

\begin{figure}[!htbp]
\centering\includegraphics[width=0.76\linewidth]{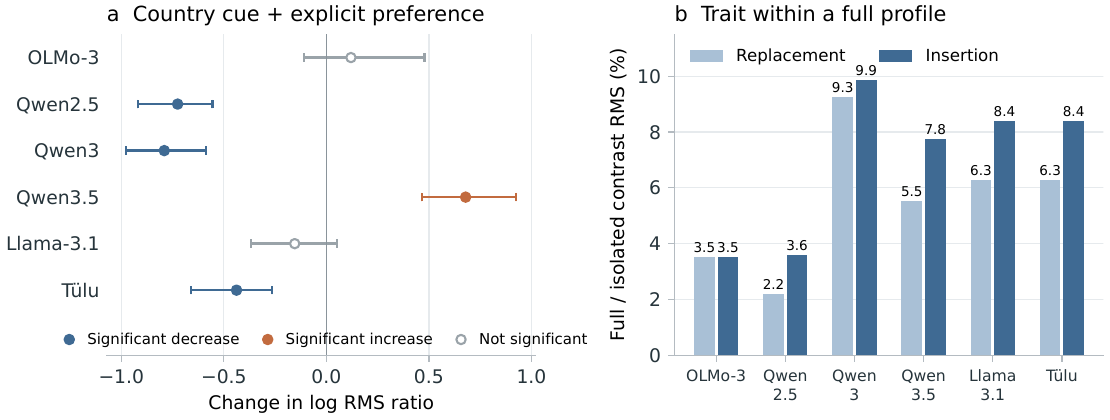}
\caption{\textbf{An attribute acts differently when other persona information is present.} Panel (a) shows how training changes the country effect under relevant versus unrelated preferences, measured as the difference of log RMS ratios. Negative values mean a loss of relative country influence under relevant preferences; 98.75\% intervals correct across four comparisons within each study, and filled markers exclude zero. Panel (b) shows how much of an isolated trait effect remains inside a full profile, as a percentage, for Base models. Panel (b) uses float32 for Qwen2.5 and bfloat16 for the other pairs.}
\label{fig:composition}
\end{figure}

\subsection{Profiles change more than effect magnitude}
We insert one trait statement into, or replace one field of, a full DeepPersona profile \citep{wang2025deeppersona}, holding the rest fixed. For Base models, profile effects retain only 2 to 10\% of isolated effect magnitude in bfloat16, and 3\% for Qwen2.5 in float32 (Figure~\ref{fig:composition}b). These small signals require caution about numerical precision.

The Qwen2.5 float32 insertion results show that rescaling alone is insufficient: target refitting improves on shared scaling by 12.7 $Q^2$ points and on both reuse predictors. Several other models show the same ordering in bfloat16. For replacement, no refitting advantage is significant. These are supporting tests of context dependence; they do not distinguish semantic interactions from effects of prompt length or insertion position. Full comparisons, absolute errors, and precision checks appear in Appendices~\ref{app:composition} and~\ref{app:controls}.

\section{A boundary: responsive is not faithful}
\label{sec:human}
The preceding sections test predictions of model behavior. Human fidelity requires a separate evaluation. On OpinionQA's U.S. survey distributions \citep{santurkar2023opinions}, we compare selecting checkpoints by persona responsiveness with selecting by human agreement. Selection regret is its extra evaluation loss relative to selecting by human agreement on development questions (Appendix~\ref{app:human}).

\begin{figure}[!htbp]
\centering\includegraphics[width=0.72\linewidth]{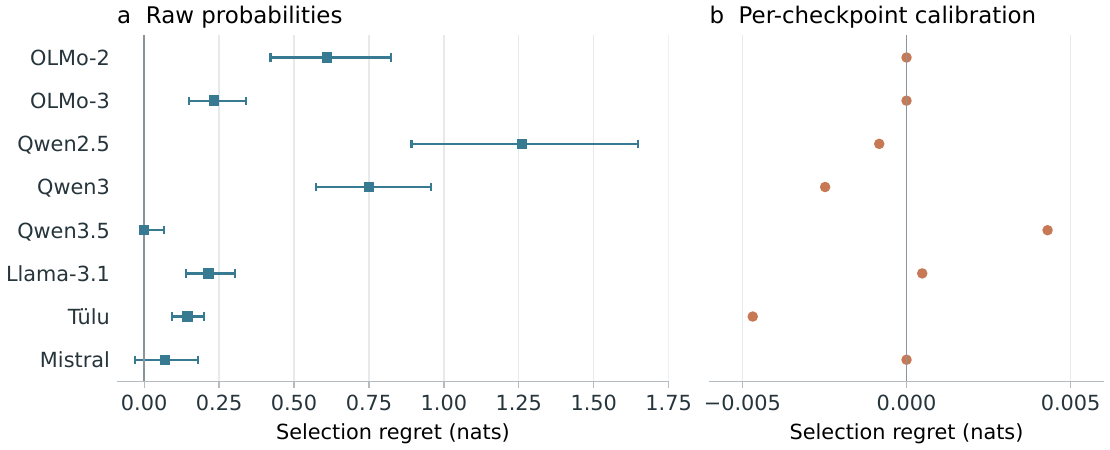}
\caption{\textbf{Choosing the more responsive checkpoint can increase survey loss.} Values show the extra loss on new age and sex survey questions from selecting checkpoints by sensitivity rather than human agreement. Blue uses raw probabilities and corrected intervals; orange applies calibration fitted on human labels to each checkpoint. The raw and calibrated panels use different scales. Orange points are descriptive estimates without displayed intervals.}
\label{fig:calibration}
\end{figure}

\paragraph{Checkpoint selection in these comparisons.} In six of eight checkpoint comparisons, the checkpoint more sensitive to personas is the worse simulator, costing up to 1.41 nats of evaluation loss in Qwen2.5 across eleven survey dimensions (Figure~\ref{fig:calibration} shows the age and sex comparisons); in Qwen3.5 both rules agree, and Mistral is inconclusive. The more sensitive checkpoint is always the trained one. This is a boundary case for using sensitivity to select checkpoints, not evidence that sensitivity independently predicts poor fit.

\paragraph{Calibration reduces selection costs but not group-ordering errors.} One temperature per checkpoint, fitted on human labels, reduces regret to at most 0.005 nats in all eight comparisons. Calibration on 32 questions with group distribution labels comes within 0.002 nats of using all 84, averaged over subsets. Calibrated persona prompts still beat matched prompts without an attribute (Appendix~\ref{app:intervention}). On a separate suite of four survey dimensions, models recover the direction of human group differences only 57 to 70\% of the time, against 95\% agreement between two halves of the human respondents. Base checkpoints recover 54 to 70\%; training significantly changes recovery only in Llama-3.1, where it improves. A positive temperature preserves group ordering within each question. On this suite, calibrated models close only 12 to 29\% of the cross entropy gap between a 50/50 guess and the human distributions (Appendix~\ref{app:human}).

\section{Related work}
\label{sec:related}
\paragraph{Simulating human samples.} Conditioned models can reproduce aggregate survey patterns and classic experiments \citep{argyle2023outofone,aher2023simulate}, but align unevenly with demographic groups and are shaped by answer order \citep{santurkar2023opinions,dominguezolmedo2024questioning}.

\paragraph{Persona attributes and composition.} Attribute relevance shapes simulation quality \citep{hu2024persona,froehling2026attributes}. Rich backstories improve matching of response distributions, while interviews support individual simulation \citep{moon2024anthology,park2024thousand}. Combining attributes can also produce interactions, stable configurations, and collapse that depends on the task \citep{bhattacharyya2026personality,chameleon2026limit}. Our controlled edits test whether relationships learned from isolated attributes still predict behavior within richer personas.

\paragraph{Robustness, steering, and calibration.} Persona and prompt studies document sensitivity to role format, phrasing, and formatting \citep{lutz2025prompt,beck2024sociodemographic,sclar2024formatspread}. Other work quantifies prompt steerability \citep{miehling2025steerability}, distinguishes it from distributional coverage and alignment \citep{sorensen2026spectrum}, and traces diversity in advice personas across OLMo-3 training stages \citep{kumar2026advice}. We test what source relationships add beyond target calibration and regularization on the same target questions. Work on constrained choices and distributional evaluation \citep{dominguezolmedo2024questioning,meister2025distributional} motivates our readout checks and calibration (Appendix~\ref{app:literature}).

\section{Discussion and conclusion}
Predictive structure does not ensure useful transfer. OLMo-3 provides the clearest comparison: it has the largest target refitting gain, and its own attribute gains predict new questions better than every tested borrowed pattern. Across model transfers, flexible reuse removes significant losses without a significant benefit over shared scaling, which separates source value from target calibration.

The evidence concerns forced-choice log odds from seven open 7 to 9B checkpoint pairs on constructed questions. Our conclusions concern multiplicative gains, which cannot represent an effect that post-training creates where the Base has none; transfer through other representations is untested. Profile and survey studies provide boundary evidence (Appendix~\ref{app:limitations}).

\paragraph{A practical comparison.} Given access to the required checkpoints, compare shared scaling, flexible reuse, and a regularized target update on questions separate from fitting. In the tested prompt transfers, regularized updates outperform both reuse strategies in median at 64 questions per attribute (about 36,700 fitting prompts per checkpoint pair; model-transfer budgets are untested). The actionable test is what source information adds beyond target calibration and regularization.

\label{maintextend}
\bibliography{references}
\bibliographystyle{iclr2027_conference}
\clearpage
\appendix
\raggedbottom
\setlength{\parskip}{.5pc}
\makeatletter
\renewcommand\section{\@startsection{section}{1}{\z@}{-1.6ex plus -0.3ex minus -.2ex}{1.0ex plus 0.2ex minus 0.2ex}{\large\sc\raggedright}}
\makeatother
\section{Scope and limitations}
\label{app:limitations}
Our experiments use fixed model releases, constructed questions, and forced-choice A/B readouts, and do not identify the mechanism behind the changes we measure. T\"ulu and Llama share a pretrained backbone, and all pairs share one question pool, so comparisons across pairs are not independent replications. Behavioral readouts are first-token log odds over two options; first-token probabilities can differ from generated answers \citep{wang2024answerc}, and only the human selection study includes a sampled-answer check (Appendix~\ref{app:human}). Chat and answer-order checks keep the forced-choice readout, and the independent question set stays within the same three task definitions (Appendix~\ref{app:robustness}). Our tests hold out questions rather than people, and human evaluation uses historical U.S. surveys; new identities, open-ended generation, other languages, and non-U.S. populations remain to be tested.

Profile effects are small enough to be sensitive to numerical precision. A full float32 rerun changed Qwen2.5's replacement results (Appendix~\ref{app:controls}); the other profile pairs are still evaluated in bfloat16, so their profile results are supporting evidence only. The cross-model and prompt conclusions do not depend on the profile study. The training-gain tests need a Base checkpoint; the deployed-model variant of Appendix~\ref{app:base-free-reuse} does not, but it predicts prompt changes rather than training changes. We have not tested whether better structural prediction implies better human simulation: the output intervention in Appendix~\ref{app:intervention} uses different survey attributes and gains.

\section{Materials and models}
\label{app:design}
Table~\ref{tab:design} maps each experiment to what it changes and how it is evaluated.

\begin{table}[htbp]
\centering\small\renewcommand{\arraystretch}{1.2}
\caption{\textbf{What each experiment asks of a persona measurement.} Behavioral studies predict model contrasts; survey studies test agreement with human response distributions. Core prediction and transfer tests share a question pool; cue competition and independent replication use separate questions.}
\label{tab:design}
\begin{tabularx}{\linewidth}{@{}>{\raggedright\arraybackslash}p{.19\linewidth}>{\raggedright\arraybackslash}X>{\raggedright\arraybackslash}X@{}}
\toprule
Study & What changes & Evaluation unit \\
\midrule
Attribute inventory & 57 attributes; matched checkpoints & Question folds; then 45 new questions \\
Training trajectory & Base, SFT, DPO, final stages & Questions used to evaluate fitted gains \\
Context transfer & Model family, five countries, prompt format & Target questions; gains use source or target training data \\
Cue competition & Country crossed with preference relevance and strength & 24 constructed questions; bootstrap, no prediction holdout \\
Full profiles & One field or trait line in a fixed profile & Question folds within represented profiles \\
Human simulation & Checkpoint selection and calibration & Survey questions with group-distribution targets \\
\bottomrule
\end{tabularx}
\end{table}

\paragraph{Coverage by study.} The seven behavioral pairs are OLMo-2, OLMo-3, T\"ulu, Qwen2.5, Qwen3, Qwen3.5, and Llama-3.1. Prompt and budget studies use the same set without OLMo-2; stage comparisons use OLMo-2, OLMo-3, and T\"ulu. The eight human selection comparisons add Mistral; OLMo selection considers four checkpoints, the others two. The respondent-split audit uses seven pairs, with Mistral in place of OLMo-2.

\paragraph{Where to check each claim.} Table~\ref{tab:evidence-map} links each main claim to its comparison, uncertainty treatment, and full results. Correction families follow the individual studies and are never pooled.

\begin{table}[htbp]
\centering\footnotesize\renewcommand{\arraystretch}{1.15}
\caption{\textbf{Evidence behind the main claims.} Each row gives the models covered, what is compared, how uncertainty is handled, and where the complete results appear.}
\label{tab:evidence-map}
\begin{tabularx}{\linewidth}{@{}>{\raggedright\arraybackslash}p{.14\linewidth}>{\raggedright\arraybackslash}p{.12\linewidth}>{\raggedright\arraybackslash}X>{\raggedright\arraybackslash}X>{\raggedright\arraybackslash}p{.11\linewidth}@{}}
\toprule
Claim & Coverage & Comparison & Uncertainty & Details \\
\midrule
New-question structure & 7 pairs & Cell gains versus shared scaling & Fitting and evaluation questions resampled with refitting; corrected across 7 pairs; fixed-gain intervals also reported & App.~\ref{app:crossmodel}, \ref{app:q2-logodds} \\
Stage differences & 3 trajectories & Shared scaling at SFT versus DPO & Question folds, resampling with refitting for OLMo-3 and T\"ulu; OLMo-2 descriptive & App.~\ref{app:scope-support} \\
Cross-model reuse & 42 directions & One-amplitude and two-parameter reuse versus shared scaling; refitting & Source and target refitted in each draw; corrected across 42 directions; permutation and noise controls & App.~\ref{app:cross-model-reuse} \\
Prompt reuse & 6 pairs; deployed variant 6 models & Adapted and flexible reuse, shared scaling, refitting & Study-specific corrections; prospective and extension runs marked & App.~\ref{app:extensions} \\
Updating costs & 6 pairs & Shared scaling, reuse, refitting, source- and shared-shrunk updates & Penalty chosen on inner folds; corrected across 12 settings; budget curves descriptive & \S\ref{sec:remeasure}; App.~\ref{app:shrinkage} \\
Robustness and precision & 7 pairs & Chat interfaces, answer order, independent questions, extreme attributes, float32 & Frozen rules; study-specific corrections & App.~\ref{app:robustness} \\
Composition & 6 pairs & Relevance change of a country cue; profile reuse versus refitting & Question and identity resampling; study-specific corrections & App.~\ref{app:composition} \\
Human fidelity & 8 selection comparisons & Sensitivity versus human-loss selection; calibration; group directions & Question resampling; respondent-split audit against chance & App.~\ref{app:human} \\
\bottomrule
\end{tabularx}
\end{table}

\paragraph{Attribute inventory.} The inventory contains 57 attributes in 12 categories (Table~\ref{tab:inventory}), rendered through 287 statements. We vary one attribute at a time rather than crossing identities. Eighteen stated ages are reduced to six analysis levels by averaging fixed groups of three log-odds values; other attributes keep their levels. All 287 statements are listed verbatim in the supplementary material (Section~S1).
\label{app:statements}

\begin{table}[htbp]
\centering\footnotesize\renewcommand{\arraystretch}{1.15}
\caption{\textbf{The attributes tested, with one example statement per category.} Statements are experimental conditions rendered through fixed sentence templates, not validated scales.}
\label{tab:inventory}
\begin{tabularx}{\linewidth}{@{}>{\raggedright\arraybackslash}p{.19\linewidth}>{\raggedright\arraybackslash}X>{\raggedright\arraybackslash}p{.30\linewidth}@{}}
\toprule
Category & Attributes & Example statement \\
\midrule
Demographic & age group, gender, country & You grew up and live in India. \\
Health and physical status & health status, chronic constraint, mobility, fitness, sleep quality, dietary constraint & You use a wheelchair for mobility. \\
Psychological and cognitive & social orientation, risk tolerance, cognitive style, life attitude, stress state & You accept moderate risk. \\
Cultural and social context & heritage engagement, language background, cross-cultural experience, religious affiliation, ritual participation & You regularly use two languages in daily life. \\
Relationships and networks & family structure, caregiving, network structure, living arrangement, social support & You regularly care for an elderly relative. \\
Career and work identity & occupation, career stage, job responsibility, employment form, work environment & You work as a self-employed freelancer. \\
Education and learning & education level, enrollment status, education interruption, study field, learning method & You prefer learning through hands-on practice. \\
Hobbies and lifestyle & outdoor interest, arts and entertainment, practical interest & A main leisure interest of yours is pottery. \\
Daily routines & daily schedule, meal habit, exercise routine, time priority & You exercise almost every day. \\
Core values and beliefs & ethical priority, interpersonal value, worldview, religious belief, life priority & A core value that guides your choices is fairness. \\
Emotional and relational skills & emotional awareness, emotion regulation, coping mechanism, resilience, empathy, conflict skill & After setbacks, you generally recover quickly. \\
Media consumption & content interest, media format, information source, media frequency, media participation & You use social media daily. \\
\bottomrule
\end{tabularx}
\end{table}

\paragraph{Questions and prompts.}
The 181 fitting questions comprise 60 opinion, 61 decision, and 60 social judgment items; the 45 evaluation questions add 15 per task. The first 45 evaluation items were written after the original within-pool result was known and frozen before inference; a further 45 items for the independent replication were written without access to any result (Appendix~\ref{app:robustness}). The standard completion template is
\begin{quote}\small\ttfamily
Persona: [attribute statement]\\
Situation: [question or scenario]\\
Which option would you choose?\\
A. [option A]\\
B. [option B]\\
Answer with only A or B.\\
Answer:
\end{quote}
Four balanced examples precede the item, with answers B, A, B, A. We score the space-prefixed answer tokens, and both checkpoints of a pair use the same interface. The social judgment category shares the common choice instruction, so its label marks a topic grouping rather than a distinct cognitive process. Typical items are an opinion statement to agree or disagree with (``When jobs are scarce, men should have more right to a job than women''), a decision between two activities (joining a neighborhood clean-up or visiting an exhibition), and a judgment of whether an observed behavior was acceptable.

The joint-change target changes both wording and examples: each example reads ``Profile of the respondent: \dots\ Scenario: \dots\ Possible answers: A.\ \dots\ B.\ \dots\ Given the profile, the respondent's selection (A or B):'', the four examples are new (answer pattern A, B, A, B), and the persona is described in the third person. The wording-only and example-only targets change one of the two components. The complete joint-change template, with its four examples, appears in the supplementary material (Section~S1).

\paragraph{Stimulus provenance.} Attribute fields and values follow the taxonomy of DeepPersona \citep{wang2025deeppersona}, itself produced by a generative pipeline; each value is rendered through a fixed sentence template. The fitting questions and the first 45 evaluation questions were constructed for this study and frozen before inference; the independent set of 45 questions was written by a separate model-based writer with no access to results. No statement or question was pretested with human respondents.

\paragraph{Models.}
The primary matched pair is \texttt{Olmo-3-1025-7B} and \texttt{Olmo-3-7B-Instruct-SFT} \citep{allenai2025base,allenai2025sft}. Stage comparisons add DPO and final checkpoints for OLMo-3 and T\"ulu-3 \citep{lambert2024tulu}. Base/Instruct comparisons bundle all intervening training stages. Model coverage varies by experiment rather than forming a complete factorial (Table~\ref{tab:checkpoints}).

\begin{table}[htbp]\centering\small\renewcommand{\arraystretch}{1.08}
\caption{\textbf{Exact checkpoints.} Hugging Face identifiers of every evaluated checkpoint. T\"ulu and Llama-3.1 share the Llama-3.1-8B Base checkpoint, evaluated from a weight-verified mirror.}
\label{tab:checkpoints}
\footnotesize\begin{tabularx}{\linewidth}{@{}ll>{\raggedright\arraybackslash}X@{}}
\toprule
Pair & Base & Post-trained \\
\midrule
Llama-3.1 & \texttt{meta-llama/Llama-3.1-8B} & \texttt{meta-llama/Llama-3.1-8B-Instruct} \\
Mistral & \texttt{mistralai/Mistral-7B-v0.3} & \texttt{mistralai/Mistral-7B-Instruct-v0.3} \\
OLMo-2 & \texttt{allenai/OLMo-2-1124-7B} & \texttt{allenai/OLMo-2-1124-7B-SFT}, \texttt{-DPO}, \texttt{-Instruct} \\
OLMo-3 & \texttt{allenai/Olmo-3-1025-7B} & \texttt{allenai/Olmo-3-7B-Instruct-SFT}, \texttt{-DPO}, \texttt{Olmo-3-7B-Instruct} \\
Qwen2.5 & \texttt{Qwen/Qwen2.5-7B} & \texttt{Qwen/Qwen2.5-7B-Instruct} \\
Qwen3 & \texttt{Qwen/Qwen3-8B-Base} & \texttt{Qwen/Qwen3-8B} \\
Qwen3.5 & \texttt{Qwen/Qwen3.5-9B-Base} & \texttt{Qwen/Qwen3.5-9B} \\
T\"ulu & \texttt{meta-llama/Llama-3.1-8B} & \texttt{allenai/Llama-3.1-Tulu-3-8B-SFT}, \texttt{-DPO}, \texttt{Llama-3.1-Tulu-3-8B} \\
\bottomrule
\end{tabularx}
\end{table}

\paragraph{Country contexts and profiles.}
Country-conditioned experiments use the United States, Argentina, Germany, Australia, and India, prompted in English. Full profiles come from 299 synthetic DeepPersona profiles in these five countries (60 per country, 59 for India). Screening excludes profiles whose remaining text repeats the edited attribute and yields one panel per attribute; the panels contain 113 distinct identities (Table~\ref{tab:profile-panels}). A replacement edit substitutes one existing value field with a taxonomy statement; an insertion edit adds one psychological or emotional trait line after the values section, leaving the rest of the profile unchanged. An edit manifest verifies that paired prompts differ only at the edited line. The longest profile prompt has 3,836 tokens, within the 4,096 cap.

\begin{table}[htbp]\centering\footnotesize\renewcommand{\arraystretch}{1.05}
\caption{\textbf{Profile identities per attribute panel.} Counts are evaluated identities, not answer conditions; an identity can appear in several panels. Small panels, especially life priority, limit profile-level precision.}
\label{tab:profile-panels}
\begin{tabular}{@{}llrrrrrr@{}}
\toprule
Edit & Attribute & US & AR & DE & AU & IN & Total \\
\midrule
Replacement & ethical priority & 1 & 3 & 2 & 1 & 2 & 9 \\
 & life priority & 0 & 0 & 0 & 1 & 1 & 2 \\
 & religious belief & 4 & 4 & 4 & 4 & 4 & 20 \\
 & worldview & 1 & 2 & 2 & 1 & 0 & 6 \\
\addlinespace[2pt]
Insertion & cognitive style & 3 & 1 & 3 & 2 & 7 & 16 \\
 & ten other traits$^{*}$ & 4 & 4 & 4 & 4 & 4 & 20 each \\
\bottomrule
\end{tabular}
\par\smallskip{\footnotesize $^{*}$Social orientation, risk tolerance, life attitude, stress state, emotional awareness, emotion regulation, coping mechanism, resilience, empathy, conflict skill.}
\end{table}

\section{Estimation and inference}
\label{app:math}
This section collects the definitions, the one formal result the paper relies on, and the uncertainty procedures. Weights are fixed and nonnegative, all norms use the same observations and weights, fits use training questions, and scores use separate evaluation questions. Further derivations appear in the supplementary material (Section~S2).

\subsection{Readout, summaries, and fits}
For persona $P$ and question $q$, the answer log odds and the probability normalized over A and B are
\begin{equation}
 z_m(q,P)=\ell_m(A\mid q,P)-\ell_m(B\mid q,P),\qquad
 p_m(q,P)=\sigmoid\!\left(z_m(q,P)\right),
 \label{eq:readout}
\end{equation}
where $\ell_m$ is the answer-token log probability. For earlier and later contrasts $x,y$, Figure~\ref{fig:breadth} summarizes magnitude by $\kappa$ and uncentered agreement by $\rho_0$; their product is the projection coefficient $\alpha$:
\begin{equation}
 \kappa=\frac{\|y\|_w}{\|x\|_w},\qquad
 \rho_0=\frac{\langle x,y\rangle_w}{\|x\|_w\|y\|_w},\qquad
 \alpha=\frac{\langle x,y\rangle_w}{\|x\|_w^2}=\kappa\rho_0.
 \label{eq:decomposition}
\end{equation}
\paragraph{Contrasts and weights.} For each question and attribute, the analysis levels are the attribute's statements, except that the 18 stated ages are averaged in fixed groups of three into six levels. Contrasts are all unordered level pairs $i<j$, and each attribute contributes the mean over its pairs, so attributes with more levels do not receive more weight. Complete pairwise contrasts encode centered variation across levels: with equal level weights, the mean squared contrast over the $G(G-1)/2$ pairs is $2G/(G-1)$ times the variance across levels. Each attribute and question receives the weight
\begin{equation}
 w_{q,a}=\underbrace{\frac{1}{12}\cdot\frac{1}{3}\cdot\frac{1}{n_{s(a)}}}_{\text{category, subcategory, attribute}}\;\cdot\;\underbrace{\frac{1}{3}\cdot\frac{1}{n_{b(q)}}}_{\text{task, question}},
 \label{eq:weights}
\end{equation}
where $n_{s(a)}$ is the number of attributes in the subcategory of $a$ and $n_{b(q)}$ the number of questions in the task of $q$ within the split. The 12 categories, their three subcategories, the three tasks, and the questions within a task are thus weighted equally, and the weights sum to one. Pairs share observations and are not independent samples; dependence is handled by resampling whole questions.

Every gain fit is weighted least squares on paired contrasts. For a collection of contrasts with moments $A=\sum_iw_ix_i^2$, $B=\sum_iw_ix_iy_i$, and $C=\sum_iw_iy_i^2$, the squared error of a scalar gain and its minimizer are
\begin{equation}
 \mathcal E(g)=C-2gB+g^2A,
 \qquad \widehat g=B/A\quad(A>0).
 \label{eq:momentfit}
\end{equation}
Shared fits pool the moments before division; cell fits use the moments of each attribute and task cell. The moments reproduce the squared error of enumerating contrasts exactly, and an independent implementation on raw contrast pairs reproduces all principal estimates.

\subsection{Predictors}
With source gains $\widehat g^{s}$, target Base contrasts $x^t$, and target-fitted parameters, the reuse predictors are
\begin{equation}
 \widehat y^{\rm source}=\widehat g^{s}_{ab}x^t,\qquad
 \widehat y^{\rm adapt}=\widehat\beta_t\,\widehat g^{s}_{ab}x^t,\qquad
 \widehat y^{\rm flex}=(\widehat\lambda_0+\widehat\lambda_1\widehat g^{s}_{ab})\,x^t,\qquad
 \widehat y^{\rm refit}=\widehat g^{t}_{ab}x^t.
 \label{eq:transfer}
\end{equation}
Source-only reuse fits no parameters on target training data. Adapted reuse fits one amplitude, $\widehat\beta_t=\langle u,y^t\rangle_w/\|u\|_w^2$ with $u=\widehat g^s x^t$ on target training questions. Flexible reuse fits two parameters and spans the same predictions as $[\lambda_0'+\lambda_1(g^{s}_{ab}-\bar g^{s})]x^t$, so it can keep the source's level, its attribute differences, both, or neither; it contains shared scaling at $\lambda_1=0$. Refitting estimates new cell gains. The regularized updates of Appendix~\ref{app:shrinkage} interpolate between refitting and reuse.

When the source is observed under several prompt formats, one common pattern $g_{ab}$ is fitted jointly with a format amplitude $s_f$ by alternating weighted least squares on $\sum_f\sum_{q,a,ij}w\,(y-s_fg_{ab}x)^2$; only the products $s_fg_{ab}$ are identified.

\subsection{Scale invariance}
A uniform change in response strength cannot manufacture an advantage for separate gains.
\begin{proposition}[Invariance to global response scale]
\label{prop:invariance}
For fixed weights and question splits, replacing $x$ by $sx$ and $y$ by $ty$, with $s,t>0$, leaves the held-out $Q^2$ of each gain model unchanged when its gains are refitted.
\end{proposition}
\begin{proof}
Each gain class of Equation~\ref{eq:gains} (shared, attribute, task, additive, cell) is closed under multiplication by a nonzero scalar. With $g'=(t/s)g$, the training error satisfies $\mathcal E'_{\rm tr}(g')=\sum_iw_i(ty_i-g'_isx_i)^2=t^2\mathcal E_{\rm tr}(g)$, and $g\mapsto(t/s)g$ is a bijection on the class, so minimizers map to minimizers (under a scale-equivariant rule such as minimum norm when they are not unique). Hence $\widehat y'_{\rm te}=t\,\widehat y_{\rm te}$, and $t$ cancels from numerator and denominator of $Q^2$. Differences between classes are preserved as well, for every fixed split and every paired resample.
\end{proof}
The result requires refitting; it does not license rescaling a target while keeping an unadjusted source coefficient. Fixed penalties, nonuniform temperatures, and nonlinear probability transformations need separate analysis, and human cross entropy is not temperature invariant.

\subsection{Reading $Q^2$ and similarity}
\label{app:q2-logodds}
\paragraph{Points in log odds.} Because $Q^2$ uses the zero prediction as reference, its points are relative. On the 45 new questions, the weighted root mean square of the trained contrasts is 1.72 log odds in OLMo-3, 2.05 in T\"ulu, and 6.68 in Qwen2.5. Shared scaling leaves a root mean square error of 1.054 in OLMo-3, and attribute and task gains reduce it to 1.010: the 3.1-point improvement removes about 4\% of the typical error and 8\% of the squared error. The corresponding reductions are from 1.094 to 1.053 in T\"ulu and from 3.204 to 3.113 in Qwen2.5 (float32). The structure is reproducible but modest in size; its value lies in what it predicts and what fails to transfer.

\paragraph{Why high gain correlation does not guarantee transfer.} Take two equally weighted cells, a source pattern $g_s=(1+\epsilon,1-\epsilon)$ and a target pattern $g_t=(1-\epsilon,1+\epsilon)$, with target contrasts $x^t=(1,1)$ and $y^t=g_t$. Their uncentered correlation is $(1-\epsilon^2)/(1+\epsilon^2)$, about 0.98 at $\epsilon=0.1$, yet adapted reuse has error $4\epsilon^2/(1+\epsilon^2)$ against $\epsilon^2$ for shared scaling and zero for refitting. A dominant common component hides opposite attribute differences, which is why Appendix~\ref{app:similarity} reports centered correlations and why transfer is judged on target predictions.

\paragraph{Attenuation and structure in profiles.} If profile contrasts were uniform rescalings of short-statement contrasts, $x^{\rm profile}=sx^{\rm short}$ and $y^{\rm profile}=ty^{\rm short}$, the Base and post-trained amplitude ratios would be $s$ and $t$ while, by Proposition~\ref{prop:invariance}, every refitted $\Delta Q^2$ would be unchanged. Attenuation alone therefore neither proves nor rules out predictive structure, and the measured amplitude ratios are not signal-to-noise ratios. A given $\Delta Q^2$ corresponds to a smaller absolute error reduction when profile contrasts are weak.

\subsection{Splits, resampling, and multiplicity}
\label{app:statistics}
\paragraph{Analysis status.} The primary comparison of each study was specified before its analysis. Two kinds of follow-up are distinguished. Checks that required new model inference, each with a protocol frozen before the runs, are the chat interfaces, the reversed answer order, the independent question set, the frozen questionnaire template (prospective for OLMo-3 and T\"ulu, later extensions for the other pairs), and the float32 reruns of the Qwen2.5 profile study and core screen. Retrospective analyses of saved outputs, each with a protocol and decision rule frozen with a SHA-256 digest before computation, are the permutation control, the half-sample and noise-injection analyses, flexible reuse across models, prompts, and profiles (including under chat interfaces, averaged answer orders, the independent question set, the two-template source, and at each budget), the deployed-model variant, the extreme-attribute check, direct transfer without a Base, and the Base control of the group-direction audit; the regularized updates, budget curves, target-reference breakdown, matched-stage subset, and centered correlations are retrospective and descriptive or study-specific. The Qwen2.5 profile float32 rerun followed a frozen rule under which float32 results replace bfloat16 results whenever any comparison changes significance; the Qwen2.5 core-screen float32 rerun was frozen with the rule that its results are primary regardless of significance changes (Appendix~\ref{app:controls}). Budget and allocation curves are descriptive. Each study applies its own correction, and corrections are never pooled across studies.

\paragraph{New questions and stages.} Gains are fitted on 181 questions and scored on 45 new ones; every condition of a question stays in one split. Two uncertainty scopes are reported. The narrower one holds fitted gains fixed and resamples evaluation questions. The broader one resamples both fitting and evaluation questions within task, refits every predictor in each draw, and corrects across the seven pairs; it determines the boldface of Table~\ref{tab:prediction}. Stage comparisons for OLMo-3 and T\"ulu use five question-grouped folds on the 181 questions with 5,000 resamples within task and fold, refitting all predictors, and 99\% intervals for adjacent-stage contrasts.

\paragraph{Transfers and prompts.} The original single-amplitude cross-model analysis uses 100,000 task-stratified whole-question draws (the flexible-reuse, direct-transfer, and float32 analyses use 20,000), refitting source gains, target amplitude, shared scaling, and target gains in each; source and target share the resampled fitting questions, and evaluation questions are resampled independently. Bonferroni percentile intervals correct across the 42 directions separately for each contrast, and the seven target-structure checks form their own family. Prompt comparisons fit the common source pattern with format amplitudes. Their primary intervals resample test questions with fitted coefficients fixed; a broader analysis with 10,000 draws also resamples source and target training questions and refits both.

\paragraph{Profiles.} Profile outputs cover only the 45 new questions, so we use a fixed, task-stratified five-fold split with 36 fitting and nine evaluation questions per fold, keeping a question's profiles, attributes, and value pairs together. Source gains come from the disjoint 181 short-statement questions. Each of 10,000 draws resamples target questions within task and fold and applies shared positive exponential multipliers to profile identities, normalized within panels, which preserves sparse panels and identities that appear under several attributes; source and target fits are repeated in every draw. Estimates therefore target new questions within the represented panels, not unseen people.

\section{Predictive structure and training stages}
\label{app:crossmodel}
\paragraph{New questions.} Table~\ref{tab:crossmodel-gates} gives both uncertainty scopes for the comparison in Table~\ref{tab:prediction}, together with the share of shared scaling's squared error that the gains remove, $\Delta Q^2/(1-Q^2_{\rm shared})$. Once fitting uncertainty is propagated and seven pairs are corrected, the advantage is significant for OLMo-3 and Qwen2.5; with fitted gains held fixed, it is also significant for Qwen3, Qwen3.5, and T\"ulu. Additive attribute and task terms recover 71 to 93\% of the full improvement in those five pairs (Table~\ref{tab:olmo-ablation}).

\begin{table}[htbp]
\centering\small\setlength{\tabcolsep}{4pt}
\caption{\textbf{Separate gains improve prediction significantly in OLMo-3 and Qwen2.5.} Cell gains minus shared scaling in $Q^2$ points, with the share of shared scaling's squared error removed. Fixed-gain intervals resample evaluation questions only (97.5\% for OLMo-3 and T\"ulu, 99\% otherwise); refit intervals also resample fitting questions and correct across seven pairs. Bold estimates have refit intervals above zero. Qwen2.5 uses the float32 rerun (Appendix~\ref{app:controls}).}
\label{tab:crossmodel-gates}
\label{tab:prediction-error-reference}
\begin{tabular*}{\linewidth}{@{\extracolsep{\fill}}lrrrrrr@{}}
\toprule
& & & \multicolumn{2}{c}{Fixed gains} & \multicolumn{2}{c}{Refit in each draw} \\
\cmidrule(lr){4-5}\cmidrule(l){6-7}
Model & Estimate & Error removed (\%) & Lower & Upper & Lower & Upper \\
\midrule
Llama-3.1 & 0.07 & 0.2 & $-2.16$ & 1.89 & $-3.46$ & 1.97 \\
OLMo-2 & $-0.02$ & $-0.1$ & $-0.92$ & 0.88 & $-1.50$ & 0.96 \\
OLMo-3 & \textbf{3.10} & 8.2 & 2.14 & 4.20 & 1.34 & 4.32 \\
Qwen2.5 & \textbf{1.28} & 5.6 & 0.54 & 2.01 & 0.15 & 2.07 \\
Qwen3 & 0.87 & 3.2 & 0.12 & 1.68 & $-0.39$ & 1.81 \\
Qwen3.5 & 1.13 & 5.2 & 0.15 & 2.12 & $-0.33$ & 2.15 \\
T\"ulu & 2.11 & 7.4 & 0.74 & 3.22 & $-0.09$ & 3.41 \\
\bottomrule
\end{tabular*}
\end{table}

\begin{table}[htbp]
\centering\small
\caption{\textbf{Additive attribute and task terms capture most of the benefit.} Improvements over shared scaling in $Q^2$ points on the 45 new questions, without regularization. Task fits one gain per task; Additive fits $\mu+u_a+v_b$; Full fits one gain per cell and reproduces Table~\ref{tab:prediction}.}
\label{tab:olmo-ablation}
\begin{tabular}{@{}lrrrr@{}}
\toprule
Model & Attribute & Task & Additive & Full \\
\midrule
Llama-3.1 & 0.83 & $-1.10$ & $-0.38$ & 0.07 \\
OLMo-2 & 0.17 & $-0.35$ & $-0.19$ & $-0.02$ \\
OLMo-3 & 1.84 & 1.09 & 2.61 & 3.10 \\
Qwen2.5 & 0.67 & 0.62 & 1.19 & 1.28 \\
Qwen3 & 0.56 & 0.08 & 0.69 & 0.87 \\
Qwen3.5 & 0.32 & 0.51 & 0.81 & 1.13 \\
T\"ulu & 1.65 & 0.10 & 1.82 & 2.11 \\
\bottomrule
\end{tabular}
\end{table}

\subsection{Training stages}
\label{app:scope-support}
\label{app:stages}
We compare predicting no change, $\widehat y=x$, with shared scaling, $\widehat y=gx$, on the same held-out observations. The ratio $R=L_{\rm shared}/L_{\rm identity}$ is the squared error left after scaling as a fraction of the error of predicting no change; it is neither a norm ratio nor a fraction of neural change. Table~\ref{tab:stage-scaling-scope} also reports, for the two main trajectories, each stage's total change $\|y-x\|_w/\|x\|_w$ and the residual after the best pooled scalar, $\kappa\sqrt{1-\rho_0^2}$, both relative to that stage's earlier endpoint.

The main contrast, $Q^2_{\rm shared}$ of SFT$\to$DPO minus that of Base$\to$SFT, is 0.297 (corrected interval 0.280 to 0.315) for OLMo-3 and 0.207 (0.188 to 0.228) for T\"ulu. The DPO step is not small: in OLMo-3 its total change (0.868) matches the SFT step (0.892), and response magnitude grows by a factor of 1.80, yet the residual after scaling halves. DPO-step ratios of 0.27 and 0.22 still leave structure beyond scaling, so the data support neither exact conservation nor a scale-only mechanism. The final transitions are near identity. Base-to-Instruct endpoints give $(Q^2_{\rm shared},R)$ of $(0.637,0.835)$ for Llama-3.1, $(0.770,0.331)$ for Qwen2.5 (float32), $(0.728,0.456)$ for Qwen3, and $(0.784,0.487)$ for Qwen3.5; they bundle stages and cannot locate a change. Figure~\ref{fig:training-detail} shows where separate gains improve on shared scaling within each transition.

\begin{table}[htbp]\centering\small\setlength{\tabcolsep}{5pt}
\caption{\textbf{Shared scaling fits the tested DPO steps better than the SFT steps.} $R$ is the squared prediction error after scaling divided by the error of predicting no change; total change and residual are relative to each stage's earlier endpoint. A high $Q^2_{\rm shared}$ means scaling predicts well, not that training leaves responses unchanged. OLMo-3 and T\"ulu use question cross-validation; OLMo-2 uses the 181/45 split.}
\label{tab:stage-scaling-scope}
\label{tab:stage-residuals}
\begin{tabular*}{\linewidth}{@{\extracolsep{\fill}}llrrrrr@{}}
\toprule
Trajectory & Transition & $Q^2_{\rm identity}$ & $Q^2_{\rm shared}$ & $R$ & Total change & Residual \\
\midrule
OLMo-2 & Base $\to$ SFT & 0.663 & 0.787 & 0.633 & & \\
 & SFT $\to$ DPO & 0.899 & 0.969 & 0.308 & & \\
 & DPO $\to$ Instruct & 0.977 & 0.994 & 0.263 & & \\
\addlinespace[2pt]
OLMo-3 & Base $\to$ SFT & 0.627 & 0.639 & 0.966 & 0.892 & 0.876 \\
 & SFT $\to$ DPO & 0.767 & 0.936 & 0.274 & 0.868 & 0.453 \\
 & DPO $\to$ Final & 0.990 & 0.990 & 0.966 & & \\
\addlinespace[2pt]
T\"ulu & Base $\to$ SFT & 0.615 & 0.755 & 0.636 & 1.255 & 1.000 \\
 & SFT $\to$ DPO & 0.825 & 0.962 & 0.219 & 0.684 & 0.320 \\
 & DPO $\to$ Final & 0.982 & 0.983 & 0.972 & & \\
\bottomrule
\end{tabular*}
\end{table}

\begin{figure}[htbp]
\centering\includegraphics[width=\linewidth]{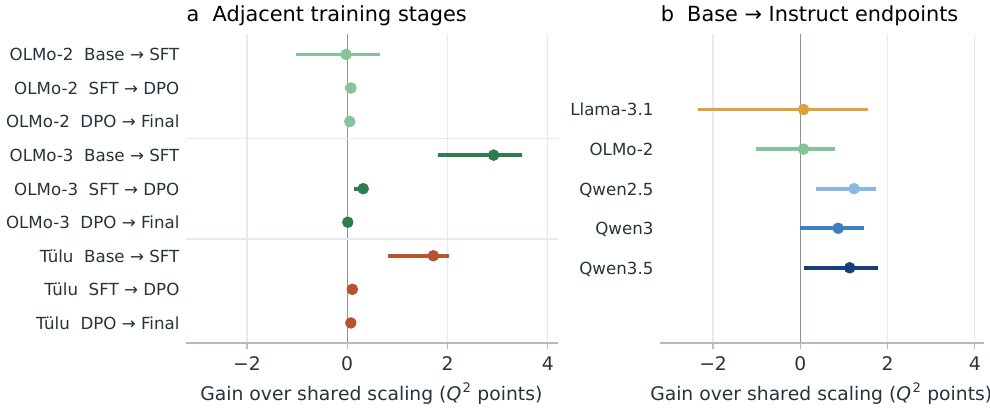}
\caption{\textbf{When do separate gains improve on one shared multiplier?} Positive values favor gains for each attribute and task. (a) Adjacent transitions in OLMo-2, OLMo-3, and T\"ulu; OLMo-3 and T\"ulu use five folds of question cross-validation, OLMo-2 the 181/45 split. (b) Base versus Instruct on the 181/45 split; these endpoints combine stages. Bars are pointwise 95\% paired bootstrap intervals with refitting, uncorrected; colors identify models as in the other figures. Each transition is evaluated separately. Numerical values are in the supplementary material. Qwen2.5 uses bfloat16 here; float32 checks appear in Appendix~\ref{app:controls}.}
\label{fig:training-detail}
\end{figure}

\section{Reuse across models}
\label{app:cross-model-reuse}
\paragraph{Design.} For each of the 42 directions, source gains are fitted on 181 questions. Predictions start from the target's Base effects and are scored against its post-trained effects on 45 new questions. Adapted reuse fits one amplitude and flexible reuse two parameters on the same 181 target questions used by shared scaling and refitting (Equation~\ref{eq:transfer}). T\"ulu and Llama share a backbone and the Qwen releases are related, so these are 42 comparisons, not 42 independent replications.

\paragraph{Results.} Table~\ref{tab:all-crossmodel} gives every direction. One amplitude predicts worse than shared scaling in 40 directions, significantly in 29, and never significantly better. Two parameters remove every significant loss and beat one amplitude in 33 directions, but beat shared scaling significantly in one, with median 0.03 points; the frozen rule for broadly reusable value (at least 22 significant directions) is not met. Their corrected upper bounds lie below one point in 40 directions and below half a point in 26. Refitting still beats flexible reuse by point estimate in 39 directions, significantly in the six into OLMo-3. Table~\ref{tab:two-param} summarizes these counts, including a variant whose source is the additive fit $\mu+u_a+v_b$.
\label{app:two-param-reuse}

\begin{table}[!tp]
\centering\footnotesize
\setlength{\tabcolsep}{2.6pt}
\caption{\textbf{Every cross-model transfer.} Differences in $Q^2$ points on 45 new target questions. One amplitude: adapted reuse minus shared scaling. Two parameters: flexible reuse minus shared scaling. Refit $-$ two: target refitting minus flexible reuse. Bounds are corrected across 42 directions separately for each comparison; bold estimates have intervals that exclude zero. One-amplitude values use 100,000 draws, the others 20,000. Significance and counts use unrounded bounds. Qwen2.5 uses bfloat16 here; float32 checks appear in Appendix~\ref{app:controls}.}
\label{tab:all-crossmodel}
\begin{tabular*}{\linewidth}{@{\extracolsep{\fill}}llrrrrrrrrr@{}}
\toprule
& & \multicolumn{3}{c}{One amplitude} & \multicolumn{3}{c}{Two parameters} & \multicolumn{3}{c}{Refit $-$ two} \\
\cmidrule(lr){3-5}\cmidrule(lr){6-8}\cmidrule(l){9-11}
Source & Target & Est. & Lower & Upper & Est. & Lower & Upper & Est. & Lower & Upper \\
\midrule
Llama-3.1 & OLMo-2 & $\mathbf{-4.28}$ & $-7.93$ & $-2.36$ & $-0.11$ & $-0.65$ & 0.26 & 0.09 & $-1.53$ & 1.03 \\
 & OLMo-3 & $\mathbf{-2.88}$ & $-5.25$ & $-1.40$ & 0.26 & $-0.09$ & 0.69 & \textbf{2.84} & 0.90 & 4.41 \\
 & Qwen2.5 & $\mathbf{-6.09}$ & $-9.19$ & $-4.52$ & \textbf{0.21} & 0.0005 & 0.50 & 1.02 & $-0.27$ & 2.04 \\
 & Qwen3 & $\mathbf{-3.45}$ & $-6.41$ & $-1.64$ & 0.00 & $-0.11$ & 0.18 & 0.87 & $-0.76$ & 2.05 \\
 & Qwen3.5 & $\mathbf{-1.76}$ & $-4.54$ & $-0.18$ & 0.38 & $-0.08$ & 0.95 & 0.75 & $-0.67$ & 1.90 \\
 & T\"ulu & $-0.66$ & $-3.32$ & 0.77 & 0.44 & $-0.49$ & 1.19 & 1.68 & $-0.60$ & 2.93 \\
\addlinespace[2pt]
OLMo-2 & Llama-3.1 & $-0.66$ & $-2.58$ & 0.47 & $-0.54$ & $-2.18$ & 0.51 & 0.61 & $-3.23$ & 2.15 \\
 & OLMo-3 & $-0.03$ & $-1.29$ & 1.02 & 0.19 & $-0.57$ & 0.97 & \textbf{2.91} & 1.09 & 4.34 \\
 & Qwen2.5 & $\mathbf{-2.09}$ & $-3.74$ & $-1.21$ & 0.25 & $-0.09$ & 0.72 & 0.98 & $-0.30$ & 1.88 \\
 & Qwen3 & $-1.08$ & $-2.82$ & 0.39 & 0.00 & $-0.19$ & 0.16 & 0.87 & $-0.79$ & 2.04 \\
 & Qwen3.5 & 0.21 & $-1.08$ & 1.38 & 0.31 & $-0.09$ & 0.95 & 0.82 & $-0.65$ & 2.03 \\
 & T\"ulu & 0.14 & $-1.17$ & 1.10 & 0.25 & $-0.64$ & 1.08 & 1.86 & $-0.39$ & 3.19 \\
\addlinespace[2pt]
OLMo-3 & Llama-3.1 & $\mathbf{-3.26}$ & $-6.91$ & $-0.98$ & $-0.68$ & $-2.07$ & 0.19 & 0.75 & $-3.22$ & 2.38 \\
 & OLMo-2 & $\mathbf{-4.27}$ & $-7.76$ & $-1.99$ & $-0.36$ & $-1.12$ & 0.21 & 0.34 & $-1.08$ & 1.14 \\
 & Qwen2.5 & $\mathbf{-5.39}$ & $-8.72$ & $-3.74$ & 0.37 & $-0.06$ & 0.88 & 0.87 & $-0.46$ & 1.78 \\
 & Qwen3 & $\mathbf{-3.46}$ & $-7.28$ & $-0.80$ & $-0.02$ & $-0.33$ & 0.21 & 0.89 & $-0.83$ & 2.18 \\
 & Qwen3.5 & $\mathbf{-1.88}$ & $-4.44$ & $-0.31$ & 0.18 & $-0.30$ & 0.71 & 0.95 & $-0.49$ & 2.05 \\
 & T\"ulu & $\mathbf{-2.44}$ & $-5.24$ & $-0.92$ & $-0.07$ & $-0.65$ & 0.30 & 2.18 & $-0.39$ & 3.60 \\
\addlinespace[2pt]
Qwen2.5 & Llama-3.1 & $\mathbf{-2.12}$ & $-3.75$ & $-0.53$ & 0.07 & $-0.67$ & 0.46 & 0.00 & $-4.53$ & 1.95 \\
 & OLMo-2 & $\mathbf{-2.68}$ & $-4.73$ & $-1.24$ & 0.09 & $-0.37$ & 0.41 & $-0.11$ & $-1.77$ & 0.93 \\
 & OLMo-3 & $\mathbf{-3.29}$ & $-5.51$ & $-2.01$ & 0.38 & $-0.12$ & 0.98 & \textbf{2.72} & 0.75 & 4.31 \\
 & Qwen3 & $-1.17$ & $-3.64$ & 0.22 & 0.05 & $-0.32$ & 0.39 & 0.82 & $-0.89$ & 2.12 \\
 & Qwen3.5 & $\mathbf{-2.46}$ & $-4.82$ & $-1.28$ & 0.03 & $-0.17$ & 0.32 & 1.10 & $-0.49$ & 2.37 \\
 & T\"ulu & $\mathbf{-1.97}$ & $-3.72$ & $-0.90$ & 0.02 & $-0.29$ & 0.26 & 2.10 & $-0.64$ & 3.62 \\
\addlinespace[2pt]
Qwen3 & Llama-3.1 & $\mathbf{-2.37}$ & $-4.47$ & $-1.20$ & 0.08 & $-0.58$ & 0.41 & $-0.01$ & $-4.51$ & 2.13 \\
 & OLMo-2 & $\mathbf{-1.71}$ & $-3.62$ & $-0.76$ & 0.00 & $-0.16$ & 0.06 & $-0.02$ & $-1.77$ & 1.18 \\
 & OLMo-3 & $\mathbf{-2.38}$ & $-4.68$ & $-0.91$ & $-0.02$ & $-0.31$ & 0.30 & \textbf{3.12} & 0.96 & 4.80 \\
 & Qwen2.5 & $\mathbf{-1.50}$ & $-3.55$ & $-0.51$ & $-0.04$ & $-0.38$ & 0.18 & 1.28 & $-0.01$ & 2.27 \\
 & Qwen3.5 & $\mathbf{-1.39}$ & $-2.77$ & $-0.58$ & 0.00 & $-0.14$ & 0.19 & 1.13 & $-0.55$ & 2.37 \\
 & T\"ulu & $\mathbf{-1.42}$ & $-3.11$ & $-0.46$ & $-0.07$ & $-0.52$ & 0.18 & 2.18 & $-0.49$ & 3.69 \\
\addlinespace[2pt]
Qwen3.5 & Llama-3.1 & $\mathbf{-1.74}$ & $-4.47$ & $-0.13$ & $-0.65$ & $-2.15$ & 0.22 & 0.72 & $-3.05$ & 2.32 \\
 & OLMo-2 & $\mathbf{-2.45}$ & $-5.04$ & $-1.02$ & $-0.25$ & $-0.87$ & 0.08 & 0.23 & $-1.28$ & 1.17 \\
 & OLMo-3 & $-0.42$ & $-2.53$ & 1.03 & 0.25 & $-0.38$ & 0.97 & \textbf{2.85} & 1.09 & 4.26 \\
 & Qwen2.5 & $\mathbf{-2.14}$ & $-4.48$ & $-0.84$ & 0.04 & $-0.23$ & 0.36 & 1.19 & $-0.08$ & 2.14 \\
 & Qwen3 & $-1.10$ & $-3.71$ & 1.03 & 0.02 & $-0.38$ & 0.48 & 0.84 & $-0.60$ & 1.87 \\
 & T\"ulu & $-1.08$ & $-3.09$ & 0.22 & $-0.03$ & $-0.59$ & 0.46 & 2.14 & $-0.46$ & 3.52 \\
\addlinespace[2pt]
T\"ulu & Llama-3.1 & $-1.59$ & $-4.99$ & 0.10 & $-0.41$ & $-2.45$ & 0.55 & 0.48 & $-2.90$ & 2.14 \\
 & OLMo-2 & $\mathbf{-2.81}$ & $-5.41$ & $-1.38$ & $-0.20$ & $-0.83$ & 0.16 & 0.18 & $-1.29$ & 1.10 \\
 & OLMo-3 & $\mathbf{-1.70}$ & $-3.72$ & $-0.13$ & 0.12 & $-0.35$ & 0.67 & \textbf{2.98} & 1.06 & 4.50 \\
 & Qwen2.5 & $\mathbf{-3.91}$ & $-6.46$ & $-2.57$ & 0.09 & $-0.14$ & 0.44 & 1.14 & $-0.17$ & 2.06 \\
 & Qwen3 & $-1.46$ & $-3.90$ & 0.35 & 0.08 & $-0.14$ & 0.47 & 0.79 & $-0.82$ & 1.82 \\
 & Qwen3.5 & $-0.87$ & $-3.10$ & 0.51 & 0.30 & $-0.06$ & 0.76 & 0.83 & $-0.69$ & 2.10 \\
\bottomrule
\end{tabular*}
\end{table}

\begin{table}[htbp]\centering\small\renewcommand{\arraystretch}{1.12}
\caption{\textbf{Separating level and variation removes significant reuse losses but adds little beyond shared scaling.} Counts over 42 directions with corrected intervals; medians in $Q^2$ points; Qwen2.5 in bfloat16 (float32: Appendix~\ref{app:controls}).}
\label{tab:two-param}
\begin{tabular}{@{}lrrrr@{}}
\toprule
Comparison & Above zero & Below zero & Positive point & Median \\
\midrule
One amplitude $-$ shared & 0 & 29 & 2 & $-1.92$ \\
Two parameters $-$ shared & 1 & 0 & 28 & 0.03 \\
Additive source, two parameters $-$ shared & 1 & 0 & 24 & 0.05 \\
Two parameters $-$ one amplitude & 33 & 0 & 42 & 2.10 \\
Refit $-$ two parameters & 6 & 0 & 39 & 0.88 \\
\bottomrule
\end{tabular}
\end{table}

\subsection{Where target structure is established}
\label{app:target-reference}
The bottom row of Figure~\ref{fig:stages} is each target's own refitting advantage, a target-data reference rather than a ceiling. OLMo-3 and Qwen2.5 are the two targets whose advantage survives fitting uncertainty and correction across seven pairs. For them, all 12 incoming transfers fall below shared scaling by point estimate, ten significantly, and refitting significantly beats all 12; with two parameters the incoming patterns gain only 0.19 and 0.15 points on average, against own refits of 3.10 and 1.23, and refitting still beats them significantly in all six directions into OLMo-3 (Table~\ref{tab:target-reference}). Lack of target signal therefore cannot explain these failures. This is a retrospective breakdown of the same data, keeping the original 42-direction correction.

\begin{table}[htbp]\centering\small
\caption{\textbf{Target gains beat borrowed patterns even where target structure is supported.} Means over incoming directions in $Q^2$ points, Qwen2.5 in bfloat16 (its own refit is 1.28 in float32; Table~\ref{tab:prediction}); counts of significant directions under the 42-direction correction. Float32 checks appear in Appendix~\ref{app:controls}.}
\label{tab:target-reference}
\begin{tabular}{@{}lrrrrrrr@{}}
\toprule
Target group & Directions & One amp. & Two par. & Refit & Reuse below & Refit above & Refit above\\
 & & $-$ shared & $-$ shared & $-$ reuse & shared & reuse & two par.\\
\midrule
OLMo-3 & 6 & $-1.78$ & 0.19 & 4.88 & 4/6 & 6/6 & 6/6 \\
Qwen2.5 & 6 & $-3.52$ & 0.15 & 4.75 & 6/6 & 6/6 & 0/6 \\
Other targets & 30 & $-1.91$ & $-0.04$ & 2.74 & 19/30 & 19/30 & 0/30 \\
\bottomrule
\end{tabular}
\end{table}

\subsection{Is the loss an artifact of estimation noise or attribute labels?}
\label{app:splithalf-noise}
\label{app:permutation-null}
\paragraph{Attribute correspondence.} Permuting the 57 rows of each source gain matrix (tasks moved together, 10,000 permutations, own amplitude per permutation) asks whether keeping gains on the correct attributes matters. Real patterns beat the permutation mean in 34 of 42 directions and significantly in seven under correction across directions; the frozen majority rule (22 directions) is not met. Six of the seven significant directions still predict worse than shared scaling, and three of them cross model families. Preserving attribute correspondence can help without making reuse outperform shared scaling. The per-direction values are in the supplementary material.

\paragraph{Estimation noise.} Fitting target cell gains on a random task-stratified half of the fitting questions and their amplitude on the other half, as in adapted reuse (2,000 splits), costs only 0.15 to 0.32 points relative to the full refit, and in the five targets with a positive own advantage the half-sample gains still beat shared scaling by 0.57 to 2.82 points, whereas incoming patterns average $-1.24$ to $-3.52$ (Table~\ref{tab:splithalf-noise}). A second frozen analysis asks how much of the gap a source's own estimation noise could produce. For source $s$, the half-sample difference $(\widehat g^{A}_s-\widehat g^{B}_s)/2$ from a task-stratified split approximates the error of its full estimate. We add it, relative to the source norm, to the normalized target pattern,
\begin{equation}
\tilde g_{s\to t}=\frac{\widehat g_t}{\|\widehat g_t\|}+\frac{\widehat g^{A}_s-\widehat g^{B}_s}{2\,\|\widehat g_s\|},
\end{equation}
refit its amplitude on the target, and score it on the new questions (2,000 splits). The noise cost is the drop from target refitting, not floored at zero; the ratio's denominator is the observed gap, refitting minus adapted reuse. Over the 29 significant losses, the noise cost is 0.15 to 0.63 points against gaps of 1.68 to 7.32, a median of 8\% and at most 17\% of the gap, and every observed reuse value lies below the 2.5\% quantile of its noise-injected distribution (a reference on fixed data, not a corrected test). This is a sensitivity analysis under its injection model, not a decomposition of the transfer loss or an upper bound on the true noise contribution: it assumes the two halves are equally precise and their errors additive and independent, and it cannot capture bias shared by both halves or question-specific noise shared across models. Split-half reliabilities of the cell gains, projected to 181 questions, are 0.80 to 0.92, whereas centered source--target correlations have median 0.24. Among the six Base$\to$SFT directions between OLMo-2, OLMo-3, and T\"ulu, which share a stage label, one amplitude has four significant losses and two parameters no significant gain.
\label{app:matched-stage-reuse}

\begin{table}[htbp]\centering\small\renewcommand{\arraystretch}{1.12}\setlength{\tabcolsep}{4.5pt}
\caption{\textbf{Under the tested noise injection model, estimation noise costs far less than the observed transfer gap.} Qwen2.5 in bfloat16 (own refit 1.28 in float32; Table~\ref{tab:prediction}). $Q^2$ points except reliability. Own refit and Half-sample: target gains from all 181 or half of the questions versus shared scaling. Reuse: mean of six incoming patterns under adapted reuse. Source noise: mean cost of injecting each incoming source's noise into the target gains. Observed gap: mean refitting minus adapted reuse. Reliability: projected split-half correlation of the cell gains. Float32 checks appear in Appendix~\ref{app:controls}.}
\label{tab:splithalf-noise}
\begin{tabular}{@{}lrrrrrrr@{}}
\toprule
& \multicolumn{3}{c}{Target data} & \multicolumn{3}{c}{Incoming patterns} & \\
\cmidrule(lr){2-4}\cmidrule(lr){5-7}
Target & Own refit & Half-sample & Halving cost & Reuse & Source noise & Observed gap & Reliability \\
\midrule
Llama-3.1 & 0.07 & $-0.24$ & 0.32 & $-1.96$ & 0.18 & 2.03 & 0.91 \\
OLMo-2 & $-0.02$ & $-0.17$ & 0.15 & $-3.03$ & 0.31 & 3.01 & 0.86 \\
OLMo-3 & 3.10 & 2.82 & 0.27 & $-1.78$ & 0.25 & 4.88 & 0.90 \\
Qwen2.5 & 1.23 & 0.95 & 0.28 & $-3.52$ & 0.38 & 4.75 & 0.80 \\
Qwen3 & 0.87 & 0.57 & 0.30 & $-1.96$ & 0.31 & 2.82 & 0.87 \\
Qwen3.5 & 1.13 & 0.92 & 0.22 & $-1.36$ & 0.30 & 2.49 & 0.87 \\
T\"ulu & 2.11 & 1.85 & 0.27 & $-1.24$ & 0.22 & 3.35 & 0.92 \\
\bottomrule
\end{tabular}
\end{table}

\subsection{Transferring persona effects directly, without a Base checkpoint}
\label{app:direct-transfer}
Cross-model gain transfer needs each target's Base checkpoint. A practitioner without one may still measure the target directly; here we ask what borrowing another model's measured effects gives, with the same 181 target fitting questions. We therefore predict each target's post-trained contrasts from a source's post-trained contrasts on the same questions, with one target-fitted scale ($\widehat y_t=c\,y_s$) or one scale per attribute and task cell ($\widehat y_t=c_{ab}\,y_s$), fitted on the 181 fitting questions and scored on the 45 new ones (frozen protocol; 20,000 whole-question draws refitting every predictor; Bonferroni across 42 directions). A source's effects, rescaled once, explain a median 0.47 of a target's new-question contrast energy (range 0.38 to 0.70), well below the 0.62 to 0.79 explained by shared scaling of the target's own Base effects; the direct predictor is significantly worse in 40 of 42 directions (median $-25.1$ points). The one exception is T\"ulu into Llama-3.1, two post-trainings of the same base, where the direct predictor explains 0.70 and beats the target's Base-anchored scaling by 6.1 points; the reverse direction does not. Letting the target refit one scale per cell adds significantly in 18 directions (Table~\ref{tab:direct-transfer}). These direct predictors are usually weaker than the Base-anchored baseline, with one directional exception between post-trainings of a shared base; they do not remove the need for target measurements.

\begin{table}[htbp]\centering\footnotesize\setlength{\tabcolsep}{4pt}\renewcommand{\arraystretch}{1.06}
\caption{\textbf{Other models' persona effects predict a target's effects only moderately.} $Q^2$ on 45 new questions. Base shared: the target's own Base effects with one scale (needs the Base). Direct: the six other models' post-trained effects with one target-fitted scale (median and range) or one scale per cell (median). Counts are directions with Bonferroni intervals (42 directions) excluding zero. Qwen2.5 uses bfloat16 here; float32 checks appear in Appendix~\ref{app:controls}.}
\label{tab:direct-transfer}
\begin{tabular*}{\linewidth}{@{\extracolsep{\fill}}lrrrrrr@{}}
\toprule
& Base shared & \multicolumn{2}{c}{Direct, one scale} & Direct, cells & Cells $>$ scale & Direct $<$ Base \\
\cmidrule(lr){3-4}
Target & $Q^2$ & Median & Range & Median & & \\
\midrule
Llama-3.1 & 0.637 & 0.474 & 0.41 to 0.70 & 0.536 & 3/6 & 5/6 \\
OLMo-2 & 0.787 & 0.461 & 0.42 to 0.49 & 0.498 & 3/6 & 6/6 \\
OLMo-3 & 0.624 & 0.411 & 0.38 to 0.45 & 0.455 & 5/6 & 6/6 \\
Qwen2.5 & 0.766 & 0.487 & 0.39 to 0.50 & 0.503 & 1/6 & 6/6 \\
Qwen3 & 0.728 & 0.477 & 0.41 to 0.50 & 0.496 & 2/6 & 6/6 \\
Qwen3.5 & 0.784 & 0.478 & 0.38 to 0.54 & 0.512 & 0/6 & 6/6 \\
T\"ulu & 0.715 & 0.491 & 0.45 to 0.70 & 0.532 & 4/6 & 5/6 \\
\bottomrule
\end{tabular*}
\end{table}

\section{Reuse across prompts}
\label{app:extensions}
The main template-transfer analyses use the 226-question pool, the 57 attributes, and the A/B readout of the main text, with the source pattern fitted jointly on two source templates with a format amplitude each; targets change the wording (including the answering perspective), the examples, or both. Exceptions are stated where they occur: the single-template flexible-reuse check, the deployed-model variant, and the 16-attribute country study.

\paragraph{Wording versus examples.} Table~\ref{tab:format-intervals} gives refitting minus adapted reuse and refitting minus shared scaling for each pair and change. Each group of pairs forms its own correction family: OLMo-3, T\"ulu, and a single-format Qwen2.5 source (six comparisons); Qwen2.5 with the second source format (two, fixed before that inference); Qwen3 and Qwen3.5 (eight); Llama (four, with the Base outputs shared with T\"ulu). Refitting beats adapted reuse after rewording in all six pairs. Changing only the examples leaves the source useful in Qwen3.5 (adapted reuse minus shared scaling 2.08 points, corrected interval 0.94 to 3.14) and T\"ulu (1.48, 0.55 to 2.24). Qwen2.5's wording refit beats adapted reuse without a significant gain over shared scaling; its Instruct readout places little mass on the answer tokens (minimum A+B mass 0.002 on the first-attribute check), so its prompt results are the least certain.

\begin{table}[htbp]
\centering\footnotesize\setlength{\tabcolsep}{3pt}\renewcommand{\arraystretch}{1.1}
\caption{\textbf{Rewording favors refitting in all six pairs; changing examples still permits reuse.} Values in $Q^2$ points. Bounds are corrected within each study, as described in the text; bold estimates have intervals that exclude zero. The last two rows learn the Qwen2.5 source from one format; the main rows use two.}
\label{tab:format-intervals}
\label{tab:sequential-formats}
\begin{tabular*}{\linewidth}{@{\extracolsep{\fill}}llrrrrrr@{}}
\toprule
& & \multicolumn{3}{c}{Refit $-$ adapted reuse} & \multicolumn{3}{c}{Refit $-$ shared} \\
\cmidrule(lr){3-5}\cmidrule(l){6-8}
Model & Change & Est. & Lower & Upper & Est. & Lower & Upper \\
\midrule
Llama-3.1 & Wording & \textbf{3.63} & 1.89 & 5.72 & \textbf{1.47} & 0.39 & 2.18 \\
 & Examples & 0.84 & $-0.33$ & 1.76 & 1.62 & $-0.59$ & 3.04 \\
\addlinespace[2pt]
OLMo-3 & Wording & \textbf{3.91} & 2.40 & 5.45 & \textbf{2.38} & 1.26 & 3.00 \\
 & Examples & \textbf{1.38} & 0.22 & 2.21 & \textbf{2.80} & 0.06 & 4.94 \\
\addlinespace[2pt]
Qwen2.5 & Wording & \textbf{1.70} & 0.08 & 3.13 & 1.29 & $-0.33$ & 2.32 \\
 & Examples & 0.70 & $-0.39$ & 1.57 & \textbf{1.66} & 0.40 & 2.28 \\
\addlinespace[2pt]
Qwen3 & Wording & \textbf{0.81} & 0.10 & 1.39 & \textbf{1.73} & 0.54 & 2.64 \\
 & Examples & 0.55 & $-0.15$ & 1.20 & \textbf{1.17} & 0.28 & 1.68 \\
\addlinespace[2pt]
Qwen3.5 & Wording & \textbf{2.57} & 0.65 & 4.76 & 1.15 & $-0.18$ & 2.15 \\
 & Examples & 0.94 & $-0.16$ & 1.87 & \textbf{3.02} & 1.04 & 4.50 \\
\addlinespace[2pt]
T\"ulu & Wording & \textbf{3.32} & 1.85 & 5.05 & \textbf{2.12} & 1.19 & 2.75 \\
 & Examples & 0.87 & $-0.04$ & 1.69 & \textbf{2.35} & 1.07 & 3.39 \\
\midrule
\multicolumn{8}{@{}l}{\textit{Single-format source (secondary)}} \\
Qwen2.5 & Wording & \textbf{2.33} & 0.29 & 4.32 & 1.29 & $-0.65$ & 2.54 \\
 & Examples & 0.85 & $-0.45$ & 2.08 & \textbf{1.66} & 0.13 & 2.43 \\
\bottomrule
\end{tabular*}
\end{table}

\paragraph{Joint wording and example change.} Table~\ref{tab:alltarget} completes the three comparisons for the joint-change target. Refitting beats adapted reuse in all six pairs by point estimate and significantly in five; it also beats shared scaling significantly in OLMo-3, T\"ulu, Llama-3.1, and Qwen3. In Qwen3.5, refitting beats adapted reuse while structure beyond shared scaling remains unresolved: removing a harmful source pattern can help without establishing useful target gains. With fitted coefficients held fixed, adapted reuse lies below shared scaling for OLMo-3 and T\"ulu ($-1.34$ and $-1.39$ points, both significant); propagating fitting uncertainty widens the OLMo-3 interval to include zero.

\begin{table}[htbp]\centering\small\setlength{\tabcolsep}{4pt}\renewcommand{\arraystretch}{1.06}
\caption{\textbf{Joint wording and example change: all three comparisons.} Paired differences in $Q^2$ points on identical target questions, with intervals including fitting uncertainty. Corrections: 24 comparisons for OLMo-3 and T\"ulu (shared with their profile comparisons), four per contrast for Llama-3.1 and Qwen2.5, four settings for Qwen3 and Qwen3.5. Bold estimates have intervals that exclude zero; n.r.\ means not reported in that study.}
\label{tab:alltarget}
\label{tab:qwen-template-scope}
\begin{tabular*}{\linewidth}{@{\extracolsep{\fill}}lrrrrrrrrr@{}}
\toprule
& \multicolumn{3}{c}{Refit $-$ shared} & \multicolumn{3}{c}{Adapted $-$ shared} & \multicolumn{3}{c}{Refit $-$ adapted} \\
\cmidrule(lr){2-4}\cmidrule(lr){5-7}\cmidrule(l){8-10}
Pair & Est. & Lower & Upper & Est. & Lower & Upper & Est. & Lower & Upper \\
\midrule
Llama-3.1 & \textbf{1.83} & 0.84 & 2.49 & $\mathbf{-1.99}$ & $-4.95$ & $-0.19$ & \textbf{3.82} & 2.01 & 6.40 \\
OLMo-3 & \textbf{2.67} & 1.08 & 3.64 & $-1.34$ & $-3.89$ & 0.19 & \textbf{4.01} & 2.18 & 6.23 \\
Qwen2.5 & 1.34 & $-0.37$ & 2.17 & $-0.17$ & $-1.62$ & 0.68 & 1.51 & $-0.30$ & 2.99 \\
Qwen3 & \textbf{1.43} & 0.55 & 2.06 & n.r. & & & \textbf{0.97} & 0.32 & 1.63 \\
Qwen3.5 & 0.49 & $-0.60$ & 1.17 & $\mathbf{-1.50}$ & $-3.14$ & $-0.48$ & \textbf{1.98} & 0.92 & 3.20 \\
T\"ulu & \textbf{1.92} & 0.87 & 2.89 & $\mathbf{-1.39}$ & $-3.52$ & $-0.17$ & \textbf{3.30} & 1.60 & 5.70 \\
\bottomrule
\end{tabular*}
\end{table}

\subsection{Flexible reuse under prompt changes}
\label{app:flex-prompt}
To test whether the refitting advantage survives the flexible baseline, we fit source cell gains on the original template of each pair and apply them to the wording-only, example-only, and joint targets with two target parameters (frozen protocol; 20,000 draws refitting source and target in each; correction across 14 cells). This uses one source template for every pair, whereas the main prompt analysis uses two. After rewording, flexible reuse stays within 0.6 points of shared scaling in all six pairs, and refitting beats it by point estimate in all six and significantly in OLMo-3, T\"ulu, and Llama-3.1; after the joint change it does so significantly in both pairs (Table~\ref{tab:flex-prompt}). After example-only changes, flexible reuse improves on shared scaling significantly in T\"ulu and Qwen3.5.

\begin{table}[htbp]\centering\footnotesize\setlength{\tabcolsep}{4pt}\renewcommand{\arraystretch}{1.06}
\caption{\textbf{Rewording still favors refitting over flexible reuse.} Differences in $Q^2$ points with intervals corrected across 14 cells; bold estimates have intervals that exclude zero.}
\label{tab:flex-prompt}
\begin{tabular*}{\linewidth}{@{\extracolsep{\fill}}llrrrrrr@{}}
\toprule
& & \multicolumn{3}{c}{Flexible reuse $-$ shared} & \multicolumn{3}{c}{Refit $-$ flexible reuse} \\
\cmidrule(lr){3-5}\cmidrule(l){6-8}
Model & Change & Est. & Lower & Upper & Est. & Lower & Upper \\
\midrule
Llama-3.1 & Wording & $-0.03$ & $-0.91$ & 0.51 & \textbf{1.50} & 0.50 & 2.31 \\
 & Examples & 0.85 & $-0.86$ & 1.95 & 0.77 & $-0.66$ & 1.81 \\
\addlinespace[2pt]
OLMo-3 & Wording & 0.03 & $-0.66$ & 0.51 & \textbf{2.35} & 1.21 & 3.13 \\
 & Examples & 1.76 & $-0.49$ & 3.83 & 1.04 & $-0.18$ & 2.00 \\
 & Both & 0.01 & $-0.87$ & 0.57 & \textbf{2.66} & 1.43 & 3.64 \\
\addlinespace[2pt]
Qwen2.5 & Wording & 0.11 & $-0.43$ & 0.53 & 1.18 & $-0.99$ & 2.67 \\
 & Examples & 0.95 & $-0.11$ & 1.62 & 0.72 & $-0.76$ & 1.80 \\
\addlinespace[2pt]
Qwen3 & Wording & 0.58 & $-0.12$ & 1.32 & 1.15 & $-0.15$ & 2.09 \\
 & Examples & 0.46 & $-0.38$ & 0.98 & 0.72 & $-0.04$ & 1.34 \\
\addlinespace[2pt]
Qwen3.5 & Wording & $-0.09$ & $-0.70$ & 0.16 & 1.23 & $-0.06$ & 2.35 \\
 & Examples & \textbf{1.79} & 0.41 & 3.02 & \textbf{1.23} & 0.22 & 2.48 \\
\addlinespace[2pt]
T\"ulu & Wording & 0.04 & $-0.31$ & 0.34 & \textbf{2.08} & 0.95 & 2.91 \\
 & Examples & \textbf{1.73} & 0.79 & 2.59 & 0.63 & $-0.35$ & 1.38 \\
 & Both & $-0.03$ & $-0.36$ & 0.22 & \textbf{1.95} & 0.92 & 2.99 \\
\bottomrule
\end{tabular*}
\end{table}

\paragraph{With the two-template source.} The analysis above uses one source template, whereas the main prompt analysis fits the source jointly on two. Repeating it with that two-template common pattern held fixed (frozen addendum; 20,000 draws refitting the target predictors only, so these intervals condition on the fitted source and are narrower in kind than those above; correction across 14 cells; full-data anchors reproduce Table~\ref{tab:format-intervals}) gives Table~\ref{tab:flex-prompt-aligned}. After rewording, all six point estimates favor refitting over flexible reuse, significantly in the same three pairs (OLMo-3, T\"ulu, Llama-3.1) and after the joint change in both; but the prespecified rule was not fully met, because in Qwen3 the two-template source retains significant value after rewording (1.00 points over shared scaling), and after example-only changes flexible reuse now beats shared scaling in four pairs rather than two. The richer source carries more reusable information when only the examples change; after rewording, refitting is favored in every pair and significantly in three.

\begin{table}[htbp]\centering\footnotesize\setlength{\tabcolsep}{2.8pt}\renewcommand{\arraystretch}{1.06}
\caption{\textbf{Prompt changes with the two-template source.} Differences in $Q^2$ points with intervals corrected across 14 cells; bold estimates have intervals that exclude zero. Intervals condition on the fitted source pattern and resample only target questions, unlike Table~\ref{tab:flex-prompt}.}
\label{tab:flex-prompt-aligned}
\begin{tabular*}{\linewidth}{@{\extracolsep{\fill}}llrrrrrrrrr@{}}
\toprule
& & \multicolumn{3}{c}{One amplitude $-$ shared} & \multicolumn{3}{c}{Two parameters $-$ shared} & \multicolumn{3}{c}{Refit $-$ two parameters} \\
\cmidrule(lr){3-5}\cmidrule(lr){6-8}\cmidrule(l){9-11}
Model & Change & Est. & Lower & Upper & Est. & Lower & Upper & Est. & Lower & Upper \\
\midrule
Llama-3.1 & Wording & $\mathbf{-2.16}$ & $-4.51$ & $-0.25$ & $-0.06$ & $-1.03$ & 0.61 & \textbf{1.53} & 0.36 & 2.44 \\
 & Examples & 0.78 & $-0.71$ & 2.02 & 1.06 & $-0.18$ & 2.07 & 0.56 & $-1.37$ & 1.64 \\
OLMo-3 & Wording & $\mathbf{-1.53}$ & $-2.70$ & $-0.23$ & 0.14 & $-0.43$ & 0.66 & \textbf{2.24} & 0.99 & 3.01 \\
 & Examples & 1.42 & $-0.59$ & 3.81 & 1.55 & $-0.23$ & 3.56 & 1.25 & $-0.50$ & 2.20 \\
 & Both & $-1.34$ & $-2.52$ & 0.03 & 0.17 & $-0.52$ & 0.77 & \textbf{2.50} & 1.13 & 3.56 \\
Qwen2.5 & Wording & $-0.41$ & $-1.35$ & 0.53 & 0.32 & $-0.20$ & 0.80 & 0.97 & $-1.29$ & 2.30 \\
 & Examples & \textbf{0.96} & 0.14 & 1.67 & \textbf{1.04} & 0.32 & 1.66 & 0.62 & $-0.99$ & 1.60 \\
Qwen3 & Wording & \textbf{0.93} & 0.09 & 1.87 & \textbf{1.00} & 0.28 & 1.83 & 0.73 & $-0.32$ & 1.20 \\
 & Examples & \textbf{0.62} & 0.22 & 1.04 & \textbf{0.66} & 0.32 & 1.01 & 0.52 & $-0.47$ & 1.12 \\
Qwen3.5 & Wording & $\mathbf{-1.42}$ & $-2.81$ & $-0.01$ & $-0.05$ & $-0.58$ & 0.22 & 1.20 & $-0.13$ & 2.32 \\
 & Examples & \textbf{2.08} & 1.08 & 3.06 & \textbf{2.18} & 0.76 & 3.51 & 0.84 & $-0.12$ & 1.38 \\
T\"ulu & Wording & $\mathbf{-1.20}$ & $-2.22$ & $-0.28$ & 0.09 & $-0.21$ & 0.31 & \textbf{2.03} & 0.90 & 2.92 \\
 & Examples & \textbf{1.48} & 0.74 & 2.19 & \textbf{1.51} & 0.87 & 2.13 & 0.85 & $-0.35$ & 1.59 \\
 & Both & $\mathbf{-1.39}$ & $-2.63$ & $-0.41$ & 0.03 & $-0.37$ & 0.30 & \textbf{1.88} & 0.80 & 3.02 \\
\bottomrule
\end{tabular*}
\end{table}

\subsection{A deployed model without a Base reference}
\label{app:base-free-reuse}
Practitioners often have one deployed model and no Base checkpoint. For six post-trained models, we take the persona contrasts under the original template as $x$ and the contrasts under a changed template as $y$. Reusing the measured effects unchanged ($\widehat y=x$) is compared with shared scaling and with refitted attribute and task gains, all fitted on 181 questions and evaluated on 45 new ones (frozen protocol; 20,000 draws; correction across 14 cells). Refitting beats unchanged reuse after rewording in all six models and beats shared scaling in four, with positive point estimates in all six; after the joint change it beats shared scaling in both models (Table~\ref{tab:base-free}). Changing only the examples matters less. For a single deployed model, new wording therefore calls for measuring attribute relationships again, not merely rescaling the old ones.

\begin{table}[htbp]\centering\footnotesize\setlength{\tabcolsep}{3pt}\renewcommand{\arraystretch}{1.1}
\caption{\textbf{Within one deployed model, rewording calls for relearning attribute relationships.} $Q^2$ on new questions for unchanged reuse, shared scaling, and refitted gains, and paired differences in points with intervals corrected across 14 cells; bold estimates have intervals above zero.}
\label{tab:base-free}
\begin{tabular*}{\linewidth}{@{\extracolsep{\fill}}llrrrrrrrrr@{}}
\toprule
& & \multicolumn{3}{c}{$Q^2$} & \multicolumn{3}{c}{Shared $-$ unchanged} & \multicolumn{3}{c}{Refit $-$ shared} \\
\cmidrule(lr){3-5}\cmidrule(lr){6-8}\cmidrule(l){9-11}
Model & Change & Unch. & Shared & Refit & Est. & Lower & Upper & Est. & Lower & Upper \\
\midrule
Llama-3.1 & Wording & 0.685 & 0.715 & 0.749 & \textbf{3.04} & 1.57 & 4.48 & \textbf{3.41} & 2.00 & 4.55 \\
 & Examples & 0.750 & 0.771 & 0.787 & \textbf{2.10} & 0.88 & 3.29 & \textbf{1.52} & 0.46 & 2.41 \\
\addlinespace[2pt]
OLMo-3 & Wording & 0.703 & 0.703 & 0.750 & 0.06 & $-0.26$ & 0.36 & \textbf{4.63} & 3.07 & 5.94 \\
 & Examples & 0.702 & 0.782 & 0.803 & \textbf{7.95} & 2.99 & 15.04 & 2.14 & $-0.08$ & 3.84 \\
 & Both & 0.615 & 0.615 & 0.667 & $-0.02$ & $-0.57$ & 0.33 & \textbf{5.22} & 3.06 & 6.95 \\
\addlinespace[2pt]
Qwen2.5 & Wording & 0.730 & 0.792 & 0.814 & \textbf{6.23} & 2.92 & 10.41 & \textbf{2.16} & 0.81 & 3.02 \\
 & Examples & 0.826 & 0.832 & 0.842 & 0.64 & $-0.53$ & 1.88 & \textbf{0.97} & 0.17 & 1.55 \\
\addlinespace[2pt]
Qwen3 & Wording & 0.279 & 0.810 & 0.826 & \textbf{53.13} & 36.14 & 73.04 & 1.54 & $-0.08$ & 2.47 \\
 & Examples & 0.796 & 0.809 & 0.828 & 1.28 & $-0.33$ & 3.08 & \textbf{1.89} & 0.14 & 3.54 \\
\addlinespace[2pt]
Qwen3.5 & Wording & 0.751 & 0.762 & 0.779 & 1.11 & $-0.73$ & 3.71 & 1.64 & $-0.08$ & 3.16 \\
 & Examples & 0.774 & 0.778 & 0.792 & 0.45 & $-0.96$ & 1.99 & 1.35 & $-0.09$ & 2.31 \\
\addlinespace[2pt]
T\"ulu & Wording & 0.806 & 0.816 & 0.839 & 1.03 & $-0.47$ & 2.28 & \textbf{2.27} & 1.20 & 3.09 \\
 & Examples & 0.800 & 0.806 & 0.813 & 0.65 & $-0.63$ & 2.03 & 0.71 & $-0.09$ & 1.05 \\
 & Both & 0.728 & 0.732 & 0.756 & 0.37 & $-0.45$ & 1.05 & \textbf{2.41} & 1.10 & 3.37 \\
\bottomrule
\end{tabular*}
\end{table}

\subsection{Similar gains, different predictions}
\label{app:similarity}
Table~\ref{tab:centered-scope} compares each source template's gains with the joint-change target gains, using energy weights $E_c=(E_{s,c}+E_{t,c})/2$; the centered correlation subtracts each vector's weighted mean. Uncentered correlations of 0.97 to 0.99 fall to 0.15 to 0.63 once the shared mean is removed. The 95\% intervals come from 2,000 paired bootstrap refits of the training questions and are descriptive. The evidence that fixed-pattern reuse fails comes from the held-out comparisons; correlation neither establishes that failure nor replaces evaluation on the target.

\begin{table}[htbp]\centering\small
\caption{\textbf{Removing the mean reveals differences between apparently similar gains.} Both correlations use the same energy weights. Source 1 and 2 are the two source templates. These correlations do not define a cutoff for successful reuse.}
\label{tab:centered-scope}
\begin{tabular}{@{}llrrrr@{}}
\toprule
& & & & \multicolumn{2}{c}{Centered 95\% interval} \\
\cmidrule(l){5-6}
Pair & Source & Uncentered & Centered & Lower & Upper \\
\midrule
Llama-3.1 & Source 1 & 0.975 & 0.345 & 0.254 & 0.427 \\
 & Source 2 & 0.977 & 0.553 & 0.481 & 0.618 \\
OLMo-3 & Source 1 & 0.975 & 0.458 & 0.369 & 0.536 \\
 & Source 2 & 0.972 & 0.458 & 0.357 & 0.545 \\
Qwen2.5 & Source 1 & 0.985 & 0.406 & 0.286 & 0.516 \\
 & Source 2 & 0.987 & 0.606 & 0.523 & 0.668 \\
Qwen3 & Source 1 & 0.991 & 0.502 & 0.406 & 0.575 \\
 & Source 2 & 0.989 & 0.629 & 0.530 & 0.695 \\
Qwen3.5 & Source 1 & 0.987 & 0.198 & 0.085 & 0.327 \\
 & Source 2 & 0.989 & 0.269 & 0.176 & 0.395 \\
T\"ulu & Source 1 & 0.977 & 0.151 & 0.078 & 0.267 \\
 & Source 2 & 0.981 & 0.378 & 0.280 & 0.488 \\
\bottomrule
\end{tabular}
\end{table}

\subsection{Where reuse holds, and simpler updates that do not help}
\paragraph{Country contexts.} A separate inventory of 16 psychological, value, and emotion attributes is measured under five English-language country contexts. Adapted reuse keeps a source-country pattern and fits its amplitude on target training questions. It beats target shared scaling in all five tested models (Table~\ref{tab:country-reuse-scope}); only OLMo-3's lower bound clears the study's one-point margin. The contexts are synthetic, not a five-country human validation.

\begin{table}[htbp]\centering\small
\caption{\textbf{Source gains remain useful across five country contexts.} Improvements over target shared scaling in $Q^2$ points. Intervals for adapted reuse condition on fitted gains and are corrected within each study (98.3 to 99.2\%).}
\label{tab:country-reuse-scope}
\begin{tabular}{@{}lrrrr@{}}
\toprule
Pair & Refit & Adapted reuse & Lower & Upper \\
\midrule
Mistral & 1.33 & 1.27 & 0.78 & 1.79 \\
OLMo-3 & 2.65 & 2.55 & 1.14 & 4.06 \\
Qwen3 & 1.41 & 1.34 & 0.85 & 1.88 \\
Qwen3.5 & 0.83 & 0.71 & 0.07 & 1.40 \\
T\"ulu & 1.88 & 1.66 & 0.58 & 2.48 \\
\bottomrule
\end{tabular}
\end{table}

\paragraph{Answering role and unmeasured attributes.} Two preregistered checks ask whether a simpler update closes the transfer gap (Table~\ref{tab:simpler-updates}). Matching the answering role (answering as the persona versus predicting the persona's answer, six templates held out in turn) gives no advantage. Measuring some attributes to update the others, through one amplitude per category or two leading principal components of the source templates' gain patterns, does not beat a single amplitude on the held-out attributes; a secondary four-component variant helps T\"ulu. Full refitting, which also uses target data for the evaluated attributes, gains more.
\label{tab:role-match}
\label{tab:heldout-attributes}

\begin{table}[htbp]\centering\small\renewcommand{\arraystretch}{1.05}
\caption{\textbf{Simpler updates do not close the gap.} Differences in $Q^2$ points, means over six held-out templates, with intervals corrected across two families; bold estimates have intervals that exclude zero. The full-refit rows are references with more information.}
\label{tab:simpler-updates}
\begin{tabular}{@{}llrrr@{}}
\toprule
Model & Contrast & Estimate & Lower & Upper \\
\midrule
OLMo-3 & Same role $-$ different role & $\mathbf{-0.39}$ & $-0.69$ & $-0.11$ \\
 & Category $-$ amplitude & $\mathbf{-0.18}$ & $-0.36$ & $-0.07$ \\
 & Two components $-$ amplitude & 0.08 & $-0.09$ & 0.21 \\
 & Four components $-$ amplitude & 0.08 & $-0.12$ & 0.23 \\
 & Amplitude $-$ no adaptation & $\mathbf{1.07}$ & $0.48$ & $1.70$ \\
 & Full refit $-$ amplitude & $\mathbf{1.84}$ & $1.37$ & $2.10$ \\
\addlinespace[2pt]
T\"ulu & Same role $-$ different role & $-0.02$ & $-0.21$ & 0.14 \\
 & Category $-$ amplitude & $-0.14$ & $-0.37$ & 0.0022 \\
 & Two components $-$ amplitude & 0.19 & $-0.12$ & 0.38 \\
 & Four components $-$ amplitude & $\mathbf{0.52}$ & $0.22$ & $0.78$ \\
 & Amplitude $-$ no adaptation & $\mathbf{2.94}$ & $1.98$ & $4.37$ \\
 & Full refit $-$ amplitude & $\mathbf{1.61}$ & $1.25$ & $1.90$ \\
\bottomrule
\end{tabular}
\end{table}

\section{Updating under a measurement budget}
\label{app:shrinkage}
\paragraph{Regularized updates.} A fixed source pattern with one amplitude is a restrictive baseline, so we also fit target gains with a penalty toward an amplitude-adjusted source pattern,
\begin{equation}
\min_{g,\beta}\;\sum_{q,c}w_{qc}(y_{qc}-g_c x_{qc})^2+\lambda\sum_c(g_c-\beta g^{\rm src}_c)^2,
\end{equation}
where $c$ indexes attribute and task cells and $\beta$ is unpenalized. At $\lambda=0$ this is refitting; as $\lambda\to\infty$ it becomes adapted reuse. A matched control replaces the source pattern by a constant, $g^{\rm src}_c=1$, which shrinks cell gains toward one common gain and separates the value of the source pattern from the value of regularization. The penalty grid is $\nu\in\{0,0.001,0.01,0.1,1,10,100,1000,\infty\}$ times the mean target cell energy, selected by five task-stratified inner folds on target training questions only, with ties resolved toward stronger shrinkage. The 10,000-draw bootstrap repeats source fitting, target fitting, and penalty selection, and each contrast is corrected across the twelve settings (six pairs, two target templates).

The source-shrunk update beats adapted reuse in ten of twelve settings, but never differs significantly from the shared-shrunk update, and full refitting never differs significantly from it (Table~\ref{tab:shrinkage}). The source pattern shows no measurable advantage over regularization toward a common gain.

\begin{table}[htbp]\centering\scriptsize\setlength{\tabcolsep}{2.2pt}\renewcommand{\arraystretch}{1.08}
\caption{\textbf{Shrinking toward the source gains gives no detected advantage over shrinking toward a shared gain.} Differences in $Q^2$ points with intervals corrected across twelve settings for each comparison; bold estimates have intervals that exclude zero. Joint: wording and example change; Frozen: the questionnaire template of \S\ref{sec:remeasure}.}
\label{tab:shrinkage}
\begin{tabular*}{\linewidth}{@{\extracolsep{\fill}}llrrrrrrrrrrrr@{}}
\toprule
& & \multicolumn{3}{c}{Source $-$ shared shrunk} & \multicolumn{3}{c}{Source shrunk $-$ adapted} & \multicolumn{3}{c}{Refit $-$ source shrunk} & \multicolumn{3}{c}{Source shrunk $-$ shared} \\
\cmidrule(lr){3-5}\cmidrule(lr){6-8}\cmidrule(lr){9-11}\cmidrule(l){12-14}
Pair & Template & Est. & Low & Up & Est. & Low & Up & Est. & Low & Up & Est. & Low & Up \\
\midrule
Llama-3.1 & Joint & 0.01 & $-0.11$ & 0.24 & \textbf{3.82} & 1.83 & 6.92 & 0.00 & $-0.02$ & 0.05 & \textbf{1.83} & 0.66 & 2.60 \\
 & Frozen & 0.00 & $-0.12$ & 0.23 & \textbf{4.27} & 2.30 & 6.56 & $-0.00$ & $-0.08$ & 0.02 & \textbf{2.40} & 1.18 & 3.34 \\
OLMo-3 & Joint & $-0.01$ & $-0.27$ & 0.17 & \textbf{3.99} & 2.21 & 5.83 & 0.02 & $-0.01$ & 0.30 & \textbf{2.65} & 1.11 & 3.48 \\
 & Frozen & 0.01 & $-0.13$ & 0.20 & \textbf{2.25} & 1.15 & 3.31 & 0.00 & $-0.10$ & 0.13 & \textbf{3.34} & 1.99 & 4.32 \\
Qwen2.5 & Joint & $-0.04$ & $-0.24$ & 0.29 & 1.52 & $-0.21$ & 3.09 & $-0.01$ & $-0.28$ & 0.14 & 1.34 & $-0.34$ & 2.28 \\
 & Frozen & $-0.02$ & $-0.18$ & 0.08 & \textbf{1.28} & 0.01 & 2.44 & 0.03 & $-0.22$ & 0.13 & \textbf{1.75} & 0.49 & 2.47 \\
Qwen3 & Joint & 0.07 & $-0.17$ & 0.50 & \textbf{0.97} & 0.29 & 1.63 & 0.00 & $-0.10$ & 0.15 & \textbf{1.43} & 0.43 & 2.12 \\
 & Frozen & 0.06 & $-0.15$ & 0.17 & \textbf{1.58} & 0.36 & 3.01 & 0.00 & $-0.07$ & 0.17 & \textbf{1.25} & 0.03 & 2.12 \\
Qwen3.5 & Joint & $-0.12$ & $-0.53$ & 0.20 & \textbf{1.89} & 0.81 & 3.17 & 0.10 & $-0.03$ & 0.23 & 0.39 & $-0.87$ & 1.20 \\
 & Frozen & 0.03 & $-0.91$ & 0.67 & 1.54 & $-0.96$ & 3.49 & $-0.12$ & $-1.45$ & 0.72 & 1.16 & $-1.00$ & 2.63 \\
T\"ulu & Joint & 0.01 & $-0.03$ & 0.19 & \textbf{3.31} & 1.70 & 5.52 & $-0.00$ & $-0.03$ & 0.03 & \textbf{1.92} & 1.00 & 2.76 \\
 & Frozen & 0.01 & $-0.09$ & 0.16 & \textbf{2.63} & 1.36 & 4.26 & $-0.01$ & $-0.13$ & 0.13 & \textbf{2.00} & 1.09 & 2.68 \\
\bottomrule
\end{tabular*}
\end{table}

\paragraph{Budget curves.} The same fits were evaluated on 200 nested, task-stratified subsets of the target fitting questions at 8, 16, 32, and 64 questions per attribute (Figure~\ref{fig:targetbudget}; Table~\ref{tab:shrink-budget}). Unregularized refitting loses to adapted reuse at eight questions in every setting. The source-shrunk update has a positive median advantage over adapted reuse at every budget in all twelve settings and trails refitting by at most 0.4 points from 32 questions on. The shared-shrunk update performs similarly and is often better at small budgets. Subset spreads describe the tested settings, not confidence intervals. For scale, sixteen questions per attribute amount to about $287\times16\times2\approx9{,}200$ forward passes per checkpoint pair.
\label{app:budget-displays}

\begin{table}[htbp]\centering\footnotesize\setlength{\tabcolsep}{3.5pt}\renewcommand{\arraystretch}{1.06}
\caption{\textbf{A regularized update beats adapted reuse at every budget.} Median $Q^2$ differences from adapted reuse, in points, over 200 subsets of the target fitting questions: unregularized refitting (Refit) and the update shrunk toward the source gains (Shrunk).}
\label{tab:shrink-budget}
\begin{tabular*}{\linewidth}{@{\extracolsep{\fill}}llrrrrrrrr@{}}
\toprule
& & \multicolumn{2}{c}{$k=8$} & \multicolumn{2}{c}{$k=16$} & \multicolumn{2}{c}{$k=32$} & \multicolumn{2}{c}{$k=64$} \\
\cmidrule(lr){3-4}\cmidrule(lr){5-6}\cmidrule(lr){7-8}\cmidrule(l){9-10}
Pair & Template & Refit & Shrunk & Refit & Shrunk & Refit & Shrunk & Refit & Shrunk \\
\midrule
Llama-3.1 & Joint change & $-1.36$ & 0.85 & 1.76 & 2.02 & 3.02 & 2.60 & 3.52 & 3.40 \\
 & Frozen & $-2.31$ & 0.76 & 1.66 & 2.31 & 3.26 & 3.35 & 3.87 & 3.90 \\
\addlinespace[2pt]
OLMo-3 & Joint change & $-2.65$ & 0.57 & 1.34 & 1.82 & 2.79 & 2.30 & 3.59 & 3.49 \\
 & Frozen & $-4.83$ & 0.07 & $-0.45$ & 0.46 & 1.05 & 1.30 & 1.79 & 1.73 \\
\addlinespace[2pt]
Qwen2.5 & Joint change & $-5.98$ & 0.05 & $-1.52$ & 0.36 & 0.21 & 0.88 & 1.05 & 1.11 \\
 & Frozen & $-6.18$ & 0.05 & $-1.56$ & 0.19 & 0.07 & 0.61 & 0.88 & 0.85 \\
\addlinespace[2pt]
Qwen3 & Joint change & $-4.32$ & 0.02 & $-1.14$ & 0.14 & 0.18 & 0.37 & 0.69 & 0.57 \\
 & Frozen & $-3.21$ & 0.13 & $-0.37$ & 0.49 & 0.78 & 0.81 & 1.28 & 1.30 \\
\addlinespace[2pt]
Qwen3.5 & Joint change & $-2.51$ & 0.11 & 0.01 & 0.39 & 1.16 & 1.03 & 1.66 & 1.63 \\
 & Frozen & $-12.47$ & 0.25 & $-4.23$ & 0.38 & $-0.90$ & 0.87 & 0.55 & 1.27 \\
\addlinespace[2pt]
T\"ulu & Joint change & $-0.48$ & 0.86 & 1.73 & 2.05 & 2.71 & 2.62 & 3.09 & 3.05 \\
 & Frozen & $-2.17$ & 0.62 & 0.79 & 1.59 & 1.84 & 1.88 & 2.33 & 2.40 \\
\bottomrule
\end{tabular*}
\end{table}

\paragraph{Flexible reuse at each budget.} The two-parameter predictor was scored on the same 200 nested subsets (frozen protocol; the regenerated shared-scaling and adapted-reuse scores match the stored draws exactly). Flexible reuse stays within 0.75 points of shared scaling at every budget in 11 of 12 settings (OLMo-3 on the frozen template, at 1.3 points, is the exception) and beats single-amplitude reuse in median throughout. Against the source-shrunk update it is competitive only at the smallest budgets: the update's median advantage over flexible reuse is positive in 1 of 12 settings at 8 questions, 11 at 32, and 12 at 64 (Table~\ref{tab:flex-budget}). Subset summaries are descriptive.

\begin{table}[htbp]\centering\footnotesize\setlength{\tabcolsep}{4pt}\renewcommand{\arraystretch}{1.06}
\caption{\textbf{The regularized update exceeds flexible reuse in median in 11 of 12 settings at 32 questions per attribute and all 12 at 64.} Median $Q^2$ differences in points over 200 subsets of the target fitting questions at $k$ questions per attribute.}
\label{tab:flex-budget}
\begin{tabular*}{\linewidth}{@{\extracolsep{\fill}}llrrrrrrrr@{}}
\toprule
& & \multicolumn{4}{c}{Two parameters $-$ shared scaling} & \multicolumn{4}{c}{Source-shrunk update $-$ two parameters} \\
\cmidrule(lr){3-6}\cmidrule(l){7-10}
Pair & Template & $k=8$ & 16 & 32 & 64 & $k=8$ & 16 & 32 & 64 \\
\midrule
Llama-3.1 & Joint & $-0.05$ & 0.01 & $-0.03$ & 0.02 & $-1.01$ & 0.06 & 0.61 & 1.42 \\
 & Frozen & 0.16 & 0.20 & 0.22 & 0.23 & $-1.26$ & 0.24 & 1.25 & 1.81 \\
OLMo-3 & Joint & 0.03 & 0.10 & 0.14 & 0.16 & $-0.79$ & 0.24 & 0.86 & 1.97 \\
 & Frozen & 1.28 & 1.32 & 1.34 & 1.35 & $-0.03$ & 0.23 & 1.05 & 1.47 \\
Qwen2.5 & Joint & 0.38 & 0.40 & 0.41 & 0.41 & $-0.36$ & $-0.16$ & 0.34 & 0.54 \\
 & Frozen & 0.72 & 0.74 & 0.74 & 0.75 & $-0.14$ & $-0.04$ & 0.36 & 0.58 \\
Qwen3 & Joint & 0.58 & 0.62 & 0.63 & 0.64 & $-0.11$ & $-0.03$ & 0.19 & 0.38 \\
 & Frozen & 0.21 & 0.24 & 0.25 & 0.26 & $-0.41$ & $-0.09$ & 0.23 & 0.70 \\
Qwen3.5 & Joint & $-0.12$ & $-0.13$ & $-0.13$ & $-0.15$ & $-1.09$ & $-0.88$ & $-0.31$ & 0.27 \\
 & Frozen & $-0.14$ & $-0.05$ & $-0.05$ & $-0.02$ & 0.01 & 0.12 & 0.61 & 0.91 \\
T\"ulu & Joint & $-0.03$ & 0.01 & 0.03 & 0.04 & $-0.51$ & 0.63 & 1.22 & 1.63 \\
 & Frozen & 0.29 & 0.30 & 0.31 & 0.31 & $-0.27$ & 0.65 & 0.94 & 1.45 \\
\bottomrule
\end{tabular*}
\end{table}

\paragraph{Frozen questionnaire template.} The template, a 16/32-question recipe, the full-refit reference, and a two-comparison primary rule were written down on 2026-09-22, before any inference on the template. OLMo-3 and T\"ulu were then measured; these two comparisons are the prospective test. Qwen2.5 and Llama-3.1, and afterwards Qwen3 and Qwen3.5, were measured later on the unchanged template, each with its own four-comparison family; because the template had been observed, these are extensions, not independent prospective replications. Both prospective pairs pass and also show structure beyond shared scaling; Llama-3.1, Qwen2.5, and Qwen3 reproduce both contrasts, and both Qwen3.5 contrasts are unresolved (Table~\ref{tab:prospective}). The descriptive recipe asked the 32-question median to exceed full-data adapted reuse and to lie within one point of full refitting: T\"ulu meets both goals, OLMo-3 exceeds adapted reuse but misses the gap goal, Llama-3.1 and Qwen2.5 miss the gap goal, and Qwen3's gap falls below one point.

\begin{table}[htbp]\centering\small\setlength{\tabcolsep}{4pt}\renewcommand{\arraystretch}{1.1}
\caption{\textbf{What refitting gains on the frozen questionnaire.} $Q^2$ points with corrected intervals: two comparisons for the prospective OLMo-3 and T\"ulu test, four within each later extension. Bold estimates have intervals above zero. Budget gaps are medians over subsets of the difference from refitting on all 181 fitting questions, not confidence intervals.}
\label{tab:prospective}
\begin{tabular*}{\linewidth}{@{\extracolsep{\fill}}llrrrrrrrr@{}}
\toprule
& & \multicolumn{3}{c}{Refit $-$ adapted reuse} & \multicolumn{3}{c}{Refit $-$ shared} & \multicolumn{2}{c}{Gap to full refit} \\
\cmidrule(lr){3-5}\cmidrule(lr){6-8}\cmidrule(l){9-10}
Pair & Design & Est. & Lower & Upper & Est. & Lower & Upper & $k=16$ & $k=32$ \\
\midrule
Llama-3.1 & Extension & \textbf{4.26} & 2.56 & 6.18 & \textbf{2.39} & 1.25 & 3.14 & 2.67 & 1.08 \\
OLMo-3 & Prospective & \textbf{2.26} & 1.32 & 3.13 & \textbf{3.34} & 2.22 & 4.04 & 3.04 & 1.33 \\
Qwen2.5 & Extension & \textbf{1.31} & 0.04 & 2.40 & \textbf{1.78} & 0.51 & 2.42 & 3.05 & 1.27 \\
Qwen3 & Extension & \textbf{1.58} & 0.47 & 2.81 & \textbf{1.25} & 0.16 & 1.95 & 1.95 & 0.82 \\
Qwen3.5 & Extension & 1.41 & $-1.26$ & 3.23 & 1.04 & $-1.32$ & 2.36 & 5.93 & 2.31 \\
T\"ulu & Prospective & \textbf{2.62} & 1.59 & 3.83 & \textbf{1.99} & 1.20 & 2.50 & 1.97 & 0.81 \\
\bottomrule
\end{tabular*}
\end{table}

\paragraph{Unregularized crossover and allocation.} Figure~\ref{fig:budget-unregularized} shows the unregularized crossover on the 4 to 181 question grid and an equal-budget allocation comparison: measuring every attribute versus measuring a category-stratified half of the inventory more deeply, with adapted reuse for the unmeasured half. Broad coverage loses at small budgets and wins once each attribute receives enough questions; the median crossover is 16 questions in OLMo-3, T\"ulu, and Llama-3.1, 32 in Qwen3.5, and 64 in Qwen2.5 and Qwen3, for both half- and quarter-inventory comparators. Both analyses are retrospective and descriptive; per-budget values are in the supplementary material.

\begin{figure}[htbp]
\centering\includegraphics[width=.94\linewidth]{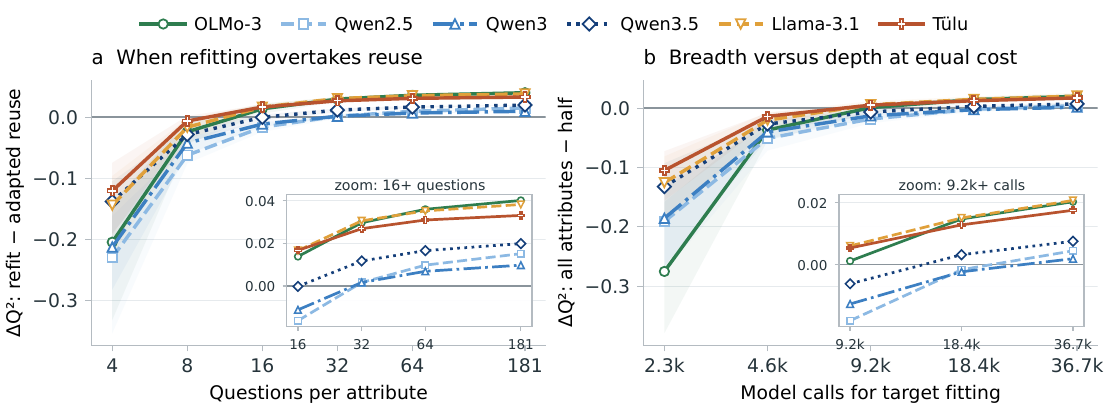}
\caption{\textbf{Unregularized crossover and breadth allocation.} Left: target refitting minus single-amplitude reuse. Right: measuring every attribute minus measuring half the attributes more deeply, at equal numbers of model calls. $\Delta Q^2$ on its original scale (0.01 is one point); medians and interquartile ranges over 200 subsets. Regularized alternatives appear in Figure~\ref{fig:targetbudget}.}
\label{fig:budget-unregularized}
\end{figure}

\section{Robustness and numerical precision}
\label{app:robustness}
These checks ask whether the main comparisons depend on the answer interface, the order of answer options, the particular evaluation questions, a few extreme attributes, or numerical precision. Each has its own protocol and decision rule, frozen before computation, and reuses the attributes, estimators, and weights of the original analysis; loaders first reproduce the original estimates within $10^{-10}$. All checks keep the forced-choice readout. Uncertainty comes from task-stratified whole-question draws that refit source and target in every draw (100,000 for the chat analysis, 20,000 otherwise). Transfer intervals are corrected across 42 directions and structure intervals across seven pairs.

\begin{table}[htbp]\centering\small\renewcommand{\arraystretch}{1.15}
\caption{\textbf{Single-amplitude reuse loses under every tested condition.} Counts refer to the 42 directed transfers with one target-fitted amplitude, corrected across 42 directions within each condition. Medians are adapted reuse minus shared scaling in $Q^2$ points. The last column counts directions in which target refitting significantly beats adapted reuse. Qwen2.5 uses bfloat16 here; float32 checks appear in Appendix~\ref{app:controls}.}
\label{tab:robust-transfer}
\begin{tabular}{@{}lrrrrr@{}}
\toprule
Condition & Below shared & Above shared & Negative point & Median & Refit above reuse \\
\midrule
Completion (original) & 29 & 0 & 40 & $-1.92$ & 31 \\
Chat interface & 29 & 0 & 42 & $-2.30$ & 31 \\
Answer orders averaged & 33 & 0 & 41 & $-2.25$ & 36 \\
Reversed order only & 29 & 0 & 42 & $-2.09$ & 30 \\
Independent questions & 28 & 0 & 39 & $-1.81$ & 28 \\
\bottomrule
\end{tabular}
\end{table}

\begin{table}[htbp]\centering\small\renewcommand{\arraystretch}{1.02}
\caption{\textbf{Most structure estimates are positive across checks, and OLMo-3 is significant in every one.} Target refitting minus shared scaling on held-out questions, in $Q^2$ points, with intervals corrected across seven pairs that include fitting uncertainty. Bold estimates have intervals above zero. Qwen2.5 uses bfloat16 here; float32 checks appear in Appendix~\ref{app:controls}.}
\label{tab:robust-structure}\label{tab:robust-structure-reversed}
\begin{tabular}{@{}llrrr@{}}
\toprule
Pair & Condition & Estimate & Lower & Upper \\
\midrule
Llama-3.1 & Completion (original) & 0.07 & $-3.46$ & 1.97 \\
 & Chat interface & 0.97 & $-3.40$ & 3.33 \\
 & Answer orders averaged & 1.44 & $-0.86$ & 3.06 \\
 & Reversed order only & 1.43 & $-0.67$ & 2.76 \\
 & Independent questions & 1.62 & $-1.04$ & 3.13 \\
\addlinespace[2pt]
OLMo-2 & Completion (original) & $-0.02$ & $-1.50$ & 0.96 \\
 & Chat interface & 0.61 & $-1.49$ & 2.13 \\
 & Answer orders averaged & 0.26 & $-1.49$ & 1.26 \\
 & Reversed order only & 0.25 & $-2.25$ & 1.41 \\
 & Independent questions & 0.22 & $-0.83$ & 0.75 \\
\addlinespace[2pt]
OLMo-3 & Completion (original) & $\mathbf{3.10}$ & $1.34$ & $4.32$ \\
 & Chat interface & $\mathbf{3.13}$ & $1.12$ & $4.42$ \\
 & Answer orders averaged & $\mathbf{3.55}$ & $2.26$ & $4.56$ \\
 & Reversed order only & $\mathbf{3.20}$ & $1.43$ & $4.55$ \\
 & Independent questions & $\mathbf{2.04}$ & $0.07$ & $3.27$ \\
\addlinespace[2pt]
Qwen2.5 & Completion (original) & $\mathbf{1.23}$ & $0.09$ & $2.01$ \\
 & Chat interface & 1.11 & $-0.19$ & 1.90 \\
 & Answer orders averaged & 1.16 & $-0.03$ & 1.85 \\
 & Reversed order only & 0.56 & $-0.88$ & 1.27 \\
 & Independent questions & 0.80 & $-2.00$ & 2.01 \\
\addlinespace[2pt]
Qwen3 & Completion (original) & 0.87 & $-0.39$ & 1.81 \\
 & Chat interface & 1.28 & $-0.11$ & 2.14 \\
 & Answer orders averaged & $\mathbf{1.04}$ & $0.15$ & $1.54$ \\
 & Reversed order only & 0.90 & $-0.42$ & 1.49 \\
 & Independent questions & 1.02 & $-0.26$ & 1.61 \\
\addlinespace[2pt]
Qwen3.5 & Completion (original) & 1.13 & $-0.33$ & 2.15 \\
 & Chat interface & 0.90 & $-1.06$ & 2.14 \\
 & Answer orders averaged & 0.59 & $-1.61$ & 2.37 \\
 & Reversed order only & $-0.53$ & $-4.37$ & 2.48 \\
 & Independent questions & $\mathbf{1.39}$ & $0.12$ & $2.25$ \\
\addlinespace[2pt]
T\"ulu & Completion (original) & 2.11 & $-0.09$ & 3.41 \\
 & Chat interface & 1.97 & $-1.23$ & 3.79 \\
 & Answer orders averaged & 2.03 & $-0.35$ & 3.47 \\
 & Reversed order only & 2.10 & $-0.27$ & 3.64 \\
 & Independent questions & $\mathbf{2.67}$ & $0.55$ & $3.61$ \\
\bottomrule
\end{tabular}
\end{table}

\paragraph{Chat interfaces.} Each post-trained checkpoint is queried through its own chat template, including any default system prompt, with the four examples rendered as prior turns; Qwen3 and Qwen3.5 disable thinking. Base checkpoints keep the completion interface, so gains here describe the change from Base under completion to the post-trained model under its deployed interface. All seven endpoints pass the frozen gate on answer-token mass (means 0.94 to 1.00). Transfer: 29 of 42 directions are significantly negative and none positive, as the frozen rule required; direction-level estimates correlate at 0.73 with the completion results and share their sign in 40 of 42. Structure: OLMo-3 keeps its advantage (3.13 versus 3.10 points), while the Qwen2.5 interval now includes zero.

\paragraph{Answer order.} The original screen presents each question in one option order. Additive position bias cancels in persona contrasts, but an interaction between persona and position would not. We rerun all 13 endpoints with options swapped and map the log odds back. Reversed and original contrasts correlate at 0.71 to 0.94 across endpoints, and 3 to 15\% of contrast energy differs between orders. Averaging the two orders, 33 of 42 transfers are significantly negative and none positive. The frozen structure rule, which required both OLMo-3 and Qwen2.5 to stay significant, is not met: all seven estimates are positive and OLMo-3 rises to 3.55 points, but the Qwen2.5 interval reaches $-0.03$.

\paragraph{Independent questions.} A separate writer with no access to results, gains, or item effects produced 45 new questions, 15 per task, from the task definitions and two examples per task; two items overlapping existing questions were replaced by a mechanical screen, and the set was frozen before inference. Predictors fitted on the original 181 questions are scored on them. Both frozen rules are met: all seven structure estimates are positive, with OLMo-3, T\"ulu, and Qwen3.5 significant, and 28 of 42 transfers are significantly negative with none positive. Direction-level estimates correlate at 0.82 with the original ones.

\paragraph{Flexible reuse under the same conditions.} The two-parameter predictor of Appendix~\ref{app:cross-model-reuse} was rerun on the chat-interface, averaged-order, and independent-question moment matrices (frozen protocols; 20,000 draws refitting source and target; correction across 42 directions). All three conditions reproduce the original conclusion under the prespecified rule (no more than two significantly positive directions and at least 36 corrected upper bounds below one point): flexible reuse never beats shared scaling significantly, its median gain stays below 0.1 points, and its corrected upper bounds lie below one point in 37, 39, and 36 of 42 directions (Table~\ref{tab:flex-conditions}).

\begin{table}[htbp]\centering\footnotesize\setlength{\tabcolsep}{3pt}\renewcommand{\arraystretch}{1.1}
\caption{\textbf{Flexible reuse adds little beyond shared scaling under every tested condition.} Counts over 42 directions with intervals corrected across directions; medians in $Q^2$ points. The last two columns count directions in which two parameters significantly beat one amplitude, and in which target refitting significantly beats two parameters. Significance and counts use unrounded bounds. Qwen2.5 uses bfloat16 here; float32 checks appear in Appendix~\ref{app:controls}.}
\label{tab:flex-conditions}
\begin{tabular*}{\linewidth}{@{\extracolsep{\fill}}lrrrrrrrr@{}}
\toprule
& \multicolumn{6}{c}{Two parameters $-$ shared} & Two $-$ one & Refit $-$ two \\
\cmidrule(lr){2-7}\cmidrule(lr){8-8}\cmidrule(l){9-9}
Condition & $>0$ & $<0$ & Point $>0$ & Median & Upper $<1$ & Upper $<0.5$ & $>0$ & $>0$ \\
\midrule
Completion (original) & 1 & 0 & 28 & 0.03 & 40 & 26 & 33 & 6 \\
Chat interface & 0 & 0 & 26 & 0.02 & 37 & 28 & 38 & 6 \\
Answer orders averaged & 0 & 0 & 26 & 0.04 & 39 & 29 & 37 & 8 \\
Independent questions & 0 & 0 & 30 & 0.07 & 36 & 23 & 32 & 1 \\
\bottomrule
\end{tabular*}
\end{table}

\subsection{Influence of the most responsive attributes}
\label{app:extreme-attributes}
Qwen2.5's post-trained contrasts are large (RMS 6.7 log odds on new questions), so a few strongly responding attributes could dominate $Q^2$. Dropping the six attributes with the largest post-trained contrast energy on the fitting questions, mostly value and attitude fields, leaves the estimates largely intact (Table~\ref{tab:extreme-attributes}): Qwen2.5 keeps a similar estimate (1.05 versus 1.23 points) although its interval now includes zero, OLMo-3 stays significant, and T\"ulu becomes significant.

\begin{table}[htbp]\centering\small\setlength{\tabcolsep}{4pt}\renewcommand{\arraystretch}{1.06}
\caption{\textbf{Dropping the most responsive attributes leaves the structure estimates largely intact.} Cell gains minus shared scaling on new questions, in $Q^2$ points, with intervals corrected across seven pairs; bold estimates have intervals above zero. Qwen2.5 uses bfloat16 here; float32 checks appear in Appendix~\ref{app:controls}.}
\label{tab:extreme-attributes}
\begin{tabular*}{\linewidth}{@{\extracolsep{\fill}}lrrrrrr@{}}
\toprule
& \multicolumn{3}{c}{All 57 attributes} & \multicolumn{3}{c}{Six most responsive dropped} \\
\cmidrule(lr){2-4}\cmidrule(l){5-7}
Pair & Est. & Lower & Upper & Est. & Lower & Upper \\
\midrule
Llama-3.1 & 0.07 & $-3.45$ & 1.98 & 0.36 & $-3.31$ & 2.39 \\
OLMo-2 & $-0.02$ & $-1.49$ & 0.96 & 0.14 & $-1.40$ & 1.15 \\
OLMo-3 & $\mathbf{3.10}$ & $1.34$ & $4.38$ & $\mathbf{2.44}$ & $0.22$ & $3.84$ \\
Qwen2.5 & $\mathbf{1.23}$ & $0.08$ & $1.97$ & 1.05 & $-0.22$ & 1.77 \\
Qwen3 & 0.87 & $-0.40$ & 1.82 & 0.78 & $-0.73$ & 1.93 \\
Qwen3.5 & 1.13 & $-0.33$ & 2.15 & 1.48 & $-0.18$ & 2.63 \\
T\"ulu & 2.11 & $-0.09$ & 3.43 & $\mathbf{2.91}$ & $0.19$ & $4.50$ \\
\bottomrule
\end{tabular*}
\end{table}

\subsection{Numerical checks and the float32 rerun}
\label{app:controls}
Saved outputs carry model and prompt fingerprints, and independent re-implementations reproduce the principal estimates to within $1.4\times10^{-15}$. Numerical robustness is checked more narrowly. In bfloat16, batch composition can change individual log odds: repeats with batch size 16 on fixed subsets give maximum per-endpoint discrepancies of 0.125 to 1.0 log odds for Qwen3 and Qwen3.5, and 2.0 to 2.25 for the Qwen2.5 Instruct prompt readout. Float32 evaluation of three attributes in the OLMo-3 screen changes adjacent-stage $\Delta Q^2$ by at most $1.6\times10^{-5}$; the Qwen2.5 screen was rerun in full (below). Such spot checks do not establish invariance of every decision.

\paragraph{Qwen2.5 core screen in float32.} Because Qwen2.5 is one of the two pairs with significant structure and its profile study changed under float32, both core-screen endpoints were rerun on the complete frozen design in float32 with TF32 disabled (protocol frozen before inference; the loaders applied to the bfloat16 outputs reproduce the stored moments exactly). The precision discrepancy is far smaller than in the profile study: the median relative difference $\|d_{\rm bf16}-d_{\rm fp32}\|/\|d_{\rm bf16}\|$ over attributes is 0.077 for Base and 0.052 for Instruct, and the RMS difference (0.086 and 0.31 log odds) is 7 and 5\% of the float32 contrast RMS (1.17 and 5.97). The main qualitative findings persist, although one marginal significance result and the count of bounds below one point change. New-question structure moves from 1.23 to 1.28 points, with refit interval 0.15 to 2.07 and fixed-gain interval 0.54 to 2.01; the additive ablation gives 1.19; shared scaling leaves 0.331 of the no-change error; category summaries of $\kappa$ and $\rho_0$ change by at most 0.07 and 0.005. Across the 12 transfer directions involving Qwen2.5 (Table~\ref{tab:fp32-core}), every one-amplitude loss keeps its status and refitting still beats flexible reuse significantly in the direction into OLMo-3. Two counts move: the marginal two-parameter gain from Llama-3.1 into Qwen2.5 (lower bound 0.0005 in bfloat16) is no longer significant, so no direction shows a significant flexible-reuse gain, and the upper bound for Qwen2.5 into OLMo-3 rises to 1.001, so 39 rather than 40 of 42 bounds lie below one point. Under the frozen rule the float32 values are primary for Qwen2.5 in Tables~\ref{tab:prediction}, \ref{tab:crossmodel-gates}, and \ref{tab:olmo-ablation}; Figures~\ref{fig:breadth} and~\ref{fig:stages} and Table~\ref{tab:all-crossmodel} keep the bfloat16 values.

\begin{table}[htbp]\centering\footnotesize\setlength{\tabcolsep}{2.6pt}\renewcommand{\arraystretch}{1.06}
\caption{\textbf{Transfers involving Qwen2.5 with its float32 outputs.} Differences in $Q^2$ points with intervals corrected across the 42 directions of the full analysis (20,000 draws); bold estimates have intervals that exclude zero. Other pairs remain in bfloat16. Significance and counts use unrounded bounds.}
\label{tab:fp32-core}
\begin{tabular*}{\linewidth}{@{\extracolsep{\fill}}llrrrrrrrrr@{}}
\toprule
& & \multicolumn{3}{c}{One amplitude $-$ shared} & \multicolumn{3}{c}{Two parameters $-$ shared} & \multicolumn{3}{c}{Refit $-$ two} \\
\cmidrule(lr){3-5}\cmidrule(lr){6-8}\cmidrule(l){9-11}
Source & Target & Est. & Lower & Upper & Est. & Lower & Upper & Est. & Lower & Upper \\
\midrule
Llama-3.1 & Qwen2.5 & $\mathbf{-6.10}$ & $-9.06$ & $-4.57$ & 0.22 & $-0.0049$ & 0.50 & 1.07 & $-0.27$ & 2.06 \\
OLMo-2 & Qwen2.5 & $\mathbf{-2.09}$ & $-3.68$ & $-1.24$ & 0.25 & $-0.07$ & 0.72 & 1.03 & $-0.27$ & 1.96 \\
OLMo-3 & Qwen2.5 & $\mathbf{-5.44}$ & $-8.82$ & $-3.81$ & 0.37 & $-0.04$ & 0.91 & 0.91 & $-0.44$ & 1.79 \\
Qwen3 & Qwen2.5 & $\mathbf{-1.49}$ & $-3.45$ & $-0.48$ & $-0.03$ & $-0.36$ & 0.19 & 1.32 & $-0.0018$ & 2.32 \\
Qwen3.5 & Qwen2.5 & $\mathbf{-2.14}$ & $-4.48$ & $-0.85$ & 0.04 & $-0.22$ & 0.36 & 1.24 & $-0.05$ & 2.15 \\
T\"ulu & Qwen2.5 & $\mathbf{-3.90}$ & $-6.33$ & $-2.60$ & 0.09 & $-0.14$ & 0.41 & 1.19 & $-0.12$ & 2.14 \\
Qwen2.5 & Llama-3.1 & $\mathbf{-2.20}$ & $-3.80$ & $-0.62$ & 0.08 & $-0.62$ & 0.47 & $-0.01$ & $-4.14$ & 2.03 \\
Qwen2.5 & OLMo-2 & $\mathbf{-2.77}$ & $-4.75$ & $-1.24$ & 0.10 & $-0.32$ & 0.38 & $-0.12$ & $-1.72$ & 0.89 \\
Qwen2.5 & OLMo-3 & $\mathbf{-3.39}$ & $-5.57$ & $-2.11$ & 0.39 & $-0.09$ & 1.0011 & $\mathbf{2.70}$ & $0.65$ & $4.31$ \\
Qwen2.5 & Qwen3 & $-1.24$ & $-3.67$ & 0.15 & 0.05 & $-0.32$ & 0.37 & 0.82 & $-0.96$ & 2.11 \\
Qwen2.5 & Qwen3.5 & $\mathbf{-2.51}$ & $-4.80$ & $-1.35$ & 0.03 & $-0.16$ & 0.29 & 1.11 & $-0.60$ & 2.37 \\
Qwen2.5 & T\"ulu & $\mathbf{-2.04}$ & $-3.81$ & $-0.96$ & 0.02 & $-0.24$ & 0.27 & 2.09 & $-0.79$ & 3.57 \\
\bottomrule
\end{tabular*}
\end{table}

\paragraph{Qwen2.5 profiles in float32.} Qwen2.5's profile contrasts are small, and a two-identity diagnostic had shown large bfloat16 versus float32 differences. We therefore reran both Qwen2.5 endpoints on the complete frozen profile design in float32, with TF32 disabled and everything else unchanged, under a protocol frozen before inference whose rule replaces bfloat16 with float32 results whenever any comparison changes significance. On the full panels, the median relative difference $\|d_{\rm bf16}-d_{\rm fp32}\|/\|d_{\rm bf16}\|$ is 0.868 for Base and 0.397 for Instruct; for Base the precision discrepancy (RMS 0.079 log odds) exceeds the float32 signal itself (0.064). All four insertion comparisons keep their significance, with smaller refitting advantages; in the replacement arm, refitting minus adapted reuse and adapted reuse minus shared scaling lose significance, and source-only transfer loses its significance in both arms, although its estimates remain negative ($-0.19$ and $-0.12$ points) (Table~\ref{tab:fp32-profile}). The paper reports the float32 values. The pooled Base attenuation ratio falls from 5.4\% to 3.3\%. The other five profile pairs remain in bfloat16 with batch-size checks, and the Qwen2.5 discrepancy is not a noise floor for them.

\begin{table}[htbp]\centering\small\renewcommand{\arraystretch}{1.12}
\caption{\textbf{Float32 evaluation changes comparisons in both profile arms; refitting advantages for inserted traits remain significant.} Differences in $Q^2$ points with intervals corrected across 16 comparisons (source-only rows use 98.75\% intervals), for the bfloat16 and float32 evaluations of the same frozen design. Bold estimates have intervals that exclude zero. Source only compares source cell gains with a source shared gain, both fitted on short statements and applied without target adjustment.}
\label{tab:fp32-profile}
\begin{tabular}{@{}llrrrrrr@{}}
\toprule
& & \multicolumn{3}{c}{bfloat16} & \multicolumn{3}{c}{float32} \\
\cmidrule(lr){3-5}\cmidrule(l){6-8}
Arm & Comparison & Est. & Lower & Upper & Est. & Lower & Upper \\
\midrule
Replacement & Source only & $\mathbf{-3.88}$ & $-8.07$ & $-2.10$ & $-0.19$ & $-0.62$ & 0.14 \\
Replacement & Refit $-$ shared & 6.52 & $-0.09$ & 22.42 & 5.41 & $-6.72$ & 18.21 \\
Replacement & Adapted $-$ shared & $\mathbf{-1.16}$ & $-2.74$ & $-0.25$ & $-1.65$ & $-4.25$ & 0.67 \\
Replacement & Refit $-$ adapted & $\mathbf{7.68}$ & $0.49$ & $24.16$ & 7.06 & $-5.68$ & 18.70 \\
Replacement & Refit $-$ task-shared & 2.68 & $-2.63$ & 14.08 & 4.45 & $-5.10$ & 21.00 \\
\addlinespace[2pt]
Insertion & Source only & $\mathbf{-3.02}$ & $-5.56$ & $-1.53$ & $-0.12$ & $-0.95$ & 0.58 \\
Insertion & Refit $-$ shared & $\mathbf{18.90}$ & $9.41$ & $25.45$ & $\mathbf{12.74}$ & $5.43$ & $20.04$ \\
Insertion & Adapted $-$ shared & $\mathbf{-3.04}$ & $-5.02$ & $-1.31$ & $\mathbf{-4.38}$ & $-6.90$ & $-2.06$ \\
Insertion & Refit $-$ adapted & $\mathbf{21.94}$ & $11.46$ & $28.79$ & $\mathbf{17.12}$ & $9.37$ & $24.58$ \\
Insertion & Refit $-$ task-shared & $\mathbf{16.39}$ & $8.82$ & $22.89$ & $\mathbf{12.66}$ & $5.54$ & $20.93$ \\
\bottomrule
\end{tabular}
\end{table}

\section{Composition}
\label{app:composition}
\subsection{Country cues combined with preferences}
The design crosses the USA/India country pair with three preference domains (risk, planning, social), eight questions per domain, three strengths, two phrasings, both information orders, and both answer orders. The country pair was chosen in a broader 72-question comparison. Each two-pair study contributes 36,096 evaluations; repeated conditions do not add independent questions. Whole-question bootstrap intervals correct for four comparisons within each study (OLMo-3/T\"ulu, Qwen2.5/Llama-3.1, Qwen3/Qwen3.5). The statistic $C=\log(\mathrm{RMS}_{\rm relevant}/\mathrm{RMS}_{\rm unrelated})$ compares the country effect under a relevant preference with the effect under an unrelated one, and a training change in $C$ is a change in \emph{relative} influence.

Training lowers $C$ in T\"ulu, Qwen2.5, and Qwen3 and raises it in Qwen3.5 (Table~\ref{tab:competition-six-pairs}); in T\"ulu the RMS ratio falls from 1.180 to 0.761, and in Qwen2.5 from 0.975 to 0.472. OLMo-3 and Llama-3.1 show no significant change. Unrelated preferences still move choices, especially for risk and planning, so they are not inert distractors. Leave-one-question and alternative aggregation checks retain T\"ulu's direction, and a repeat of the Qwen3 and Qwen3.5 subset with another batch size keeps their opposite directions. A matched control with non-demographic cues leaves the country-specific residual unresolved, so the data do not establish a competition mechanism unique to demographic cues.

\begin{table}[htbp]\centering\small
\caption{\textbf{Preference relevance and preference strength ask different questions.} Estimates with lower and upper bounds corrected within each study; bold estimates have intervals that exclude zero. A negative relevance change means training reduces country influence under relevant preferences relative to unrelated preferences. The additional effect of strong versus weak preferences increases in Qwen3.5 and remains unresolved in the other five pairs.}
\label{tab:competition-six-pairs}
\begin{tabular*}{\linewidth}{@{\extracolsep{\fill}}lrrrrrr@{}}
\toprule
& \multicolumn{3}{c}{Relevance-specific stage change} & \multicolumn{3}{c}{Additional strength-gradient change} \\
\cmidrule(lr){2-4}\cmidrule(l){5-7}
Pair & Estimate & Lower & Upper & Estimate & Lower & Upper \\
\midrule
Llama & $-0.1543$ & $-0.3681$ & $0.0539$ & $0.1815$ & $-0.1997$ & $0.5800$ \\
OLMo & $0.1198$ & $-0.1077$ & $0.4784$ & $-0.2200$ & $-0.5040$ & $0.1474$ \\
Qwen2.5 & $\mathbf{-0.7252}$ & $-0.9206$ & $-0.5554$ & $-0.2293$ & $-0.5420$ & $0.0949$ \\
Qwen3 & $\mathbf{-0.7904}$ & $-0.9756$ & $-0.5878$ & $-0.1880$ & $-0.5116$ & $0.0864$ \\
Qwen3.5 & $\mathbf{0.6790}$ & $0.4650$ & $0.9260$ & $\mathbf{0.2910}$ & $0.0487$ & $0.4978$ \\
T\"ulu & $\mathbf{-0.4381}$ & $-0.6594$ & $-0.2635$ & $-0.2252$ & $-0.4967$ & $0.0912$ \\
\bottomrule
\end{tabular*}
\end{table}

\subsection{Full profiles}
\label{app:qwen-profile-extension}
Profile construction is described in Appendix~\ref{app:design}. Inside profiles, Base trait effects shrink to 2.2\% (Qwen2.5 replacement, float32) to 9.9\% (Qwen3 insertion) of their isolated size across the six pairs and both arms; post-trained ratios, where reported, are similar or larger (11.3\% for Qwen2.5 and 8.9 to 12.5\% for Qwen3 and Qwen3.5). Neither ratio is a signal-to-noise ratio, and weaker contrasts alone do not identify cue competition, length effects, or suppression.

Table~\ref{tab:profile-target-shared} gives the three comparisons for every pair and arm. Source only compares short-statement cell gains with a short-statement shared gain, with no target adjustment; it is significantly negative in four settings. Refitting on profile data beats shared scaling and adapted reuse for inserted traits in Qwen2.5, Qwen3, Llama-3.1, and Qwen3.5; no replacement setting shows a significant refitting advantage. Table~\ref{tab:profile-baselines} completes the comparison with the two remaining baselines: adapted reuse against shared scaling, and refitting against a task-shared gain that lets scale differ by task. Adapted reuse is significantly worse than shared scaling for inserted traits in Qwen2.5, Qwen3, Qwen3.5, and T\"ulu. Against the task-shared baseline, the insertion advantage remains significant in Llama-3.1, Qwen2.5, and Qwen3 but not in Qwen3.5, which therefore does not establish attribute-specific value beyond task-dependent scaling. The profiles test controlled edits inside synthetic descriptions, and cross-validation splits questions while keeping identities, so the results do not extend to unseen people.

\begin{table}[htbp]\centering\footnotesize\setlength{\tabcolsep}{2.8pt}\renewcommand{\arraystretch}{1.06}
\caption{\textbf{Do gains learned from isolated traits predict effects inside full profiles?} Differences in $Q^2$ points on held-out profile questions; positive values favor the first predictor named. Intervals include source fitting and resampling of profiles and questions, corrected within each study family (24 comparisons for OLMo-3/T\"ulu, 16 each for Qwen2.5/Llama-3.1 and Qwen3/Qwen3.5; source-only rows use four). Bold estimates have intervals that exclude zero. Qwen2.5 uses float32; other pairs bfloat16. Source only compares source cell gains with a source shared gain, both fitted on short statements and applied without target adjustment.}
\label{tab:profile-target-shared}
\label{tab:alltarget-profile}
\begin{tabular*}{\linewidth}{@{\extracolsep{\fill}}llrrrrrrrrr@{}}
\toprule
& & \multicolumn{3}{c}{Source only} & \multicolumn{3}{c}{Refit $-$ shared} & \multicolumn{3}{c}{Refit $-$ adapted} \\
\cmidrule(lr){3-5}\cmidrule(lr){6-8}\cmidrule(l){9-11}
Pair & Edit & Est. & Lower & Upper & Est. & Lower & Upper & Est. & Lower & Upper \\
\midrule
Llama-3.1 & Replacement & $\mathbf{-10.22}$ & $-16.85$ & $-5.99$ & 0.39 & $-7.77$ & 4.84 & 1.49 & $-7.20$ & 7.03 \\
 & Insertion & 0.72 & $-0.59$ & 1.71 & \textbf{12.41} & 1.60 & 24.67 & \textbf{11.67} & 1.13 & 23.57 \\
OLMo-3 & Replacement & 0.56 & $-1.37$ & 3.19 & 0.63 & $-8.48$ & 7.20 & 0.57 & $-7.92$ & 8.51 \\
 & Insertion & 0.77 & $-0.26$ & 1.79 & 3.00 & $-1.27$ & 4.83 & 2.41 & $-1.59$ & 4.16 \\
Qwen2.5 & Replacement & $-0.19$ & $-0.62$ & 0.14 & 5.41 & $-6.72$ & 18.21 & 7.06 & $-5.68$ & 18.70 \\
 & Insertion & $-0.12$ & $-0.95$ & 0.58 & \textbf{12.74} & 5.43 & 20.04 & \textbf{17.12} & 9.37 & 24.58 \\
Qwen3 & Replacement & $-1.79$ & $-4.14$ & 0.02 & 4.63 & $-3.10$ & 11.49 & 5.60 & $-2.43$ & 13.47 \\
 & Insertion & $\mathbf{-5.67}$ & $-7.57$ & $-4.29$ & \textbf{12.63} & 4.65 & 18.28 & \textbf{16.36} & 7.44 & 22.96 \\
Qwen3.5 & Replacement & $-0.08$ & $-1.67$ & 1.67 & 2.61 & $-3.85$ & 7.57 & 3.00 & $-3.46$ & 8.25 \\
 & Insertion & $\mathbf{-1.85}$ & $-2.51$ & $-1.07$ & \textbf{7.09} & 0.44 & 12.65 & \textbf{8.96} & 2.03 & 15.18 \\
T\"ulu & Replacement & $-1.25$ & $-4.36$ & 0.86 & $-0.27$ & $-10.44$ & 4.17 & $-0.18$ & $-10.47$ & 4.66 \\
 & Insertion & $\mathbf{-3.97}$ & $-5.86$ & $-2.62$ & 8.93 & $-4.63$ & 18.81 & 10.35 & $-3.86$ & 20.71 \\
\bottomrule
\end{tabular*}
\end{table}

\begin{table}[htbp]\centering\footnotesize\setlength{\tabcolsep}{3.5pt}\renewcommand{\arraystretch}{1.06}
\caption{\textbf{Remaining profile baselines.} Differences in $Q^2$ points on held-out profile questions, with the same corrections as Table~\ref{tab:profile-target-shared}; bold estimates have intervals that exclude zero. Qwen2.5 uses float32.}
\label{tab:profile-baselines}
\begin{tabular*}{\linewidth}{@{\extracolsep{\fill}}llrrrrrr@{}}
\toprule
& & \multicolumn{3}{c}{Adapted reuse $-$ shared} & \multicolumn{3}{c}{Refit $-$ task-shared} \\
\cmidrule(lr){3-5}\cmidrule(l){6-8}
Pair & Edit & Est. & Lower & Upper & Est. & Lower & Upper \\
\midrule
Llama-3.1 & Replacement & $-1.10$ & $-3.14$ & 0.01 & $-1.08$ & $-7.66$ & 2.71 \\
 & Insertion & 0.74 & $-0.92$ & 2.18 & \textbf{11.26} & 0.95 & 22.31 \\
OLMo-3 & Replacement & 0.06 & $-2.43$ & 2.38 & $-0.43$ & $-10.48$ & 4.56 \\
 & Insertion & 0.58 & $-0.60$ & 1.77 & 2.41 & $-0.80$ & 3.99 \\
Qwen2.5 & Replacement & $-1.65$ & $-4.25$ & 0.67 & 4.45 & $-5.10$ & 21.00 \\
 & Insertion & $\mathbf{-4.38}$ & $-6.90$ & $-2.06$ & \textbf{12.66} & 5.54 & 20.93 \\
Qwen3 & Replacement & $-0.97$ & $-3.07$ & 0.17 & 0.98 & $-5.58$ & 5.84 \\
 & Insertion & $\mathbf{-3.73}$ & $-5.76$ & $-2.03$ & \textbf{11.57} & 4.16 & 17.24 \\
Qwen3.5 & Replacement & $-0.39$ & $-1.91$ & 0.71 & 0.89 & $-4.85$ & 4.60 \\
 & Insertion & $\mathbf{-1.86}$ & $-3.94$ & $-0.36$ & 5.07 & $-0.39$ & 9.59 \\
T\"ulu & Replacement & $-0.10$ & $-0.89$ & 0.54 & $-0.90$ & $-9.69$ & 3.34 \\
 & Insertion & $\mathbf{-1.42}$ & $-2.79$ & $-0.34$ & 7.98 & $-3.71$ & 16.36 \\
\bottomrule
\end{tabular*}
\end{table}

\paragraph{Absolute error.} Table~\ref{tab:profile-rmse} converts the profile comparisons into root mean square error in log odds. Profile contrasts are small in most pairs (post-trained RMS 0.12 to 0.35 log odds) and near one log odd only in Qwen2.5 and Qwen3. Refitting lowers the error of shared scaling by 12.9\% (Qwen2.5) and 9.0\% (Qwen3) for inserted traits, by 5.6 to 7.6\% for inserted traits in Llama-3.1, Qwen3.5, and T\"ulu, and by at most 4.4\% elsewhere, so a large relative $Q^2$ gain in a profile corresponds to a change of a few hundredths to a tenth of a log odd.

\begin{table}[htbp]\centering\footnotesize\setlength{\tabcolsep}{4pt}\renewcommand{\arraystretch}{1.06}
\caption{\textbf{Profile prediction error in log odds.} Post-trained profile contrast RMS and held-out RMSE of shared scaling, adapted reuse, and target refitting under question cross-validation; the last column is the relative reduction from shared scaling to refitting. Qwen2.5 uses float32; other pairs bfloat16.}
\label{tab:profile-rmse}
\begin{tabular*}{\linewidth}{@{\extracolsep{\fill}}llrrrrr@{}}
\toprule
Pair & Edit & Contrast RMS & Shared & Adapted & Refit & Reduction (\%) \\
\midrule
Llama-3.1 & Replacement & 0.203 & 0.197 & 0.198 & 0.196 & +0.2 \\
 & Insertion & 0.353 & 0.325 & 0.323 & 0.300 & +7.6 \\
OLMo-3 & Replacement & 0.122 & 0.111 & 0.110 & 0.110 & +0.4 \\
 & Insertion & 0.124 & 0.115 & 0.115 & 0.113 & +1.8 \\
Qwen2.5 & Replacement & 0.825 & 0.657 & 0.666 & 0.629 & +4.4 \\
 & Insertion & 1.096 & 0.797 & 0.830 & 0.695 & +12.9 \\
Qwen3 & Replacement & 0.676 & 0.624 & 0.628 & 0.607 & +2.8 \\
 & Insertion & 0.992 & 0.850 & 0.871 & 0.773 & +9.0 \\
Qwen3.5 & Replacement & 0.157 & 0.139 & 0.139 & 0.136 & +1.7 \\
 & Insertion & 0.207 & 0.162 & 0.164 & 0.152 & +6.0 \\
T\"ulu & Replacement & 0.206 & 0.201 & 0.201 & 0.202 & -0.1 \\
 & Insertion & 0.267 & 0.243 & 0.245 & 0.230 & +5.6 \\
\bottomrule
\end{tabular*}
\end{table}

\paragraph{Flexible reuse inside profiles.}
\label{app:flex-profile}
We reran the profile analysis for all six pairs and both arms, adding the two-parameter predictor fitted on the same target folds with bootstrap-refitted source gains (frozen protocol; same seeds; 10,000 crossed identity and question draws; correction across 12 settings). For inserted traits, refitting significantly beats flexible reuse in Qwen2.5, Llama-3.1, Qwen3, and Qwen3.5; flexible reuse itself improves on shared scaling in Qwen2.5 and Qwen3, so the source pattern carries some value there. No replacement comparison is significant (Table~\ref{tab:flex-profile}).

\begin{table}[htbp]\centering\footnotesize\setlength{\tabcolsep}{4pt}\renewcommand{\arraystretch}{1.06}
\caption{\textbf{Inserted traits still favor refitting over flexible reuse.} Differences in $Q^2$ points on held-out profile questions, with intervals corrected across 12 settings; bold estimates have intervals that exclude zero.}
\label{tab:flex-profile}
\begin{tabular*}{\linewidth}{@{\extracolsep{\fill}}llrrrrrr@{}}
\toprule
& & \multicolumn{3}{c}{Flexible reuse $-$ shared} & \multicolumn{3}{c}{Refit $-$ flexible reuse} \\
\cmidrule(lr){3-5}\cmidrule(l){6-8}
Pair & Arm & Est. & Lower & Upper & Est. & Lower & Upper \\
\midrule
Llama-3.1 & Replacement & 0.23 & $-1.84$ & 2.47 & 0.16 & $-7.56$ & 4.88 \\
 & Insertion & 0.77 & $-0.22$ & 2.12 & $\mathbf{11.64}$ & $1.71$ & $23.16$ \\
\addlinespace[2pt]
OLMo-3 & Replacement & 0.15 & $-1.60$ & 2.67 & 0.48 & $-7.25$ & 6.70 \\
 & Insertion & 0.55 & $-0.11$ & 1.56 & 2.44 & $-1.13$ & 4.06 \\
\addlinespace[2pt]
Qwen2.5 & Replacement & 0.48 & $-1.31$ & 2.17 & 4.92 & $-6.31$ & 17.54 \\
 & Insertion & $\mathbf{2.48}$ & $0.16$ & $5.04$ & $\mathbf{10.27}$ & $3.42$ & $17.75$ \\
\addlinespace[2pt]
Qwen3 & Replacement & 1.77 & $-0.82$ & 6.43 & 2.86 & $-3.92$ & 8.53 \\
 & Insertion & $\mathbf{1.81}$ & $0.28$ & $3.63$ & $\mathbf{10.82}$ & $3.92$ & $15.95$ \\
\addlinespace[2pt]
Qwen3.5 & Replacement & $-0.46$ & $-2.40$ & 0.71 & 3.07 & $-2.33$ & 7.61 \\
 & Insertion & 0.45 & $-0.95$ & 2.12 & $\mathbf{6.64}$ & $0.36$ & $12.25$ \\
\addlinespace[2pt]
T\"ulu & Replacement & $-0.12$ & $-1.12$ & 0.43 & $-0.15$ & $-8.91$ & 3.39 \\
 & Insertion & 0.07 & $-0.22$ & 0.56 & 8.86 & $-3.86$ & 17.72 \\
\bottomrule
\end{tabular*}
\end{table}

\section{Human surveys}
\label{app:human}
\paragraph{Definitions.} For survey-weighted human proportions $h_{qk}$ of group $k$ and model probabilities $\bar p_m(q,k)$ averaged over answer orders, human cross entropy, persona sensitivity, and selection regret are
\begin{align}
 L(m)&=\E_q\sum_k w_k\,\CE(h_{qk},\bar p_m(q,k)), \label{eq:ce}\\
 S(m)&=\E_q\sum_{k<k'}\pi_{kk'}\left|\bar p_m(q,k)-\bar p_m(q,k')\right|, \label{eq:sensitivity}\\
 R_{\rm select}&=L_{\rm te}\!\left(\arg\max_m S_{\rm tr}(m)\right)-L_{\rm te}\!\left(\arg\min_m L_{\rm tr}(m)\right). \label{eq:regret}
\end{align}
Temperature calibration rescales the clipped logit $\bar z_m=\operatorname{logit}\operatorname{clip}(\bar p_m,10^{-6},1-10^{-6})$ by one scalar per checkpoint, fitted on development questions:
\begin{equation}
 \widehat t_m=\argmin_{t\in[\log 0.05,\log 50]}
 \E_{q\in Q_{\rm dev}}\sum_k w_k\,
 \CE\!\left(h_{qk},\sigmoid(e^{-t}\bar z_m(q,k))\right).
 \label{eq:temperature}
\end{equation}
Any positive scalar preserves the ordering of groups within a question, so calibration cannot reverse an incorrect group difference. Earlier calibration variants with other bounds or penalties are described in the supplementary material.

\begin{figure}[t]
\centering
\includegraphics[width=.94\linewidth]{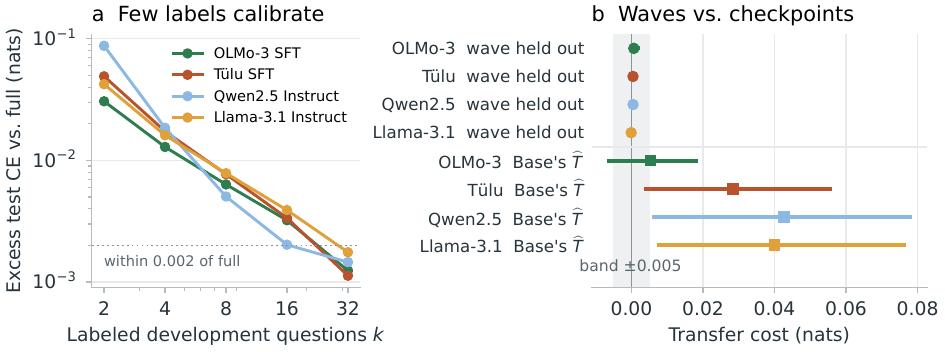}
\caption{\textbf{How many human labels does calibration need, and can it be reused?} Left: extra test loss compared with calibration on 84 development questions with group distribution labels, averaged across label subsets. Right: extra loss when temperatures are reused across survey waves or checkpoints, with 99.375\% intervals. The shaded reference band is $\pm0.005$ nats. Intervals for reuse across waves fall within $\pm0.005$ nats in the four main families. Reuse across checkpoints adds loss in Qwen2.5 and Llama; the difference remains unresolved in OLMo and T\"ulu.}
\label{fig:caltransfer}
\end{figure}

\paragraph{Survey targets.} OpinionQA supplies question/option mappings and survey responses. We keep native two-answer items without collapsing substantive multilevel answers. Human proportions use each wave's survey weights and only valid responses. Eligible groups require at least 100 valid respondents and a Kish effective size $(\sum_iw_i)^2/\sum_iw_i^2\ge100$. The main suite contains 84 questions from 14 waves, selected by fixed hash order with at most six per wave. A further 72 replication questions exclude exact matches, common battery roots, and close textual duplicates; they may share survey waves and respondents with the development questions.

\paragraph{Groups, weights, and readout.} The main groups are ages 18--29, 30--49, 50--64, and 65+, and recorded sex Male/Female. Loss weights give equal mass to the age and sex dimensions and uniform mass within each. Sensitivity weights compare only values within a dimension: each of the six age pairs receives $1/12$ and the sex pair $1/2$, while cross-dimension pairs receive zero. The generic persona states that the respondent is an adult living in the United States; the no-attribute condition keeps that description and omits only the group attribute. For original and reversed option-order log-odds $z^{(0)}$ and $z^{(1)}$, the probability used in survey analyses is
\begin{equation}
 \bar p=\tfrac12[\sigmoid(z^{(0)})+\sigmoid(-z^{(1)})].
\end{equation}
Probabilities are clipped to $[10^{-6},1-10^{-6}]$ for loss and logit computation. Five question-grouped folds separate selection and evaluation, and bootstrap resamples repeat the selection; ties prefer Base. The eleven-dimension suite adds religious attendance, marital status, metropolitan residence, volunteering, education, income, party affiliation, ideology, and region, weighting dimensions equally. Per-dimension results are descriptive rather than eleven independent claims per family.

\begin{table}[ht]
\centering\small
\caption{Selection regret across all admitted model families on the original survey dimensions. Raw intervals are 97.5\% for the original OLMo-3, T\"ulu, Qwen2.5, and Llama comparisons and 99\% for the later extensions. Calibrated controls repeat fitting and selection within training folds. OLMo comparisons select among four checkpoints; other rows compare two.}
\begin{tabular}{lrrrr}
\toprule
Model & Raw regret & Lower & Upper & Calibrated regret \\
\midrule
OLMo-2 & 0.611 & 0.421 & 0.823 & 0.000 \\
OLMo-3 & 0.233 & 0.149 & 0.341 & 0.000 \\
Qwen2.5 & 1.261 & 0.892 & 1.649 & -0.001 \\
Qwen3 & 0.751 & 0.573 & 0.956 & -0.002 \\
Qwen3.5 & 0.000 & -0.005 & 0.065 & 0.004 \\
Llama-3.1 & 0.215 & 0.138 & 0.302 & 0.000 \\
T\"ulu & 0.145 & 0.092 & 0.200 & -0.005 \\
Mistral & 0.070 & -0.032 & 0.181 & 0.000 \\
\bottomrule
\end{tabular}
\end{table}

\paragraph{Replication.} On the eleven-dimension suite, pooled raw regret is 0.329, 0.168, 1.414, and 0.224 nats for OLMo-3, T\"ulu, Qwen2.5, and Llama; these four pooled comparisons and the four dimension-calibration comparisons use 99.375\% intervals. On the 72 replication questions, raw regret is 0.240, 0.150, 1.138, and 0.188 nats for the same four families. On the age/sex dimensions, OLMo-2 and Qwen3 reach 0.611 $[0.421,0.823]$ and 0.751 $[0.573,0.956]$ with 99\% intervals, Mistral is 0.070 $[-0.032,0.181]$, and Qwen3.5 is 0.000 $[-0.005,0.065]$.

\paragraph{Sampled answers and interface scope.}
The sampled readout draws 64 short answers with an eight-token budget per condition. Regret is 0.154 for T\"ulu, 0.594 for Qwen2.5, and 0.217 for Llama, with corrected lower bounds above 0.005. OLMo's post-trained checkpoints do not yield reliably parsed answers, so no matched sampled comparison is reported for OLMo. Some probability interfaces also place low mass on the answer tokens; these conditional A/B distributions should not be read as unconstrained generation quality. A zero-shot structural format and a few other checkpoints likewise lack a usable matched interface, and no results are imputed for them.

\paragraph{Which checkpoint each rule selects.} In all eight comparisons the more persona-sensitive checkpoint is the post-trained one; for example, the sensitivity score rises from 0.078 to 0.160 in Qwen3, from 0.069 to 0.102 in OLMo-2 SFT, and from 0.036 to 0.078 in Mistral, and only slightly from 0.060 to 0.063 in Qwen3.5. Development-set human loss selects Base in the seven comparisons with positive regret and the post-trained checkpoint in Qwen3.5, where both rules agree. Sensitivity selection and post-training are therefore confounded in these data. Calibration removing the regret indicates that the cost lies mainly in the probability scale of post-trained outputs; it does not show that the calibrated models simulate people accurately.

\paragraph{Absolute fit after calibration.} Regret compares two selection rules, not fidelity. On the four-dimension survey suite (religious attendance, marital status, metropolitan residence, volunteering; 72 test questions), calibrated post-trained checkpoints reach a weighted cross entropy of 0.661 to 0.680 nats, against 0.582 for the entropy of the human distributions and 0.693 for a 50/50 prediction (Table~\ref{tab:absolute-fit}). Calibrated models therefore close only 12 to 29\% of the gap between an uninformative guess and a perfect match.

\begin{table}[htbp]\centering\small\renewcommand{\arraystretch}{1.08}
\caption{\textbf{Calibrated models remain far from the human distributions.} Weighted cross entropy in nats on 72 held-out questions of the four-dimension suite. The share of gap closed is $(0.693-\text{CE})/(0.693-0.582)$.}
\label{tab:absolute-fit}
\begin{tabular}{@{}lrrr@{}}
\toprule
Checkpoint & Calibrated CE & Excess over human entropy & Share of gap closed (\%) \\
\midrule
Llama-3.1 & 0.663 & 0.081 & 27 \\
Mistral & 0.674 & 0.092 & 17 \\
OLMo-3 & 0.680 & 0.098 & 12 \\
Qwen2.5 & 0.663 & 0.081 & 27 \\
Qwen3 & 0.670 & 0.088 & 21 \\
Qwen3.5 & 0.663 & 0.081 & 27 \\
T\"ulu & 0.661 & 0.079 & 29 \\
\bottomrule
\end{tabular}
\end{table}

\paragraph{Selecting group differences.} The direction audit uses the four survey dimensions above. Respondents in each wave are split into two halves by a fixed hash of their panel identifier, so a respondent stays in the same half across waves. On the discovery half, a pair of groups within a question and dimension is selected when both groups have an effective sample size of at least 30, the gap is at least 0.05, and it exceeds 1.96 approximate standard errors; the selected pair is kept only if both groups also reach an effective size of 30 in the evaluation half. Of 516 candidate pairs, 141 in 58 questions meet these rules. The evaluation half is used only as the target: human replication is the share of selected gaps whose sign agrees across halves (95\%), and model recovery is the share whose predicted sign matches the evaluation half. Selection of the direction and size of a difference uses only the discovery half; the evaluation half enters only through the sample-size eligibility rule, and no model output is used.

\paragraph{Group directions relative to chance.}
\label{app:direction-chance}
Temperature scaling is monotone, so the direction recovery rates of \S\ref{sec:human} are identical before and after calibration. We test each post-trained checkpoint against the 50\% expected from fair guesses, resampling the 58 questions with all their selected group pairs (10,000 draws; Bonferroni across seven checkpoints). Four checkpoints are detected above chance; for OLMo-3, T\"ulu, and Qwen3.5 no difference from chance is detected (Table~\ref{tab:direction-chance}). No pairwise difference between checkpoints survives correction across 21 comparisons. All rates remain far below the 95\% at which a held-out half of human respondents recovers the same directions. Base checkpoints are scored as a control below.

\begin{table}[htbp]\centering\small\renewcommand{\arraystretch}{1.1}
\caption{\textbf{Models recover some group directions, far fewer than human replication.} Percentage of 141 selected group differences whose direction the model recovers, with uncorrected question clustered 95\% intervals. Reported $p$ values are raw one sided tests against 50\%; the final column applies Bonferroni correction across seven models.}
\label{tab:direction-chance}
\begin{tabular}{@{}lrlrl@{}}
\toprule
Checkpoint & Recovered (\%) & 95\% interval & $p$ & Above chance (Bonferroni) \\
\midrule
Llama-3.1 & 69.5 & [62.1, 76.7] & 0.0001 & Yes \\
Mistral & 67.4 & [59.3, 74.6] & 0.0001 & Yes \\
OLMo-3 & 56.7 & [46.4, 66.0] & 0.1033 & Not detected \\
Qwen2.5 & 67.4 & [59.0, 75.2] & 0.0002 & Yes \\
Qwen3 & 69.5 & [60.3, 78.3] & 0.0001 & Yes \\
Qwen3.5 & 61.0 & [51.4, 69.9] & 0.0142 & Not detected \\
T\"ulu & 61.0 & [52.1, 69.9] & 0.0089 & Not detected \\
\bottomrule
\end{tabular}
\end{table}

\paragraph{Base checkpoints as a control.} The audit above scores only post-trained checkpoints. Scoring each family's Base checkpoint on the same 141 differences (same answer-order averaging; the stored post-trained accuracies are reproduced exactly; question-clustered draws shared by Base and post-trained; Bonferroni across seven families) shows Base recovery of 53.5 to 69.9\%, above chance in four families (Table~\ref{tab:direction-base}). Post-training changes recovery significantly only in Llama-3.1, where it rises by 16.0 points from the shared Llama-3.1 Base; the other six changes, between $-3.5$ and $+7.4$ points, are unresolved. Neither Base nor post-trained checkpoints approach the 95\% replication of the human halves, so the gap in group directions is not created by post-training.

\begin{table}[htbp]\centering\footnotesize\setlength{\tabcolsep}{4pt}\renewcommand{\arraystretch}{1.06}
\caption{\textbf{Group-direction recovery before and after post-training.} Percentage of the 141 selected differences recovered. Bonferroni intervals across seven families for the post-trained minus Base difference; bold estimates have intervals excluding zero. Llama-3.1 and T\"ulu share a Base.}
\label{tab:direction-base}
\begin{tabular*}{\linewidth}{@{\extracolsep{\fill}}lrlrrr@{}}
\toprule
Family & Base (\%) & Base above chance & Post-trained (\%) & Difference & Corrected interval \\
\midrule
Llama-3.1 & 53.5 & Not detected & 69.5 & \textbf{+16.0} & 3.9 to 30.1 \\
Mistral & 62.8 & Yes & 67.4 & +4.6 & -4.6 to 14.4 \\
OLMo-3 & 60.3 & Not detected & 56.7 & -3.5 & -12.8 to 5.1 \\
Qwen2.5 & 63.8 & Yes & 67.4 & +3.5 & -3.9 to 11.3 \\
Qwen3 & 69.9 & Yes & 69.5 & -0.4 & -10.0 to 9.5 \\
Qwen3.5 & 62.8 & Yes & 61.0 & -1.8 & -10.8 to 7.0 \\
T\"ulu & 53.5 & Not detected & 61.0 & +7.4 & -1.1 to 17.1 \\
\bottomrule
\end{tabular*}
\end{table}

\paragraph{Calibration data and reuse.}
The selection controls, calibration transfer tests, and label budget results use Equation~\ref{eq:temperature}. In the descriptive label-budget analysis, 32 human-labeled questions bring all eleven evaluated checkpoints within 0.002 nats of calibration on all 84 development questions; each question supplies group distributions. Label-budget curves average 200 nested development subsets, and their spread describes subsets rather than a population confidence interval. Leaving one survey wave out costs at most 0.0007 nats in all four main families, with 99.375\% intervals within $\pm0.005$. Reusing Base's temperature after training instead adds about 0.043 nats for Qwen2.5 and 0.040 for Llama, while OLMo and T\"ulu remain inconclusive. Calibrated regret near zero means the two selectors reach similar loss, not that all checkpoints are equivalent or that an internal representation has been repaired.

\subsection{Persona benefit and an output intervention}
\label{app:intervention}
This section reports the persona-benefit comparison cited in \S\ref{sec:human} and an output-level intervention that removes the component predicted by fitted gains. The analyses use four survey dimensions: religious attendance, marital status, metropolitan residence, and volunteering. These are recorded behavioral, household, and geographic variables, not psychological traits. The seven checkpoint pairs are OLMo-3, T\"ulu, Qwen2.5, Llama-3.1, Qwen3, Qwen3.5, and Mistral, a set that differs from the seven behavioral pairs of \S\ref{sec:results}, which include OLMo-2 rather than Mistral. Source gains are fitted on the exact survey persona sentences over 60 opinion questions, with 15 further questions for prediction checks, and are distinct from the 57-attribute coefficients of the main text.

\paragraph{Persona benefit.} Against directly matched no-attribute outputs, calibrated persona prompts improve cross entropy by about 0.00579, 0.01065, 0.01517, and 0.00588 nats for OLMo-3, T\"ulu, Qwen2.5, and Llama-3.1. The no-attribute prompt keeps the generic adult U.S. respondent description. In OLMo-3, SFT benefits more from the respondent's group information, yet Base better matches human responses overall. This separates the incremental benefit of persona information from overall simulation quality.

\paragraph{Removing the predicted component.} Let $z_0(k)$ and $z_1(k)$ denote Base and post-trained logits conditioned on group $k$, and let $x_k=z_0(k)-\E_k z_0(k)$. For fitted dimension gain $g_a$ and shared gain $\bar g$, the predicted heterogeneous component is $H_k=(g_a-\bar g)x_k$. We evaluate
\begin{equation}
 p_{\lambda,k}=\sigmoid\!\left(a_\lambda[z_1(k)-\lambda H_k]\right),
\end{equation}
where every condition fits $a_\lambda$ with the same development-label budget. The primary comparison removes the component ($\lambda=1$) and compares it with the unchanged output ($\lambda=0$) and with magnitude-matched shuffled readouts. Across the seven pairs, all fourteen 99.64\% intervals lie within $\pm0.005$ nats; this tolerance equals 33--86\% of the persona benefits above, so it bounds the effect at a chosen resolution rather than proving it negligible. The intervals condition on fitted source gains and fixed human targets, and the intervention acts on outputs, not on an internal neural pathway.

In descriptive checks, removal changes calibrated probabilities by about 0.09--1.25 percentage points RMS across doses from $-1$ to 2 and both temperature choices; a human-supervised correction matched in size has lower loss in five of seven families, with small absolute differences; and in the respondent-split audit of 141 group contrasts, removal's contrast-error intervals span zero in all seven families.

\section{Related work in more detail}
\label{app:literature}
\label{app:related-details}
\paragraph{Simulation targets.} Persona conditioning can reproduce patterns in survey samples and experiments \citep{argyle2023outofone,aher2023simulate}, and interactive agents extend simulation to behavior supported by memory and planning \citep{park2023generative}. These uses have different targets: plausible behavior, individual prediction, and population distributions. OpinionQA documents uneven agreement across demographic groups \citep{santurkar2023opinions}, and synthetic survey responses can misstate the variation and relationships in human data \citep{bisbee2024synthetic}. Richer descriptions can help: generated backstories approximate response distributions \citep{moon2024anthology}, interview-grounded agents predict individuals \citep{park2024thousand}, and survey-derived personas preserve population information \citep{rupprecht2026ggss}. Our profile edits ask a narrower question, how surrounding information changes one attribute's effect; attenuation of that effect does not make a complete profile a worse simulator.

\paragraph{Attributes and composition.} \citet{hu2024persona} relate prompting benefits to the human variation explained by persona variables, and \citet{froehling2026attributes} study which attributes to include for a given question. We instead predict how training changes controlled model contrasts, so our budget comparison concerns measuring a fixed inventory, not choosing which attributes best predict people. \citet{liu2024personasteered} combine demographic descriptions with congruous or incongruous stances; \citet{chameleon2026limit} and \citet{bhattacharyya2026personality} examine collapse and stable configurations when attributes combine. Our preference controls measure one cue's influence conditional on another, and matched profile edits separate uniform attenuation, which shared scaling absorbs, from changed relationships, for which target refitting recovers additional predictive value.

\paragraph{Prompt variation, steering, and reuse.} \citet{beck2024sociodemographic} distinguish sensitivity, performance, and robustness to paraphrase; \citet{lutz2025prompt} compare role formats and demographic priming; FormatSpread measures variation across formats \citep{sclar2024formatspread}. \citet{miehling2025steerability} quantify prompt steerability, while \citet{sorensen2026spectrum} distinguish steerability, distributional coverage, and alignment. Advice research separates available perspectives from their appropriate selection \citep{kumar2026advice}. Our comparison is predictive: it asks what a source relationship adds once the target may recalibrate it, and whether regularization alone would supply the same benefit. A change can leave source information useful while still making target refitting preferable, and failing to detect a reuse advantage does not show that the source contains no information.

\paragraph{Training and the assistant persona.} Post-training has been described as refining an assistant persona \citep{marks2026selection}, and persona vectors, the assistant axis, and persona features in emergent misalignment connect behavior to activation directions \citep{chen2025persona,lu2026assistant,wang2025personafeatures}. These accounts motivate our stage comparisons but study a different object: our gains describe responses to assigned attributes across checkpoints, not the assistant's identity or an internal direction.

\paragraph{Measurement and calibration.} Answer order and constrained choices affect measured opinions \citep{dominguezolmedo2024questioning,rottger2024compass}, and token probabilities can disagree with generated answers \citep{wang2024answerc}; these motivate our answer-order and sampled-answer checks. Temperature scaling \citep{guo2017calibration} and distributional evaluation against human labels \citep{meister2025distributional} motivate calibration as a control fitted on development questions and evaluated on held-out ones. A study-by-study comparison with the closest experiments appears in the supplementary material (Section~S4).

\clearpage
\section*{Ethics statement}
The survey targets describe historical group response distributions, not fixed properties of identities or predictions about individuals. Persona conditioning may reproduce stereotypes and obscure variation within groups. Our experiments do not justify replacing human participation in consequential decisions. We analyze existing survey data and synthetic prompts and report no individual respondents' records.
\section*{Reproducibility statement}
The appendix specifies the prompts, profile selection, weighting, data splits, estimators, and uncertainty calculations needed to reproduce the comparisons; the repository lists every statement and the prompt templates. A public repository (\url{https://github.com/thzva/persona-gain}) provides the question pools and frozen experimental designs, the cached answer log odds of the core attribute experiments, every saved numerical result behind the tables and figures, the analysis and figure scripts, and a map from each reported result to its file. Model weights and raw survey respondent records are not redistributed. Principal numerical estimates were checked using separate implementations on the saved model outputs.
\section*{AI use statement}
AI tools assisted with research and writing; the authors take responsibility for the content of this paper.
\end{document}